\documentclass{article}

\usepackage{style/arxiv}

\usepackage[utf8]{inputenc}
\usepackage[T1]{fontenc}
\usepackage{hyperref}
\usepackage{url}
\usepackage{booktabs}
\usepackage{nicefrac}
\usepackage{microtype}
\usepackage{amsmath,amssymb,amsthm}
\usepackage{mathtools}
\usepackage{mathrsfs}
\usepackage{graphicx}
\usepackage{subcaption}
\usepackage[inline]{enumitem}
\usepackage{bbm,bm}
\usepackage{upgreek}
\usepackage{xcolor}
\hypersetup{
   colorlinks=true,
   linkcolor=blue,
   filecolor=blue,
   citecolor=purple,
   urlcolor=cyan,
}
\usepackage[disable]{todonotes}
\usepackage{ifthen}
\usepackage{xargs}
\usepackage{accents}
\usepackage{stmaryrd}
\usepackage{cleveref}
\usepackage{autonum}

\title{Characterizing the Generalization Error of Random Feature Regression with Arbitrary Data-Augmentation}

\date{}

\author{%
  Lucas Morisset\thanks{Corresponding author: \texttt{lucas.morisset@polytechnique.edu}} \\
  Qube Research \& Technologies \\
  \'Ecole Polytechnique
  \And
  Alain Durmus \\
  \'Ecole Polytechnique
  \And
  Adrien Hardy \\
  Qube Research \& Technologies
}

\newcommand{\E}{\mathbb{E}} 

\newcommand{\Xtt}{\mathtt{X}}
\newcommand{\xtt}{\mathtt{x}}
\newcommand{\Ytt}{\mathtt{Y}}
\newcommand{\ytt}{\mathtt{y}}

\newcommand{\Ztt}{\mathtt{Z}}

\newcommand{\Var}{\mathrm{Var}} 
\newcommand{\cov}{\mathrm{Cov}} 
\newcommand{\Cov}{\mathrm{Cov}} 
\newcommand{\Brm}{\mathrm{B}}

\newcommand{\Prob}{\mathbb{P}} 
\newcommand{\tr}{\mathrm{tr}} 
\newcommand{\Resolvent}{\mathrm{R}}
\newcommand{\detequiv}{\overline{\mathrm{R}}}
\newcommand{\SampleCov}{\mathrm{C}}
\newcommand{\SampleCrossCov}{\mathrm{H}}
\newcommand{\Diag}{\mathrm{Diag}}

\newcommand{\Aug}{\mathrm{Aug}}

\newcommand{\Zbf}{\boldsymbol{Z}}

\newcommand{\Abf}{\boldsymbol{\mathrm{A}}}

\newcommand{\etabf}{\boldsymbol{\eta}}

\newcommand{\Qfun}{\mathrm{Q}_\Aug}
\newcommand{\Sfun}{\mathrm{S}_\Aug}
\newcommand{\Sequiv}{\overline{\mathrm{S}}_\Aug}
\newcommand{\Qequiv}{\overline{\mathrm{Q}}_\Aug}

\renewcommand{\complement}{\mathrm{c}}

\newcommand{\consta}{\mathrm a}
\newcommand{\constb}{\mathrm b}
\newcommand{\const}{\mathrm c}
\newcommand{\thetahat}{\hat{\theta}}
\newcommand{\stochdom}{\prec_s}

\def\bfu{\mathbf{u}}

\def\bfD{\mathbf{D}}

\def\msk{\mathsf{K}}

\def\msn{\mathsf{N}}

\def\mse{\mathsf{E}}

\def\msm{\mathsf{M}}
\def\msu{\mathsf{U}}

\def\mce{\mathcal{E}}

\def\mct{\mathcal{T}}

\def\rset{\mathbb{R}}

\def\cset{\mathbb{C}}

\def\nset{\mathbb{N}}

\def\rmd{\mathrm{d}}

\def\rme{\mathrm{e}}
\def\rmi{\mathrm{i}}

\newcommand{\R}{\mathbb R}

\newcommandx{\functionspace}[2][1=+]{\mathbb{F}_{#1}(#2)}

\newcommand{\argmin}{\operatorname*{arg\,min}}

\newcommandx{\VarDeux}[3][3=]{\operatorname{Var}^{#3}_{#1}\left\{#2 \right\}}

\newcommand{\1}{\mathbbm{1}}

\newcommand{\LeftEqNo}{\let\veqno\@@leqno}

\newcommand{\N}{\ensuremath{\mathbb{N}}}
\newcommand{\C}{\ensuremath{\mathbb{C}}}

\newcommand{\abs}[1]{\left\vert #1 \right\vert}

\newcommandx{\Vnorm}[2][1=V]{\| #2 \|_{#1}}
\newcommandx{\VnormEq}[2][1=V]{\left\| #2 \right\|_{#1}}
\newcommandx{\norm}[2][1=]{\ifthenelse{\equal{#1}{}}{\left\Vert #2 \right\Vert}{\left\Vert #2 \right\Vert^{#1}}}
\newcommandx{\normLigne}[2][1=]{\ifthenelse{\equal{#1}{}}{\Vert #2 \Vert}{\Vert #2\Vert^{#1}}}

\newcommandx\probaMarkovTilde[2][2=]
{\ifthenelse{\equal{#2}{}}{{\widetilde{\mathbb{P}}_{#1}}}{\widetilde{\mathbb{P}}_{#1}\left[ #2\right]}}

\newcommand{\plusinfty}{+\infty}

\def\ie{\textit{i.e.}}

\def\eqsp{\;}

\newcommand{\ccint}[1]{\left[#1\right]}

\newcommandx{\weight}[2][2=n]{\omega_{#1,#2}^N}

\newcommandx\sequence[3][2=,3=]
{\ifthenelse{\equal{#3}{}}{\ensuremath{\{ #1_{#2}\}}}{\ensuremath{\{ #1_{#2}, \eqsp #2 \in #3 \}}}}
\newcommandx\sequenceD[3][2=,3=]
{\ifthenelse{\equal{#3}{}}{\ensuremath{\{ #1_{#2}\}}}{\ensuremath{( #1)_{ #2 \in #3} }}}

\newcommandx{\sequencen}[2][2=n\in\N]{\ensuremath{\{ #1_n, \eqsp #2 \}}}
\newcommandx\sequenceDouble[4][3=,4=]
{\ifthenelse{\equal{#3}{}}{\ensuremath{\{ (#1_{#3},#2_{#3}) \}}}{\ensuremath{\{  (#1_{#3},#2_{#3}), \eqsp #3 \in #4 \}}}}
\newcommandx{\sequencenDouble}[3][3=n\in\N]{\ensuremath{\{ (#1_{n},#2_{n}), \eqsp #3 \}}}

\def\iid{\text{i.i.d.}}

\newcommand{\opnorm}[1]{{\left\vert\kern-0.25ex\left\vert\kern-0.25ex\left\vert #1
    \right\vert\kern-0.25ex\right\vert\kern-0.25ex\right\vert}}
\newcommand{\op}{\operatorname{op}}

\newcommandx{\CPE}[3][1=]{{\mathbb E}_{#1}\left[\left. #2 \, \middle \vert \, #3 \right. \right]} 

\newcommandx{\CPELigne}[3][1=]{{\mathbb E}_{#1}[\left. #2 \,  \vert \, #3 \right. ]} 

\newcommandx{\CPVar}[3][1=]{\mathrm{Var}^{#3}_{#1}\left\{ #2 \right\}}
\newcommand{\CPP}[3][]
{\ifthenelse{\equal{#1}{}}{{\mathbb P}\left(\left. #2 \, \right| #3 \right)}{{\mathbb P}_{#1}\left(\left. #2 \, \right | #3 \right)}}

\newcommandx{\osc}[2][1=]{\mathrm{osc}_{#1}(#2)}

\def\sphere{\mathrm{S}}

\newcommand\coupling[2]{\Gamma(\mu,\nu)}

\renewcommand{\geq}{\geqslant}
\renewcommand{\leq}{\leqslant}

\def\bfZ{\mathbf{Z}}

\def\Ltt{\mathtt{L}}

\newcommand{\fraca}[2]{(#1/#2)}

\def\loiGauss{\mathbf{N}}

\def\gauss{\mathrm{N}}
\def\Idd{\mathrm{I}_d}
\newcommand{\Id}[1]{\mathrm{I}_{#1}}

\def\bfX{\mathbf{X}}
\def\bfY{\mathbf{Y}}

\def\bfo{\mathbf{o}}

\newcommandx{\Voi}[1][1=i]{\mathfrak{V}_{\bfo,#1}}
\newcommandx{\Vlyapc}[2][1=\bfo,2=i]{\mathfrak{V}_{#1,#2}}

\def\Wlyap{\mathfrak{W}}
\def\bfomega{\boldsymbol{\omega}}

\newcommandx{\Woi}[1][1=i]{\Wlyap_{\bfomega,#1}}
\newcommandx{\Wlyapc}[2][1=\bfomega,2=i]{\Wlyap_{#1,#2}}

\def\Frob{\mathrm{F}}

\def\op{\mathrm{op}}

\def\msn{\mathsf{N}}
\def\proj{\mathrm{proj}}

\newtheorem{theorem}{Theorem}
\newtheorem{definition}{Definition}
\newtheorem{lemma}{Lemma}
\newtheorem{corollary}{Corollary}
\newtheorem{proposition}{Proposition}
\newtheorem{example}{Example}
\newtheorem{assumption}{\textbf{H}\hspace{-3pt}}

\crefname{theorem}{theorem}{theorems}
\Crefname{theorem}{Theorem}{Theorems}
\crefname{definition}{definition}{definitions}
\Crefname{definition}{Definition}{Definitions}
\crefname{lemma}{lemma}{lemmas}
\Crefname{lemma}{Lemma}{Lemmas}
\crefname{corollary}{corollary}{corollaries}
\Crefname{corollary}{Corollary}{Corollaries}
\crefname{proposition}{proposition}{propositions}
\Crefname{proposition}{Proposition}{Propositions}
\crefname{example}{example}{examples}
\Crefname{example}{Example}{Examples}
\crefname{assumption}{\textbf{H}}{\textbf{H}}
\Crefname{assumption}{\textbf{H}\hspace{-3pt}}{\textbf{H}\hspace{-3pt}}

\begin{document}

\maketitle

\begin{abstract}
This paper aims at analyzing the regularization effect that data augmentation 
induces on supervised regression methods in the proportional regime, 
where the number of covariates grows proportionally to the number of samples. 
We provide a tight characterization of the test error, measured in mean squared error, 
in terms only of the population quantities of the true data, 
as well as first and second order statistics of the augmentation scheme.
Our results are valid under misspecified feature maps, 
and for any network architecture where only the last readout layer is trained, 
and the rest of the network is either frozen or randomly initialized.
We specify our results in the case of Gaussian data, 
and show that our asymptotic characterization is tight in this setting.
\end{abstract}

\section{Introduction} \label{sec:introduction}

Data augmentation (DA) is now a standard ingredient in modern machine learning pipelines,
with extensive empirical evidence reporting improvements in generalization across modalities and tasks \cite{mumuni2022data,wang2025comprehensive}.
It is often used to encode task-relevant symmetries directly into the training procedure, 
for instance by encouraging invariance to image rotations or other transformations of the input \cite{ShortenKhoshgoftaar2019SurveyImageDA, ChenDobribanLee2020GroupTheoreticDA}.
It has also been identified as one of the most effective regularization techniques across both supervised learning settings
\cite{Bishop1995NoiseTikhonov, cubuk2019autoaugment, mumuni2022data, wang2025comprehensive} and self-supervised/unsupervised learning
\cite{feng2021survey, VanAsselIbrahimBiancalaniRegevBalestriero2025}.
Domain-specific augmentation pipelines have been central to progress in computer vision
\cite{ShortenKhoshgoftaar2019SurveyImageDA, kumar2024image}, natural language processing
\cite{feng2021survey, shorten2021text, bayer2022survey}, and time-series or audio applications
\cite{wen2020time, iwana2021empirical, iglesias2023data}.
Despite these empirical successes, the benefits of DA remain highly task- and data-dependent, and augmentation schemes are often engineered in an ad hoc manner
\cite{fawzi2016adaptive, cubuk2019autoaugment, lim2019fast, hataya2020faster}.

In contrast with this rich empirical literature, comprehensive theoretical analyses of DA remain relatively scarce. 
Two classical starting points are, first, the interpretation of additive Gaussian noise as a form of explicit (ridge-like)
regularization \cite{Bishop1995NoiseTikhonov, LinKaushikDyerMuthukumar2024}, and second, the idea that leveraging distributional invariances and group 
structure in the learning objective helps decrease the variance of the model without increasing its bias \cite{ChenDobribanLee2020GroupTheoreticDA}.
Yet, when applied to modern and complex augmentation schemes, these works either provide only upper bounds on the generalization error \cite{LinKaushikDyerMuthukumar2024}, 
or require very strong assumptions on the data distribution (e.g. approximate invariance under the augmentation scheme \cite{ChenDobribanLee2020GroupTheoreticDA}).
The present manuscript is motivated by the need for a tractable theoretical framework capable of capturing augmentation schemes used in deep learning, 
including those that may induce strong dependence across observations without preserving the data distribution.

Among the tractable proxy models used to study neural networks, random feature networks, obtained by training only the last layer of a random representation, have emerged as particularly natural candidates 
\cite{BelkinHsuMaMandal2019, DhifallahLu2020PreciseRF, pmlr-v145-goldt22a, HuLu2023UniversalityRF}.
Beyond their mathematical convenience, random networks at initialisation have proven to be a valuable theoretical testbed, and can capture empirically observed phenomena such as double-descent \cite{MeiMontanari2022GenErrorRF}
and benign overfitting \cite{BartlettLongLugosiTsigler2020}.
As a result, tractable proxies for networks at initialisation, such as random feature models \cite{BelkinHsuMaMandal2019, DhifallahLu2020PreciseRF, pmlr-v145-goldt22a, HuLu2023UniversalityRF, pmlr-v202-schroder23a, Schroeder2024}, 
have received considerable attention.
Remarkably, it has also been observed that structured random networks, whose weights are sampled to match certain statistics of trained models, can achieve performance comparable to that of
fully trained networks \cite{GuthMenardRochetteMallat2023Rainbow}. In some cases, matching second-order statistics alone is sufficient, which further supports random feature models as a natural
testbed for deep learning. This perspective has led to a comprehensive theory of random feature models \cite{pmlr-v202-schroder23a, NEURIPS2023_85d456fd, pmlr-v195-bosch23a, Schroeder2024}. 
Yet this theory has so far remained confined to the i.i.d. setting, and a central technical challenge of the present work is to extend it to the dependent regime induced by data augmentation.

More precisely, we aim to provide a precise asymptotic characterization of the out-of-sample error of feature regression trained with data augmentation.
To this end, we consider a general augmentation framework that covers a broad class of modern augmentation schemes. 
Our results apply to any class of architecture with a trained linear readout and either frozen or randomly initialized representation, and they reveal that data augmentation can 
induce subtle and sometimes competing effects on generalization. Our contributions are as follows:

\begin{itemize}
    \item We derive deterministic equivalents for the generalization error of random-feature regression trained with data augmentation under possible misspecification, 
    extending previous results from \cite{DhifallahLu2020PreciseRF, pmlr-v145-goldt22a, HuLu2023UniversalityRF, pmlr-v202-schroder23a, Schroeder2024} to the dependent setting. 
    In particular, we obtain a closed-form deterministic quantity that approximates the true generalization error in the proportional regime, 
    \ie, when the number of features scales proportionally with the sample size. We also derive explicit non-asymptotic bounds for the corresponding approximation error 
    as a function of the ridge parameter and the sample size.
    
    \item Beyond the test error, we derive deterministic equivalents for the resulting estimator, as well as for key quadratic forms and transformations arising in the generalization analysis, 
    which we believe are of independent interest.
    
    \item We empirically validate the accuracy of the deterministic equivalents on synthetic Gaussian experiments across a range of dimensions, aspect ratios, augmentation strengths, and ridge parameters. 
    We also investigate how data augmentation affects the bias and variance components of the error when the feature model is misspecified. 
    One of our main findings is that augmentation can lead to a substantial reduction in variance and improved generalization over a broad range of ridge parameters, 
    without necessarily inducing a monotonic increase in bias, contrary to the usual intuition drawn from correctly specified settings.
    \end{itemize}

\subsection{Notation}
Motivated by high-dimensional statistics (and while all our results are non-asymptotic), we view the sample size $n$, the covariate dimenson $d = d(n)$, 
and the feature dimension $p = p(n)$ as variable scaling parameters. For notational simplicity, we typically omit the explicit dependence on $n$ for non-scalar quantities, 
as it is always implicitly determined by the dimensions of the relevant vectors and matrices. Throughout, $\const$ denotes a universal constant, 
independent of $n$, $d$, $p$, and $\lambda$; its value may change from line to line. For a matrix $\mathbf{A}$, we write $\|\mathbf{A}\|_{\Frob}$ and $\|\mathbf{A}\|_{\op}$
for its Frobenius and operator norms, respectively. Moreover, for any $\mathbf{A}\in\R^{d\times n}$, any $i \in {1,\ldots,n}$, and any function $f:\R^d\to\R^p$,
we denote by $\mathbf{A}_i$ the $i$-th column of $\mathbf{A}$ and define $f(\mathbf{A})=(f(\mathbf{A}_1),\dots,f(\mathbf{A}_n))\in\R^{p\times n}$.
Finally, for a random matrix $A$, we define its variance by $\Var(A)=\E\left[\|A-\E[A]\|_{\Frob}^2\right]$, which corresponds to the sum of the entrywise variances.
The matrix $\Id{p}$ denotes the $p \times p$ identity matrix. For two random matrices $A_1 \in \R^{p \times k}$ and $A_2 \in \R^{l \times k}$, we define the associated 
covariance matrix as $\Cov(A_1, A_2) = \E\left[(A_1-\E[A_1])(A_2-\E[A_2])^\top\right]$ and $\Cov(A_1) = \Cov(A_1, A_1)$. We denote by $\gauss(\bf m, \bf \Sigma)$ the Gaussian 
distribution with mean $\bf m$ and covariance matrix $\bf \Sigma$.
\section{Settings} \label{sec:settings}
Consider a training dataset $\Ztt = (\Ztt_i)_{i=1}^n = \{(\Xtt_i,\Ytt_i)\}_{i=1}^n$, where $\Xtt_i\in\R^d$ and $\Ytt_i\in\R$ denote the
covariates and responses respectively, satisfying the following condition.
\begin{assumption}
    \label{ass:data_distribution}
    We assume that the true covariates $\Ztt = \{(\Xtt_i, \Ytt_i)\}_{i=1}^n$ are centered i.i.d. random variables,
    such that the responses write for some $\varphi_\star : \R^d \to \R^{p_\star}$ and $\theta_\star \in \R^{p_\star}$, $p_{\star} \geq 1$,
    \begin{equation}
    \label{eq:response_shape}
    \Ytt_i = \theta_\star^\top \varphi_\star(\Xtt_i) + \varepsilon_i\eqsp,
    \end{equation}
    for all $i \leq n$, and where $\varepsilon_i$ is a centered, sub-Gaussian random variable with parameter  $\sigma^2$.
\end{assumption}
Let $\varphi:\R^d\to\R^p$ be a feature map which may differ from $\varphi_{\star}$; note that we
allow different codimensions $p$ and $p_\star$. In particular, the learned representation $\varphi:\R^d\to\R^p$ used by the estimator need not coincide with the
ground-truth representation $\varphi_\star:\R^d\to\R^{p_\star}$ in \eqref{eq:response_shape},
allowing for model misspecification.
Equipped with $\varphi$, feature regression model aims to fit the data finding $\theta\in\rset^p$ that minimizes the empirical risk function 
$\mathcal{L}(\theta)= n^{-1}\sum_{i=1}^n \bigl\{\Ytt_i-\theta^\top \varphi(\Xtt_i)\bigr\}^2$. 
To improve performance various DA strategies have been proposed. In this paper, we model a data augmentation scheme as a random transformation
$(\tau_{\xtt},\tau_{\ytt})$ acting on feature/label pair. More precisely we consider $\tau_{\xtt} : \rset^{d}\times \rset \times \mse \to \rset^d$ 
and $\tau_{\ytt}: \rset^{d}\times \rset \times \mse \to \rset$, referred to as the feature and label
transformations respectively, and  where $(\mse,\mce)$ is a measurable space. Then, a data augmentation scheme is completed by specifying a distribution 
$\mathcal{T}$ over $(\mse,\mce)$. Artificial samples are then defined as $(\tau_{\xtt}(\Ztt_i,\eta),\tau_{\ytt}(\Ztt_i,\eta))$, where $\eta$ 
is distributed according to $\mathcal{T}$ and is independent of the dataset $\Ztt$. This framework covers both input-only and label-dependent augmentations.
Based on $(\tau_{\xtt},\tau_{\ytt},\mct)$, we can then consider the augmented risk functions defined as
  \begin{align}
          \label{eq:augmented_ls}
    \mathcal{L}_{\Aug}(\theta)
    =\eqsp \frac{1}{n}\sum_{i=1}^n \E\!\left[\bigl\{\tau_{\ytt}(\Ztt_i,\eta)-\theta^\top \varphi(\tau_{\xtt}(\Ztt_i,\eta))\bigr\}^2 \eqsp \middle| \eqsp \Ztt\right]\eqsp,
    \end{align}
and consider the estimator obtained by minimizing a linear combinaison of the two previous risk functions: for $\lambda \geq 0$ and $\alpha \in \ccint{0,1}$,
\begin{equation} \label{eq:def_thetahat}
\thetahat_{\alpha,\lambda}
:=\arg\min_{\theta\in\R^p}\ (1-\alpha)\mathcal{L}(\theta)
+\alpha\,\mathcal{L}_{\Aug}(\theta)
+\lambda\|\theta\|_2^2 \eqsp.
\end{equation}
Note that we could set $\alpha = 0$ up to a change of $(\tau_{\xtt},\tau_{\ytt},\mct)$.
We do not adopt this presentation, however, as it would make our assumptions and results harder to state.

Our main object of interest is the generalization error of $\thetahat_{\alpha,\lambda}$
\begin{equation}\label{eq:generalization_error}
\mathcal{G}(\alpha,\lambda)
:=\E\Bigl[\bigl\{\varphi(X)^{\top}\thetahat_{\alpha,\lambda}-Y\bigr\}^2\Bigr]\eqsp,
\end{equation}
where $(X,Y)$ is an independent test pair drawn from the same distribution as the training data.
In contrast  to  previous works \cite{LinKaushikDyerMuthukumar2024},
we do not suppose that the augmentation scheme is unbiased,
i.e., for some $(x, y) \in \R^d \times \R$ it may hold that $\E\left[\tau_{\xtt}(x,y,\eta)\right] \neq x$, and $\E\left[\tau_{\ytt}(x,y,\eta)\right] \neq y$,
which enables the analysis of potential bias introduced by the DA scheme. To conduct our analysis,
we introduce the first moments of the DA scheme defined for any $z \in \R^{d} \times \R$  as
\begin{align} \label{eq:definition_mu_x_and_mu_y}
    \mu_{\xtt}(z) = \E\left[\varphi(\tau_{\xtt}(z,\eta))\right] \eqsp, \quad \mu_{\ytt}(z) = \E\left[\tau_{\ytt}(z,\eta)\right] \eqsp,
\end{align}
as well as the corresponding covariance operators,
\begin{align} \label{eq:definition_lambda_and_omega}
    \Lambda(z) = \cov[\varphi(\tau_{\xtt}(z,\eta))] \eqsp, \qquad
    \Omega(z) = \cov[\varphi(\tau_{\xtt}(z,\eta)),\tau_{\ytt}(z,\eta)] \eqsp.
\end{align}
We also define their empirical averages over the training samples,
\begin{align}
    \Lambda(\Ztt) = \dfrac{1}{n}\sum_{i=1}^n \Lambda(\Ztt_i) \eqsp, \quad \text{and} \quad \Omega(\Ztt) = \dfrac{1}{n}\sum_{i=1}^n \Omega(\Ztt_i) \eqsp,
\end{align}
as well as their population counterparts,
\begin{align}
    \overline{\Lambda} = \E\left[\Lambda(z)\right] \eqsp, \quad \overline{\Omega} = \E\left[\Omega(z)\right] \eqsp.
\end{align}
These quantities are of natural interest, as they quantify the implicit regularization effect induced by the DA scheme \cite{Bishop1995NoiseTikhonov, LinKaushikDyerMuthukumar2024}.
Indeed, writting a bias variance decomposition on \eqref{eq:augmented_ls} gives,
\begin{align}
    \mathcal{L}_{\Aug}(\theta) = \dfrac{1}{n} \|\mu_{\ytt}(\Ztt) - \theta^\top \mu_{\xtt}(\Ztt)\|_2^2 - 2 \theta^\top \Omega(\Ztt) + \theta^\top \Lambda(\Ztt) \theta + \const \eqsp,
\end{align}
where here, $\const$ is a constant that doesn't depend on $\theta$.
Thus the first order condition in \eqref{eq:def_thetahat} yields,
\begin{align} \label{eq:closed_form_thetahat}
    \hat{\theta}_{\alpha, \lambda} = \Resolvent_{\alpha, \lambda} \SampleCrossCov_\alpha \eqsp,
\end{align}
where
\begin{equation} \label{eq:definition_resolvent}
    \begin{aligned}
        \SampleCrossCov_\alpha = \left((1-\alpha) \SampleCrossCov + \alpha \SampleCrossCov' + \alpha \Omega(\Ztt)\right) \eqsp, \qquad 
        \Resolvent_{\alpha, \lambda} = \left((1-\alpha)\SampleCov + \alpha \SampleCov' + \alpha \Lambda(\Ztt) + \lambda \Id{p}\right)^{-1}\eqsp,
    \end{aligned}
\end{equation}
and
\begin{equation}\label{eq:definition_sample_covs}
    \begin{aligned}
        \makebox[0.8\linewidth][c]{$
        \SampleCov = \dfrac{1}{n} \varphi(\Xtt)\varphi(\Xtt)^\top \eqsp, \quad \SampleCrossCov = \dfrac{1}{n}\varphi(\Xtt)\Ytt^\top \eqsp,
        $}\\
        \makebox[0.8\linewidth][c]{$
        \SampleCov' = \dfrac{1}{n}\mu_{\xtt}(\Ztt)\mu_{\xtt}(\Ztt)^\top \eqsp, \quad \SampleCrossCov' = \dfrac{1}{n}\mu_{\xtt}(\Ztt)\mu_{\ytt}(\Ztt)^\top \eqsp.
        $}
    \end{aligned}
\end{equation}
The matrix $\Resolvent_{\alpha,\lambda}$ in \eqref{eq:definition_resolvent} is the resolvent associated with the effective
empirical covariance $(1-\alpha)\SampleCov+\alpha\SampleCov'+ \alpha \Lambda(\Ztt)$:
it therefore governs the stability and sensitivity of
$\hat{\theta}_{\alpha,\lambda}$ from \eqref{eq:closed_form_thetahat}.
Resolvent matrices are a central object in Random Matrix Theory,
as the Stieltjes transform of the empirical spectral distribution can be expressed in
terms of normalized traces of $\Resolvent_{\alpha,\lambda}$, which makes it a convenient
tool to characterize eigenvalue distributions in high dimension.
Data augmentation modifies this resolvent in two key ways. First, it effectively
enriches the dataset by incorporating $\mu_{\xtt}(\Ztt)$.
Second, it introduces an extra data-dependent regularization
term $\Lambda(\Ztt)$ which is typically anisotropic and inflates directions in feature space in
proportion to the augmentation-induced variability.

Our analysis of the generalisation error is made under the proportional regime, i.e.,
\begin{assumption}
    \label{ass:porportionality}
    We assume that $p = p(n)$ and
    $$\limsup_{n \to \infty} p/n < \infty\eqsp.$$
\end{assumption}
\Cref{ass:porportionality} is standard in high-dimensional asymptotics and allows
the covariate dimension $p$ to scale at most linearly with the sample size $n$.
In particular, it covers both underparameterized and overparameterized regimes, including cases where $p>n$.
Although it is not the focus of this paper, our results also cover the classical regime $p=\mathcal{O}(1)$.

We shall work under the following assumptions on the data distribution and the DA scheme:
\begin{assumption}\label{ass:feature_map_and_concentration}
    We assume that:
    \begin{enumerate}[wide, labelwidth=!, labelindent=0pt,label=(\roman*)]
        \item\label{ass:feature_map_and_concentration:concentration}
        for any $1$-Lipschitz continuous function $f:\R^d\to\R$, the random variable $f(\Xtt_1)$ is $1$-sub-Gaussian.
        \item\label{ass:feature_map_and_concentration:lipschitz}
        the feature maps $\varphi:\R^d\to\R^p$ and $\varphi_\star:\R^d\to\R^{p_\star}$ are $\Ltt_\varphi$-Lipschitz continuous, with $\Ltt_\varphi=\mathcal{O}(1)$ independent of $n$.
    \end{enumerate}
\end{assumption}
\Cref{ass:feature_map_and_concentration} formalizes the concentration of measure conditions supposed throughout the paper.
In particular the concentration hypothesis in \Cref{ass:feature_map_and_concentration}-\ref{ass:feature_map_and_concentration:concentration} is commonly refered to as the \textit{Lipschitz concentration} property,
for which we refer the reader to \cite{LouartCouillet2018ConcentrationLargeRandomMatrices, LouartCouillet2020RandomEquationsConcentration, LouartCouillet2021SpectralPropertiesSampleCovariance} for more background and details.
Naturally, combining \Cref{ass:data_distribution} and \Cref{ass:feature_map_and_concentration} implies that the pairs $\Ztt_i=(\Xtt_i,\Ytt_i)$ satisfies \Cref{ass:feature_map_and_concentration}-\ref{ass:feature_map_and_concentration:concentration} with a sub-Gaussian parameter bounded by $1+\Ltt_{\varphi}\norm{\theta_{\star}}+\sigma^2$.
As a simple example of distribution for which \Cref{ass:feature_map_and_concentration} holds, one can consider any Gaussian distribution $\gauss(0, \Abf)$, where $\|\Abf\|_\op \leq 1$, \cite{LouartCouillet2018ConcentrationLargeRandomMatrices}.

We now impose a mild stability condition on the augmentation scheme, requiring its first and second-order statistics to vary smoothly with the data.
\begin{assumption}
    \label{ass:artificial_data}
    We assume that $\mu_{\xtt} : \R^d \times \R \to \R^p$, $\mu_{\ytt} : \R^d \times \R \to \R$ defined in \eqref{eq:definition_mu_x_and_mu_y},
    and $\Lambda : \R^d \times \R \to \R^{p \times p}$, $\Omega : \R^d \times \R \to \R^{p \times 1}$ defined in \eqref{eq:definition_lambda_and_omega}
    are all Lipschitz continuous, with Lipschitz constants universally bounded in $n$. In addition, we assume that for a constant $c > 0$ independant of $n$,
    \begin{equation} \label{eq:uniform_variance_bound}
        \begin{aligned}
            \Var(\Lambda(\Ztt_1)) +  \Var(\Omega(\Ztt_1)) \leq \const \eqsp.
        \end{aligned}
    \end{equation}
\end{assumption}

Note that \Cref{ass:artificial_data} is a mere technical assumption that serves only the purpose of simplicity, 
we expect that our results can be extended under relaxed versions of this assumption as well as  \Cref{ass:feature_map_and_concentration}. 
Indeed, previous works in Random Matrix Theory suggest that we can extend the results developped in \Cref{proposition:det_equivs}
under only finite moment assumptions on the entries of $\Ztt$ \cite{LedoitPeche2011Eigenvectors}.
The bounds in \eqref{eq:uniform_variance_bound} are necessary to ensure the stability of the data augmentation scheme with respect to perturbations of the dataset. 
Without such a stability condition, any general analysis of the generalization error $\mathcal{G}(\alpha, \lambda)$ appears very challenging.

Our setting covers a wide range of augmentation pipelines used in practice,
including label-preserving perturbations such as additive Gaussian noise, random feature dropout/masking,
and small geometric or temporal deformations, as well as more structured transformations that explicitly modify the target.
In particular, we significantly enlarge the class of DA schemes that can be analyzed compared to prior works, 
e.g., \cite{LinKaushikDyerMuthukumar2024} requires that $\tau_\ytt((x,y), \cdot) = y$ and $\tau_{\xtt}((x,y), \cdot) = \tau_{\xtt}'(x, \cdot)$ for all $(x,y) \in \R^d \times \R$ and a map $\tau_{\xtt}'$.
Below, we describe a simple salt-and-pepper DA scheme in the identity feature map case, for which $\Lambda(\Ztt)$ and $\Omega(\Ztt)$ admit closed-form expressions, and for which \Cref{ass:artificial_data} can easily be verified.
Additional examples of schemes are given in \Cref{sec:artificial_data_examples}: 

\paragraph{Example: Salt-and-pepper noise injection.}
Let $k\in\N$, and fix a masking probability $\mathfrak{m}:\R^d\times\R\to[0,1]$ and a replacement map $s:\R^d\times\R\to\R^{d\times k}$.
We consider the augmentation scheme with transformations for any $z=(x,y) \in\rset^d\times \rset$,
\begin{equation}
    \begin{aligned}
        \tau_{\xtt}(z,\eta) &= x \odot m_{z} + s(z) (\eta_2 - \eta_2 \odot m_{z}) \eqsp,
        \\
      \tau_{\ytt}(z,\eta) &= y \eqsp,
    \end{aligned}
\end{equation}
where $\eta = (\{U_i\}_{i=1}^d,\eta_2)$ are independent random variables such that $\{U_i\}_{i=1}^d$ is i.i.d.~with distribution 
$\mathrm{Unif}(\ccint{0,1})$, $m_{z}$ is the vector with $i$-th component given by $\1\{U_i \leq \mathfrak{m}(z)\}$ and  
$\eta_2$ is a zero-mean $k$-dimensional random vector with $\E\left[\eta_2 \eta_2^\top \right] = \Id{k}$.
This construction yields $\Omega(\Ztt_i)=0$ for all $i\le n$, and
\begin{equation}
\Lambda(\Ztt_i) = (1-\mathfrak{m}(\Ztt_i))\left[\mathfrak{m}(\Ztt_i) \Diag(\Xtt_i \Xtt_i^\top) + s(\Ztt_i) s(\Ztt_i)^\top\right]\eqsp.
\end{equation}
We refer to \Cref{sec:artificial_data_examples} for details of the computation of $\Lambda(\Ztt_i)$. 


\section{Main results} \label{sec:main_results}

We start by introducing the notion of stochastic domination, introduced in \cite{ErdosKnowlesYau2013AveragingResolvents}. 
In our work, this definition is used to simplify the presentation of our results by retaining only the dominant term in $n$ in our bounds. 
This convention is also adopted in \cite{pmlr-v202-schroder23a, Schroeder2024}.

\begin{definition} \label{def:stochastic_domination}
    We say the sequence of random variables $(U_n)_{n \in \N}$ is stochastically dominated by the sequence of random variables $(V_n)_{n \in \N}$, 
    and write $U_n \stochdom V_n$, if for all $\varepsilon, D > 0$, it holds for $n$ large enough that
    \begin{equation}
        \Prob\left(U_n \geq n^{\varepsilon} V_n\right) \leq n^{-D} \eqsp.
    \end{equation}
\end{definition}

To state our main result in a concise and informative way, we first introduce the following random matrices
\begin{equation} \label{eq:definition_phi_psi}
    \Phi_i = (\varphi(\Xtt_i), \mu_{\xtt}(\Ztt_i)) \eqsp, \eqsp \Psi_i = (\Ytt_i, \mu_{\ytt}(\Ztt_i)) \eqsp,
  \end{equation}
which are used to define the following dilation matrices:
\begin{align}
    \hspace*{-0.75cm}
    \mathrm{A}_{\alpha, \lambda} &= \begin{pmatrix}
        \tr\left(\Sigma\, \Resolvent_{\alpha, \lambda}\right) &
        \tr\left(\Sigma'\, \Resolvent_{\alpha, \lambda}\right) \\
        \tr\left(\Sigma'\, \Resolvent_{\alpha, \lambda}\right) &
        \tr\left(\Sigma''\, \Resolvent_{\alpha, \lambda}\right)
    \end{pmatrix} \eqsp, \quad
    \mathrm{B}_{\alpha, \lambda} = \mathrm{W}_{\alpha}\left(\Id{2} + \frac{1}{n} \mathrm{W}_{\alpha} \E\left[\mathrm{A}_{\alpha, \lambda}\right] \mathrm{W}_{\alpha}\right)^{-1} \mathrm{W}_{\alpha}\eqsp,
    \label{eq:definition_dilation_matrices}\\
    \hspace*{-2cm}\mathrm{C}_{\alpha, \lambda} &= \begin{pmatrix}
        \tr\left(\Sigma\, \Resolvent_{\alpha, \lambda}\Sigma\, \Resolvent_{\alpha, \lambda}\right) &
        \tr\left(\Sigma'\, \Resolvent_{\alpha, \lambda}\Sigma\, \Resolvent_{\alpha, \lambda}\right) \\
        \tr\left(\Sigma'\, \Resolvent_{\alpha, \lambda}\Sigma\, \Resolvent_{\alpha, \lambda}\right) &
        \tr\left(\Sigma''\, \Resolvent_{\alpha, \lambda}\Sigma\, \Resolvent_{\alpha, \lambda}\right)
    \end{pmatrix} \eqsp, \quad
    \mathrm{D}_{\alpha, \lambda} = \frac{1}{n}\mathrm{B}_{\alpha, \lambda} \E\left[\mathrm{C}_{\alpha, \lambda}\right] \mathrm{B}_{\alpha, \lambda} \eqsp,
    \label{eq:definition_dilation_matrices_D}
\end{align}

where $\mathrm{W}_{\alpha} = \mathrm{Diag}(\sqrt{1-\alpha}, \sqrt{\alpha})$ is a diagonal weight matrix, and
$\Sigma, \Sigma', \Sigma''$ are the feature covariance matrices defined as
\begin{equation}
    \begin{aligned}
        \Sigma = \E\left[\varphi(\Xtt_1)\varphi(\Xtt_1)^\top\right] \eqsp,
        \quad \Sigma' = \E\left[\varphi(\Xtt_1)\mu_{\xtt}(\Ztt_1)^\top\right] \eqsp,
        \quad \Sigma'' = \E\left[\mu_{\xtt}(\Ztt_1)\mu_{\xtt}(\Ztt_1)^\top\right] \eqsp.
    \end{aligned}
  \end{equation}

The matrices introduced in \eqref{eq:definition_dilation_matrices} and \eqref{eq:definition_dilation_matrices_D} extend the dilation factors that appear in previous works \cite{chouard2022quantitative} for the non-augmented setting.
They arise naturally in our analysis through a fixed-point system satisfied by the resolvent matrix $\Resolvent_{\alpha,\lambda}$ defined in \eqref{eq:definition_resolvent}, together with concentration estimates for high-dimensional quadratic forms.
More concretely, in the proportional regime the resolvent is asymptotically unchanged when the contribution of a fixed sample (say $\Phi_1$) is removed from the empirical covariance, 
so one can compare the full resolvent to its 'leave-one-out'counterpart and enforce self-consistency between the two descriptions; 
typically through the use of the Woodbury matrix identity, which yields the said fixed-point system.

With these notations in place, the deterministic equivalent can be written in terms of the following population quantities:
\begin{equation}
    \begin{aligned}
    \makebox[\linewidth][c]{$
    \overline{\Sigma}_{\alpha, \lambda} = \E\left[\Phi_1 \mathrm{B}_{\alpha, \lambda}  \Phi_1^\top\right] \eqsp,
    \quad
    \overline{\Sigma}'_{\alpha, \lambda} = \E\left[\Phi_1 \mathrm{D}_{\alpha, \lambda} \Phi_1^\top\right] \eqsp,
    $} \\
    \makebox[\linewidth][c]{$
    \overline{\Gamma}_{\alpha, \lambda} = \E\left[\Phi_1 \mathrm{B}_{\alpha, \lambda} \Psi_1^\top\right] \eqsp, 
    \quad \overline{\Gamma}'_{\alpha, \lambda} = \E\left[\Phi_1 \mathrm{D}_{\alpha, \lambda} \Psi_1^\top\right] \eqsp,
    \quad \overline{\gamma}_{\alpha, \lambda} = \E\left[\Psi_1 \mathrm{D}_{\alpha, \lambda} \Psi_1^\top\right] \eqsp.
    $} \\
    \end{aligned}
\end{equation}
Note that because true and artificial data are stacked together in \eqref{eq:definition_phi_psi}, each of these quantities can be expanded into linear combinations of
population (cross-)covariance terms involving $\varphi(\Xtt_1)$, $\mu_{\xtt}(\Ztt_1)$, $\Ytt_1$, and $\mu_{\ytt}(\Ztt_1)$ only, 
where the linear coefficient terms are given by the entries of $\mathrm{B}_{\alpha, \lambda}$ in \eqref{eq:definition_dilation_matrices}, and $\mathrm{D}_{\alpha, \lambda}$ in \eqref{eq:definition_dilation_matrices_D} respectively.
Our Theorem below gives an explicit deterministic equivalent of the generalization error introduced in \eqref{eq:generalization_error}, 
together with an explicit non-asymptotic control of the approximation error.

\begin{theorem}
    \label{thm:main_result}
    Let $\alpha \in [0,1]$ and $(\lambda_n)_{n \in \N}$, such that for
    all $n$, $\lambda_n \geq \const/n$ for some $\const >0$.  Assume that \Cref{ass:data_distribution},
    \Cref{ass:porportionality},
    \Cref{ass:feature_map_and_concentration} and
    \Cref{ass:artificial_data} hold, and define
    \begin{align}
        \overline{\mathcal{G}}(\alpha, \lambda_n) = \theta_\star^\top \Sigma_{\star\star} \theta_\star - 2 \theta_\star^\top \Sigma_\star \overline{\theta}_{\alpha, \lambda_n}  + \overline{\gamma}_{\alpha, \lambda}
        + \overline{\theta}_{\alpha, \lambda_n}^\top \{\overline{\Sigma}'_{\alpha, \lambda} + \Sigma\} \overline{\theta}_{\alpha, \lambda_n} - 2 \overline{\Gamma}_{\alpha, \lambda_n}'^\top \overline{\theta}_{\alpha, \lambda_n}   
        + \sigma^2\eqsp,
    \end{align}
    where
    \begin{align}
  & \hspace{-0.5cm}
        \Sigma_\star = \E\left[\varphi_\star(\Xtt_1)\varphi(\Xtt_1)^\top\right] \eqsp,
        \quad \Sigma_{\star\star} = \E\left[\varphi_\star(\Xtt_1)\varphi_\star(\Xtt_1)^\top\right] \eqsp.
\\\label{eq:definition_equiv_theta}
   \text{and } \quad &     \overline{\theta}_{\alpha, \lambda_n} = \left(\overline{\Sigma}_{\alpha, \lambda} + \alpha \overline{\Lambda} + \lambda_n \Idd\right)^{-1} \left(\overline{\Gamma}_{\alpha, \lambda} + \alpha \overline{\Omega}\right)\eqsp.
    \end{align}    
    Then the generalization error
    $\mathcal{G}$ defined in \eqref{eq:generalization_error} satisfies
    \begin{equation}
        \left|\mathcal{G}(\alpha, \lambda_n) - \overline{\mathcal{G}}(\alpha, \lambda_n)\right| \stochdom \dfrac{1}{\lambda_n^{9/2} \sqrt{n}} + \dfrac{ \|\theta_\star\|_2}{\lambda_n^{7/2} \sqrt{n}}\eqsp.
        \label{eq:main_result}
    \end{equation}
\end{theorem}
The proof of \Cref{thm:main_result} is defered to \Cref{sec:proofs_main_results}.

The previous result provides a non-asymptotic control with an explicit dependence on the ridge parameter $\lambda_n$.
In particular, the remainder term scales as $\lambda_n^{-9/2}$.
This explicit $\lambda_n$-dependence is important in the proportional regime, where $\lambda_n$ may be small.
In particular, our results may be extended to the Ordinary Least-Squares regression, by allowing $\lambda_n$ to converge to $0$ slowly enough.
Note that this dependence in $\lambda_n^{-9/2}$ is competitive with the sharpest available bounds in the literature \cite{chouard2022quantitative,pmlr-v195-bosch23a, MeiMontanari2022GenErrorRF}.

To gain some qualitative intuition, we specialize the theorem to the well-specified case $\varphi=\varphi_\star$ together with an unbiased, label-preserving augmentation scheme,
so that $\mu_{\xtt}(\Ztt_1)=\varphi(\Xtt_1)$ and $\mu_{\ytt}(\Ztt_1)=\Ytt_1$.
In that regime, the mean augmented features coincide with the original ones, hence $\SampleCov'=\SampleCov$ and $\SampleCrossCov'=\SampleCrossCov$, 
while the feature covariances satisfy $\Sigma = \Sigma' = \Sigma'' = \Sigma_{\star} = \Sigma_{\star\star}$. 
We also introduce
\begin{align}
    \beta_{\alpha, \lambda_n} = \boldsymbol{1}^\top \mathrm{B}_{\alpha, \lambda_n} \boldsymbol{1} \eqsp,
    \quad 
    \delta_{\alpha, \lambda_n} = \boldsymbol{1}^\top \mathrm{D}_{\alpha, \lambda_n} \boldsymbol{1} \eqsp.
\end{align}
In this regime, the equivalent of the regressor $\overline{\theta}_{\alpha, \lambda_n}$ in \eqref{eq:definition_equiv_theta} reduces to
\begin{equation}
    \overline{\theta}_{\alpha, \lambda_n} = \left(\beta_{\alpha, \lambda_n} \Sigma + \alpha \overline{\Lambda} + \lambda_n \Idd\right)^{-1} \left(\beta_{\alpha, \lambda_n} \Gamma + \alpha \overline{\Omega}\right)\eqsp,
    \quad \text{where} \quad
    \Gamma = \E\left[\varphi(\Xtt_1)\Ytt_1\right]\eqsp.
\end{equation}
Moreover, the equivalent of the generalization error simplifies to
\begin{equation}
    \overline{\mathcal{G}}(\alpha, \lambda_n) = (1 + \delta_{\alpha, \lambda_n}) \left[\big(\overline{\theta}_{\alpha, \lambda_n} - \theta_\star\big)^\top \Sigma \big(\overline{\theta}_{\alpha, \lambda_n} - \theta_\star\big) + \sigma^2 \right] \eqsp.
\end{equation}
This expression makes the role of DA particularly transparent.
In this simplified setting, DA acts as an additional anisotropic regularization through $\alpha \overline{\Lambda}$, on top of the usual isotropic ridge shift $\lambda_n \Idd$.
Hence, DA can improve over standard ridge when $\overline{\Lambda}$ is large along directions that carry little or no predictive signal, 
while remaining small along directions aligned with $\theta_\star$.
Intuitively, the augmentation should inject variability in nuisance directions rather than in informative ones (e.g. directions associated with symetries in the distribution).
Conversely, if $\overline{\Lambda}$ penalizes signal-carrying directions too strongly, the resulting bias increase can dominate the variance reduction and worsen the generalization error.

In what follows, we present intermediate results that we derive to complete the proof of \Cref{thm:main_result} and that we believe are of independent interest.
In particular, we provide in our next result  a deterministic equivalent
for $\hat{\theta}_{\alpha,\lambda_n}$ and $\hat{\theta}_{\alpha,\lambda_n}^{\top}\Sigma\hat{\theta}_{\alpha,\lambda_n}$.
Indeed, expending \eqref{eq:generalization_error}, we get
\begin{align}\label{eq:generalization_error_rewritten}
    \mathcal{G}(\alpha, \lambda_n) = \hat{\theta}_{\alpha, \lambda_n}^\top \Sigma \hat{\theta}_{\alpha, \lambda_n} - 2 \theta_\star^\top \Sigma_\star \hat{\theta}_{\alpha, \lambda_n} + \theta_\star^\top \Sigma_{\star\star} \theta_\star + \sigma^2\eqsp,
\end{align}
The proof of \Cref{thm:main_result} mostly relies on deterministic equivalents for the first two terms in the decomposition above. 
This is the purpose of \Cref{proposition:det_equivs}. 
There, we show that these terms are approximated by deterministic quantities involving the population covariances of both the original and augmented data. 

We first define the deterministic equivalent of the resolvent
\begin{equation}
    \overline{\Resolvent}_{\alpha, \lambda_n}
    =
    \left(\overline{\Sigma}_{\alpha, \lambda_n}
    + \alpha \overline{\Lambda}
    + \lambda_n \Idd\right)^{-1}
    \eqsp,
\end{equation}
as well as the following notation for the quadratic form in \eqref{eq:generalization_error_rewritten} and its deterministic equivalent:
\begin{equation}
    \begin{aligned}
        \hat{\chi}_{\alpha, \lambda_n}
        &=
        \hat{\theta}_{\alpha, \lambda_n}^\top
        \Sigma
        \hat{\theta}_{\alpha, \lambda_n}
        \eqsp, \\
        \overline{\chi}_{\alpha, \lambda_n}
        &=
        \overline{\theta}_{\alpha, \lambda_n}^\top
        \{\overline{\Sigma}'_{\alpha, \lambda_n} + \Sigma\}
        \overline{\theta}_{\alpha, \lambda_n}
        - 2 \overline{\theta}_{\alpha, \lambda_n}^\top
        \overline{\Gamma}'_{\alpha, \lambda_n}
        + \overline{\gamma}_{\alpha, \lambda_n}
        \eqsp.
    \end{aligned}
\end{equation}
Then, it holds:
\begin{proposition} \label{proposition:det_equivs}
    Let $(\lambda_n)_{n \in \N}$, such that for all $n$, $\lambda_n \geq \const/n$, for $\const >0$.
    Assume that \Cref{ass:data_distribution}, \Cref{ass:porportionality}, \Cref{ass:feature_map_and_concentration} and \Cref{ass:artificial_data} hold.
    Then for any $\alpha \in [0,1]$, we have:
    \begin{enumerate}
        \item
            For any $\mathbf{A} \in \R^{p \times p}$
            \begin{equation}
                \left|\tr\left(\mathbf{A} \left\{\Resolvent_{\alpha, \lambda_n} - \overline{\Resolvent}_{\alpha, \lambda_n}\right\}\right)\right| \stochdom \dfrac{ \|\mathbf{A}\|_\op}{\lambda_n^{7/2} \sqrt{n}}\eqsp.
            \end{equation}
        \item \label{proposition:det_equivs-item1}
            For any $\mathbf{a} \in \R^{p}$
            \begin{equation}
                \left|\mathbf{a}^\top \{\hat{\theta}_{\alpha, \lambda_n} - \overline{\theta}_{\alpha, \lambda_n}\} \right| \stochdom \dfrac{ \|\mathbf{a}\|_2}{\lambda_n^{4} \sqrt{n}}\eqsp,
            \end{equation}
            where $\overline{\theta}_{\alpha, \lambda_n}$ was defined in \eqref{eq:definition_equiv_theta}.
        \item \label{proposition:det_equivs-item2}
            Finally
            \begin{equation}
                \left|\hat{\chi}_{\alpha, \lambda_n} - \overline{\chi}_{\alpha, \lambda_n}\right| \stochdom \dfrac{1}{\lambda_n^{9/2} \sqrt{n}}\eqsp.
            \end{equation}
    \end{enumerate}
\end{proposition}
The previous result provides a deterministic equivalent for resolvents of sample covariances matrices with non-independent samples,
as for all $i$, the pair $(\varphi(\Xtt_i), \mu_{\xtt}(\Ztt_i))$ is allowed to have an arbitrary joint distribution.
To our knowledge, such relaxation of the independance assumption has only been considered in \cite{MorissetHardyDurmus2025},
also in the context of data-augmentation.


\section{Simulations} \label{sec:simulations}

In this section, we validate \Cref{thm:main_result} and \Cref{proposition:det_equivs} on both synthetic and real data.
We first use controlled synthetic datasets to assess the finite-sample accuracy of the deterministic equivalents (\Cref{sec:simulations_synthetic}), then turn to a practical inpainting task on MNIST (\Cref{sec:additional_simulation_inpainting_mnist}).

\subsection{Synthetic Data}\label{sec:simulations_synthetic}

We use synthetic datasets to assess the finite-sample accuracy of
\Cref{thm:main_result} and \Cref{proposition:det_equivs}.
We find that the proposed deterministic equivalents closely match their
empirical counterparts across a broad range of dimensions and sample sizes,
provided the regularization level $\lambda_n$ is taken sufficiently large, which is consistent
with the requirements suggested by our error bounds.
More unexpectedly, and contrary to common intuition, our simulations show that DA
impacts the model's bias in a non-monotonic way, under feature misspecification.
In particular, when the augmentation scheme substantially alters the covariate structure
of the distribution, DA can avoid (or even mitigate) the bias increase one might anticipate.
This contrasts with prior results in the correctly specified setting,
where stronger regularization is shown to monotonically increase bias \cite{pmlr-v195-bosch23a, MeiMontanari2022GenErrorRF, LinKaushikDyerMuthukumar2024}.

In all experiments, we consider a synthetic regression problem in which the training data
$(\Xtt_i,\Ytt_i)_{i=1}^n$ are generated as follows. The covariates are i.i.d.\ Gaussian,
$\Xtt_i \sim \mathcal{N}(0,\Sigma_\Xtt)$, and the responses are generated according to
\Cref{ass:data_distribution}. To mimic the strong anisotropy often observed in high-dimensional
learning, we take $\Sigma_\Xtt \;=\; Q\,\mathrm{Diag}\!\bigl((k^{-1})_{k=1}^{1000}\bigr)\,Q^\top$,
i.e., a power-law spectrum, where $Q$ is Haar-distributed on the orthogonal group of $\R^{1000}$.

To highlight the effects of feature misspecification, we use $\varphi_\star:\mathbb{R}^{1000}\to\mathbb{R}^{700}$ to generate the labels, and $\varphi:\mathbb{R}^{1000}\to\mathbb{R}^{700}$
to predict the labels, as randomly initialized MLPs, of respective hidden layer sizes $[1000,1000,1000]$ and $[500, 500]$.
Note that this choice of feature map is made to ensure strong misspecification.
We set the ground-truth parameter to $\theta_\star = (1,\ldots,1)^\top/\sqrt{700} \in \mathbb{R}^{700}$
and add Gaussian label noise $\varepsilon \sim \mathcal{N}(0,0.5)$.

We consider a \emph{salt-and-pepper} augmentation scheme, defined by
\begin{align}
    \tau_{\Xtt}(x,y) \;=\; \mathfrak{m}\odot x \;+\; (1-\mathfrak{m})\odot \eta,
    \qquad
    \tau_{\Ytt}(x,y) \;=\; y \eqsp,
\end{align}
where $\mathfrak{m}\in\{0,1\}^{1000}$ is a random mask of Bernoulli random variables with parameter $0.2$,
and $\eta$ is an independent isotropic noise vector of variance $0.5$.
Although simple, this transformation substantially alters the features, and yields non-trivial functions $\mu_\xtt$ and $\Lambda$.
These quantities do not admit a simple closed form for nonlinear $\varphi$ and are therefore
computed via Monte--Carlo estimation in our experiments.

\begin{figure*}
    \centering
    \includegraphics[width=0.98\textwidth,height=0.16\textheight]{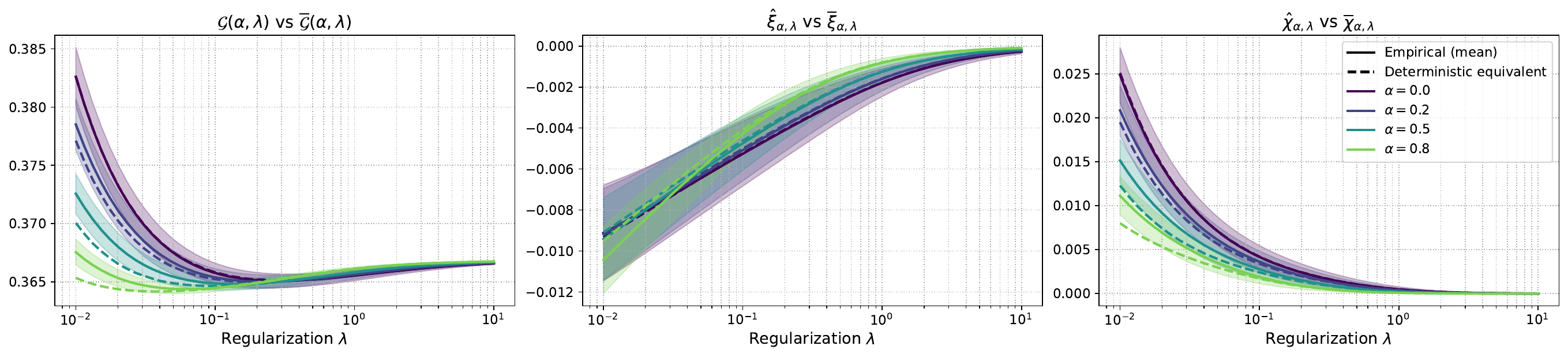}
    \caption{\textbf{Deterministic equivalents accuracy for various $\lambda$ and $\alpha$.}
    Comparison between empirical estimates (solid lines) and deterministic equivalents (dashed lines) as the regularization $\lambda$ varies, for multiple augmentation strengths $\alpha$ (colors).
    Shaded bands represent $\pm 1$ empirical standard deviation over our Monte-Carlo approximations of $\E\left[\mathcal{G}(\alpha,\lambda)\right]$, $\E\left[\xi_{\alpha, \lambda}\right]$ and $\E\left[\chi_{\alpha, \lambda}\right]$.}
    \label{fig:1}
\end{figure*}

Our setup also allows us to disentangle how DA affects the bias and variance contributions to generalization.
Indeed, the generalization error can be decomposed as
\begin{align} \label{eq:generalization_error_bias_variance}
    \E\left[\mathcal{G}(\alpha, \lambda_n)\right] = \mathrm{bias}^2(\alpha, \lambda_n) + \mathrm{V}(\alpha, \lambda_n) \eqsp,
\end{align}
where
\begin{equation}
    \mathrm{bias}^2(\alpha, \lambda_n) = \E\left[\left(\E\left[\hat{\theta}_{\alpha, \lambda_n}\right]^\top \varphi(X) - \theta_\star^\top \varphi_\star(X) \right)^2\right] \eqsp,
\end{equation}
which rewrites as
\begin{align}
    \mathrm{bias}^2(\alpha, \lambda_n) = \E[\hat{\theta}_{\alpha, \lambda_n}]^\top \Sigma \E[\hat{\theta}_{\alpha, \lambda_n}]  + \theta_\star^\top \Sigma_{\star\star} \theta_\star 
    - 2 \theta_\star^\top \Sigma_\star \E[\hat{\theta}_{\alpha, \lambda_n}] \eqsp.
\end{align}
In light of \Cref{proposition:det_equivs-item1} of \Cref{proposition:det_equivs}, we thus introduce
\begin{align}
    \overline{\mathrm{bias}}^2(\alpha, \lambda_n) = \overline{\theta}_{\alpha, \lambda_n}^\top \Sigma \overline{\theta}_{\alpha, \lambda_n}  + \theta_\star^\top \Sigma_{\star\star} \theta_\star - 2 \theta_\star^\top \Sigma_\star \overline{\theta}_{\alpha, \lambda_n} \eqsp,
\end{align}
ensuring that for a large enough constant $\const > 0$
\begin{align}
    \left|\mathrm{bias}^2(\alpha, \lambda_n) - \overline{\mathrm{bias}}^2(\alpha, \lambda_n)\right| \leq \dfrac{\const}{\lambda_n^{7/2} \sqrt{n}} \eqsp,
\end{align}

Note that the variance term $\mathrm{V}(\alpha, \lambda_n)$ does not admit a simple closed-form expression,
still it is estimated as the remainder in \eqref{eq:generalization_error_bias_variance},
where both the generalization error and the bias are estimated thanks to Monte-Carlo approximations.

\begin{figure*}
    \centering
    \includegraphics[width=\textwidth,height=0.16\textheight]{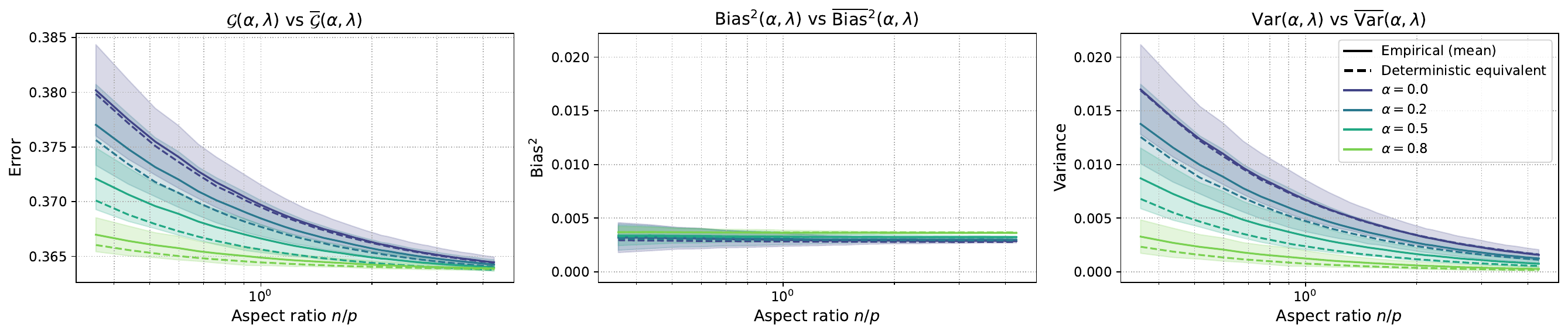}
    \caption{\textbf{Bias--variance decomposition for various aspect ratios.}
    Empirical (solid) and deterministic-equivalent (dashed) learning curves as the aspect ratio $n/p$ varies (log scale), at a fixed regularization level $\lambda = 0.05$.
    Shaded bands show $\pm 1$ empirical standard deviation over our Monte-Carlo approximations of $\mathcal{G}(\alpha, 0.05)$, $\mathrm{bias}^2(\alpha, 0.05)$ and $\mathrm{V}(\alpha, 0.05)$.}
    \label{fig:2}
\end{figure*}

\Cref{fig:1} reports the accuracy of our deterministic equivalents for the generalization error
$\mathcal{G}(\alpha,\lambda_n)$ and for the auxiliary quantities $\xi_{\alpha,\lambda_n} = \theta_\star^\top \Sigma_\star \hat{\theta}_{\alpha,\lambda_n}$ and $\chi_{\alpha,\lambda_n}$
introduced in \Cref{thm:main_result} and \Cref{proposition:det_equivs}. Overall, the deterministic predictions closely
track their empirical counterparts across the range of regularization parameters considered.
Consistent with the error bounds, the approximation deteriorates as $\lambda_n \to 0$. For this reason, we restrict
attention to $\lambda_n \ge 10^{-2}$, where the match is uniformly tight. In this regime, the proposed salt-and-pepper
augmentation yields a clear improvement over the non-augmented baseline across all $\lambda_n \in [10^{-2},1]$.

\Cref{fig:2} illustrates how augmentation impacts the bias and variance components of the prediction error,
across multiple choices of $\alpha$ and the aspect ratio $p/n$, at the fixed regularization level $\lambda_n = 0.05$.
As expected, increasing the diversity of samples through DA reduces the variance term,
consistent with the view of data augmentation as an implicit regularizer.
More surprisingly, we do not observe a systematic increase of the bias, even though the augmented covariates can differ
substantially from the original ones. Across a broad range of
practitioner-relevant augmentation strengths, the bias remains largely unchanged, and the overall gain is driven primarily by
variance reduction. A marked bias increase appears only when the augmentation is pushed to unrealistically strong regimes.

This contrasts with previous observations in correctly specified settings, where stronger regularization effects are
typically accompanied by increased bias, and only appears in misspecified feature models: since the estimator is already biased due to mismatch
between $\varphi$ and $\varphi_\star$, the additional bias induced by augmentation can be negligible compared to this intrinsic
misspecification bias. These results indicate that DA may
exhibit little to no bias--variance trade-off, acting predominantly as a variance-reduction mechanism.

\subsection{Inpainting on MNIST}\label{sec:additional_simulation_inpainting_mnist}

We now assess the finite-sample accuracy of
\Cref{thm:main_result} and \Cref{proposition:det_equivs}
on an inpainting task built from the MNIST dataset.
Overall, we find that the proposed deterministic equivalents
closely track their empirical counterparts across a broad range
of aspect ratios, not only at the level of the test risk, but
also for its bias--variance decomposition.

We use the standard MNIST split, composed of $60{,}000$ training images
and $10{,}000$ test images, each of size $28\times 28$.
All pixel intensities are rescaled to $[0,1]$.
For each image $I_i\in\R^{28\times 28}$, we remove a centered square patch
of size $5\times 5$ and aim at reconstructing the missing pixels from the
remaining ones.
This yields a multivariate regression problem in which the covariates
$X_i\in\R^{759}$ are the visible pixels and the responses
$Y_i\in\R^{25}$ are the masked central pixels.
Representative masked inputs are shown in
\Cref{fig:mnist_appendix_mask_examples}.

\begin{figure*}[h!]
    \centering
    \begin{subfigure}[t]{0.48\textwidth}
        \centering
        \includegraphics[width=\linewidth,trim=0 514.7bp 0 0,clip]{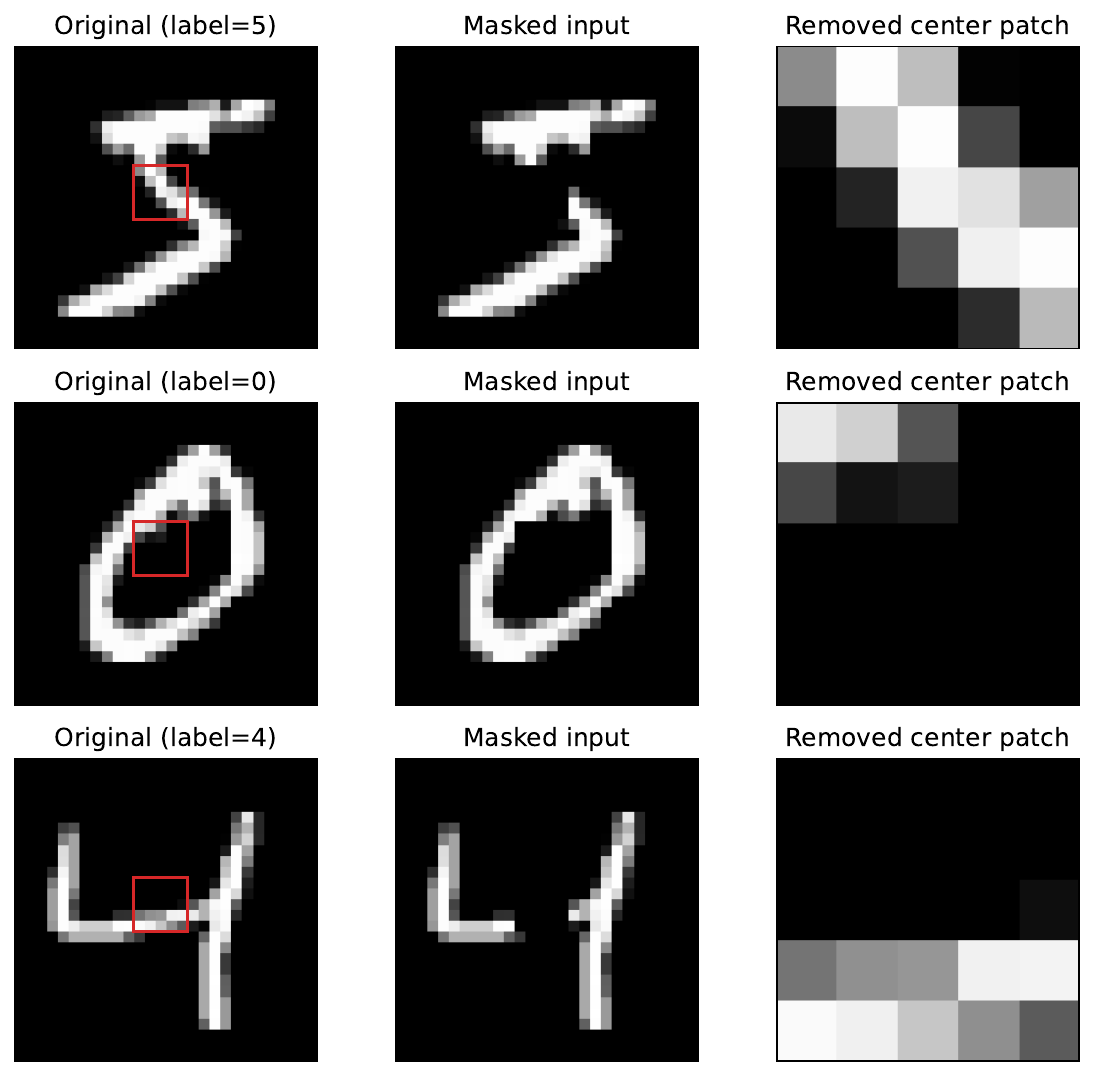}
        \caption{First three examples.}
    \end{subfigure}
    \hfill
    \begin{subfigure}[t]{0.48\textwidth}
        \centering
        \includegraphics[width=\linewidth,trim=0 0 0 514.7bp,clip]{figures/figure_mnist_mask_examples.pdf}
        \caption{Last three examples.}
    \end{subfigure}
    \caption{\textbf{Masked MNIST examples used in the inpainting task.}
    Each panel displays original digits, their masked inputs, and the removed center patches.
    The original tall figure has been split into two halves so that three examples appear in each panel.}
    \label{fig:mnist_appendix_mask_examples}
\end{figure*}

Although our theoretical developments were stated for scalar responses,
the extension to the present vector-valued setting is immediate.
Indeed, for any feature map $\varphi:\R^{759}\to\R^p$, the multivariate
ridge estimator is obtained coordinatewise:
\begin{align}
    \widehat{\Theta}_{\alpha,\lambda_n}
    &\in \argmin_{\Theta\in\R^{p\times 25}}
    \frac{1}{n}\sum_{i=1}^n
    \bigl\|Y_i-\Theta^\top\varphi(X_i)\bigr\|_2^2
    + \lambda_n \|\Theta\|_{\mathrm F}^2 = \bigl[\widehat{\theta}^{(1)}_{\alpha,\lambda_n},\ldots,
    \widehat{\theta}^{(25)}_{\alpha,\lambda_n}\bigr],
\end{align}
where, for each coordinate $j\in\{1,\ldots,25\}$,
$\widehat{\theta}^{(j)}_{\alpha,\lambda_n}\in\R^p$
is the scalar ridge estimator obtained by regressing the $j$-th missing
pixel on $\varphi(X_i)$.
Accordingly, the deterministic equivalents derived in
\Cref{thm:main_result} and \Cref{proposition:det_equivs} apply coordinatewise, and
the reported prediction error is the output-averaged test risk
\begin{equation}
    \mathcal{G}(\alpha,\lambda_n)
    := \frac{1}{25}\,
    \E\Bigl[\bigl\|Y-\widehat{\Theta}_{\alpha,\lambda_n}^{\!\top}\varphi(X)\bigr\|_2^2\Bigr] 
    = \dfrac{1}{25} \sum_{j=1}^{25} \E\left[\left(Y_j-\widehat{\theta}^{(j)\top}_{\alpha,\lambda_n}\varphi(X)\right)^2\right].
\end{equation}

We consider three label-preserving augmentation schemes, written in the
notation of the previous sections as transformations
$(\tau_{\xtt},\tau_{\ytt})$ with $\tau_{\ytt}(z,\eta)=y$ for
$z=(x,y)\in\R^{759}\times\R^{25}$:
\begin{enumerate}
    \item \emph{Noise injection:}
    $\tau_{\xtt}(z,\eta) = x + \sigma_{\mathrm{aug}} \eta$,
    where $\eta\sim\loiGauss(0,\Id{759})$.
    \item \emph{Random masking:}
    $\tau_{\xtt}(z,\eta) = x \odot m$, where the entries of
    $m$ are i.i.d. Bernoulli variables with keep probability parameter $q=0.85$.
    \item \emph{Salt-and-pepper noise:}
    $\tau_{\xtt}(z,\eta) = x \odot m + (1-m)\odot \xi$,
    where again $m$ has i.i.d. Bernoulli$(q)$ entries with $q=0.85$,
    and $\xi\sim\loiGauss(0,\sigma_{\mathrm{aug}}^2 \Id{759})$.
\end{enumerate}
In all experiments, we set $\sigma_{\mathrm{aug}}=0.25$,
$\lambda_n = 10^{-3}$, and
$\alpha\in\{0,0.25,0.5,0.75\}$, and choose
$\varphi = \mathrm{id}_{759}$, so that $p=759$.
We additionally report the same diagnostics in a second setting with a
nontrivial feature map; the corresponding plots are labeled
``with feature map'' below.
We vary the aspect ratio $p/n$ by subsampling the training set. 
The population quantities entering the deterministic equivalents are
estimated once from the full MNIST training set, before subsampling.
We also estimate the
bias--variance decomposition of the test risk at fixed aspect ratio.
This allows us to disentangle how data augmentation affects the bias
and variance contributions to generalization.

\begin{figure*}[t]
    \centering
    \begin{subfigure}[t]{\textwidth}
        \centering
        \makebox[\linewidth][c]{\includegraphics[width=1.1\linewidth]{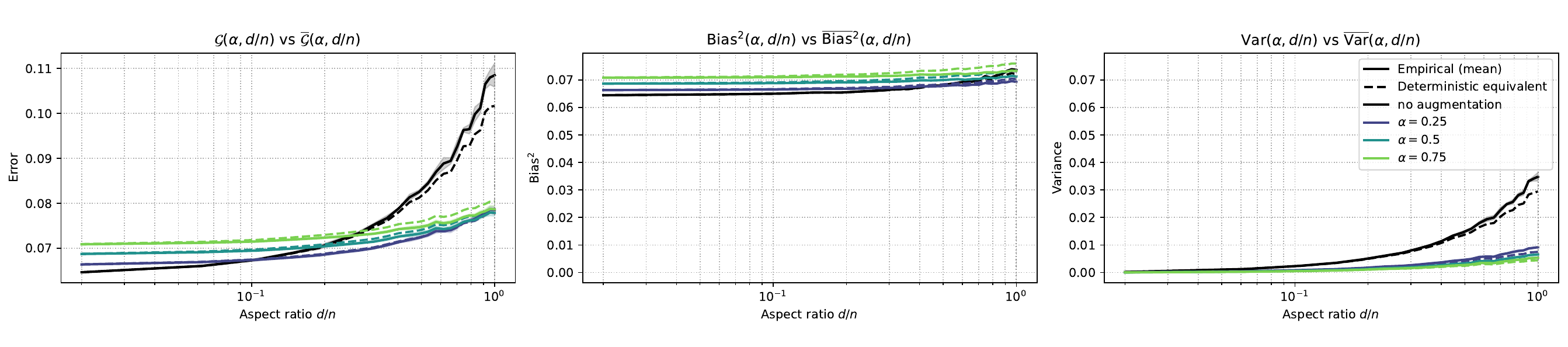}}
        \caption{Noise injection.}
    \end{subfigure}

    \vspace{0.8em}
    \begin{subfigure}[t]{\textwidth}
        \centering
        \makebox[\linewidth][c]{\includegraphics[width=1.1\linewidth]{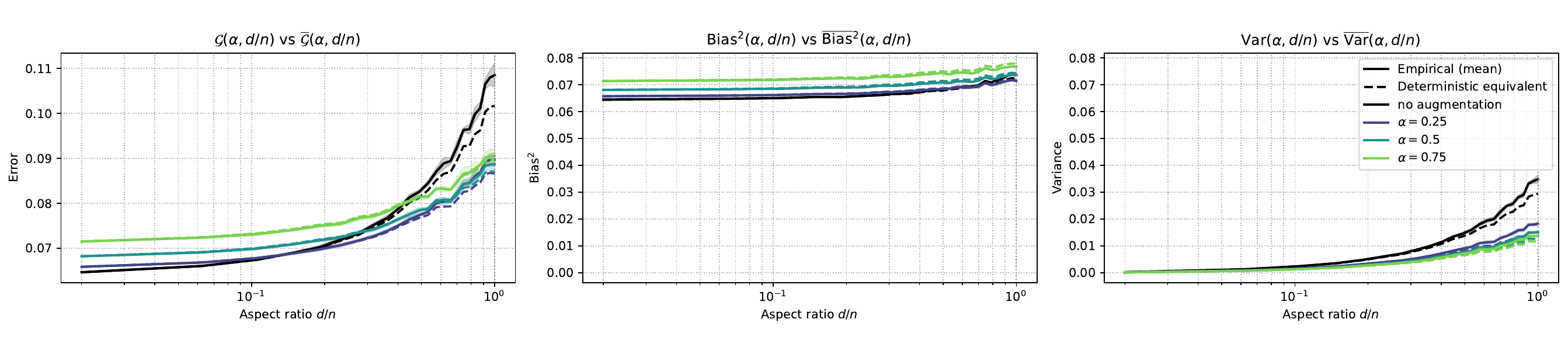}}
        \caption{Random masking.}
    \end{subfigure}

    \vspace{0.8em}
    \begin{subfigure}[t]{\textwidth}
        \centering
        \makebox[\linewidth][c]{\includegraphics[width=1.1\linewidth]{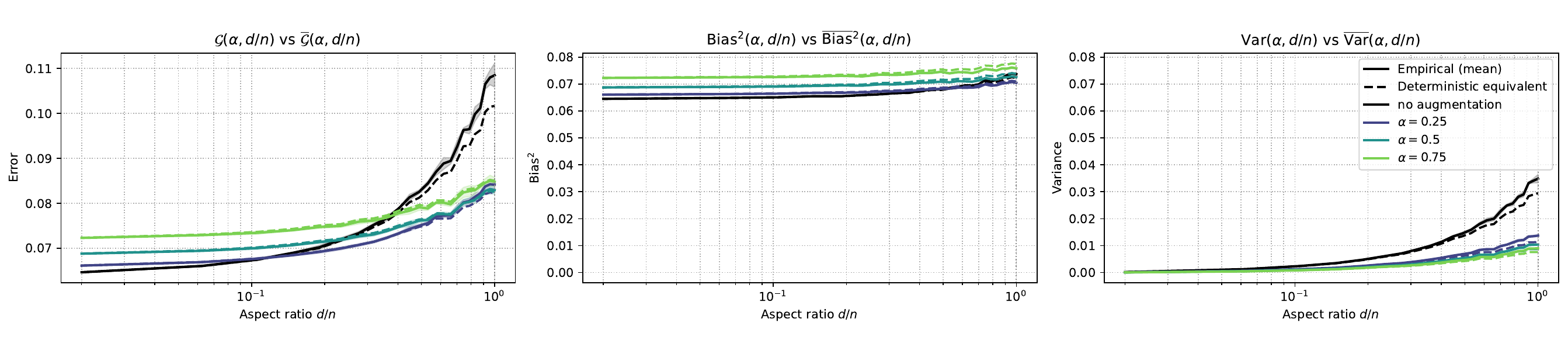}}
        \caption{Salt-and-pepper noise.}
    \end{subfigure}
    \caption{\textbf{MNIST inpainting with the identity feature map.}
    Empirical estimates and deterministic equivalents for the test error, bias, and variance as functions of the aspect ratio, for the three augmentation schemes considered in the main text and several augmentation strengths $\alpha$.}
    \label{fig:mnist_appendix_identity_curves}
\end{figure*}

More precisely, for each output coordinate $j\in\{1,\ldots,25\}$, we
apply the theory of
\Cref{thm:main_result} and \Cref{proposition:det_equivs} to the each coordinate in $Y_i$, i.e. the regressions of
$Y_{i,j}$ on $\varphi(X_i)$, and denote the corresponding quantities by
$\mathcal{G}^{(j)}(\alpha,\lambda_n)$,
$\widehat{\theta}_{\alpha,\lambda_n}^{(j)}$ and
$\overline{\theta}_{\alpha,\lambda_n}^{(j)}$.
Then,
$\E[\mathcal{G}^{(j)}(\alpha,\lambda_n)]
    = \mathrm{bias}_j^2(\alpha,\lambda_n)
    + \mathrm{V}_j(\alpha,\lambda_n)$,
where
\begin{align}
    \mathrm{bias}_j^2(\alpha,\lambda_n)
    &= \E\left[\left(
    \E\left[\widehat{\theta}_{\alpha,\lambda_n}^{(j)}\right]^\top
    \varphi(X) - {\theta_\star^{(j)}}^\top \varphi_\star(X)
    \right)^2\right] \nonumber \\
    &= \E\left[\widehat{\theta}_{\alpha,\lambda_n}^{(j)}\right]^\top
    \Sigma \E\left[\widehat{\theta}_{\alpha,\lambda_n}^{(j)}\right]
    + {\theta_\star^{(j)}}^\top \Sigma_{\star\star}^{(j)} \theta_\star^{(j)}
    - 2 {\theta_\star^{(j)}}^\top \Sigma_\star^{(j)}
    \E\left[\widehat{\theta}_{\alpha,\lambda_n}^{(j)}\right] \eqsp,
\end{align}
and its deterministic equivalent is
\begin{align}
    \overline{\mathrm{bias}}_j^2(\alpha,\lambda_n)
    &= {\overline{\theta}_{\alpha,\lambda_n}^{(j)}}^\top
    \Sigma \overline{\theta}_{\alpha,\lambda_n}^{(j)}
    + {\theta_\star^{(j)}}^\top \Sigma_{\star\star}^{(j)} \theta_\star^{(j)}
    - 2 {\theta_\star^{(j)}}^\top \Sigma_\star^{(j)}
    \overline{\theta}_{\alpha,\lambda_n}^{(j)} \eqsp.
\end{align}
We define $\mathrm{V}_j(\alpha,\lambda_n)$ and its deterministic equivalent as the remainders between the generalization error and the above bias quantities, 
and additionally report $\mathrm{bias}^2(\alpha, \lambda_n) = \fraca{1}{25} \sum_{j=1}^{25} \mathrm{bias}_j^2(\alpha, \lambda_n)$ and 
$\mathrm{V}(\alpha, \lambda_n) = \fraca{1}{25} \sum_{j=1}^{25} \mathrm{V}_j(\alpha, \lambda_n)$.
The aspect-ratio curves for the identity feature map are reported in
\Cref{fig:mnist_appendix_identity_curves}. They show that the
deterministic equivalents remain close to the empirical error, bias, and
variance curves, and that the main effect of augmentation is a reduction
of the variance term.

\begin{figure*}[h!]
    \centering
    \begin{subfigure}[t]{\textwidth}
        \centering
        \makebox[\linewidth][c]{\includegraphics[width=1.1\linewidth]{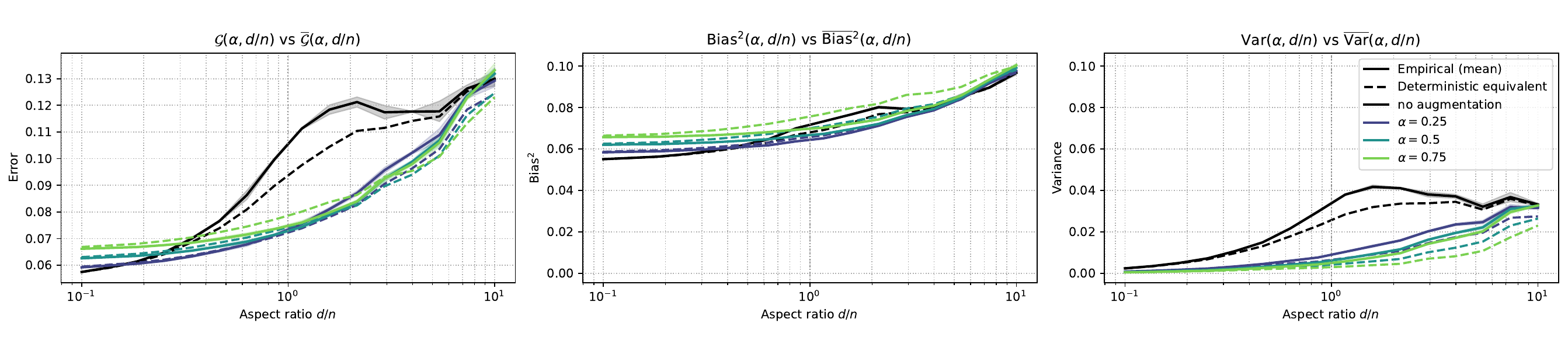}}
        \caption{Noise injection with feature map.}
    \end{subfigure}

    \vspace{0.8em}
    \begin{subfigure}[t]{\textwidth}
        \centering
        \makebox[\linewidth][c]{\includegraphics[width=1.1\linewidth]{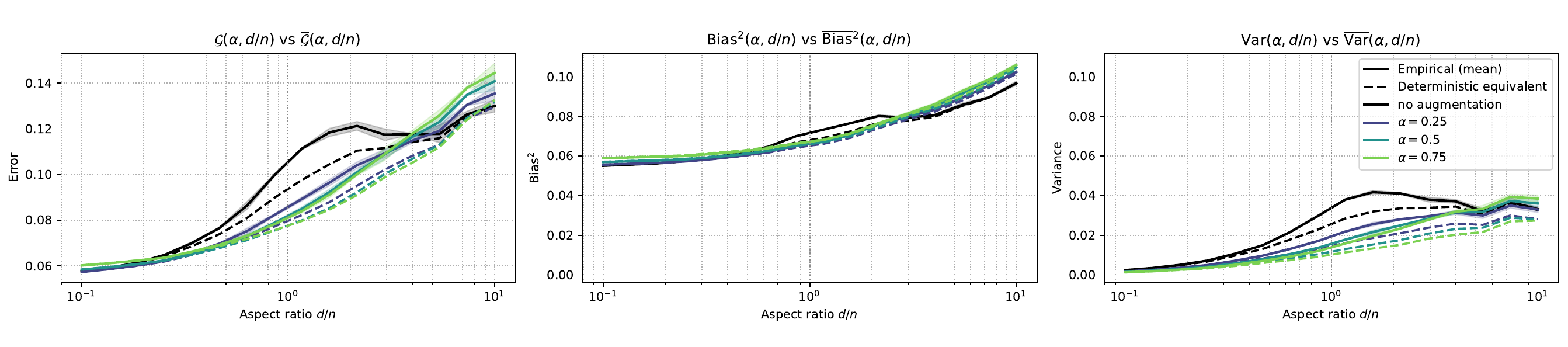}}
        \caption{Random masking with feature map.}
    \end{subfigure}

    \vspace{0.8em}
    \begin{subfigure}[t]{\textwidth}
        \centering
        \makebox[\linewidth][c]{\includegraphics[width=1.1\linewidth]{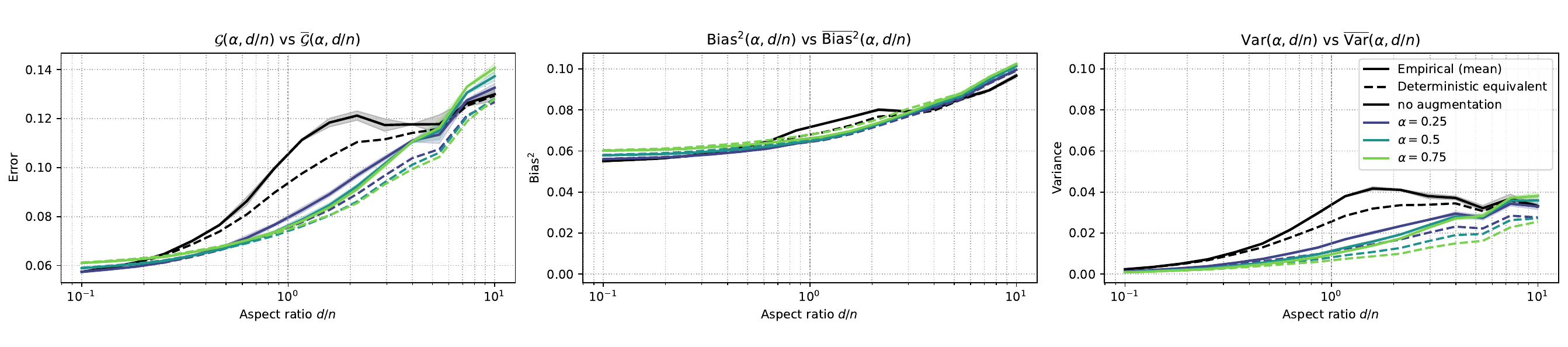}}
        \caption{Salt-and-pepper noise with feature map.}
    \end{subfigure}
    \caption{\textbf{MNIST inpainting with feature map.}
    Same comparison as in \Cref{fig:mnist_appendix_identity_curves}, after replacing the raw-pixel representation by a feature map.}
    \label{fig:mnist_appendix_fm_curves}
\end{figure*}

The corresponding experiments with feature map are displayed in
\Cref{fig:mnist_appendix_fm_curves}. The same qualitative behavior
persists in this second setting, with close agreement between empirical
quantities and deterministic equivalents across the three augmentation
schemes.

We finally note that the quality of the fit is excellent in the identity
feature-map setting, where the deterministic equivalents are nearly
indistinguishable from their empirical counterparts across the range of
aspect ratios considered. In the ``with feature map'' setting, the fit
remains qualitatively accurate, although the discrepancy becomes more
visible around the peak of the error curves. We notice on this practical setting 
that all DA schemes improve over the non-augmented Ridge baseline when the aspect ratio is large, 
i.e. in a regime where data is scarce, with the gain driven mostly by
variance reduction and only a limited increase in bias, that seems not to depend on the aspect ratio. 

\section{Conclusion}

We developed a high-dimensional theory for supervised feature regression trained with
arbitrary data augmentation, allowing for substantial perturbations of the covariates.
In the proportional regime, we derived deterministic equivalents for the augmented
ridge estimator and its out-of-sample error, together with explicit non-asymptotic
error controls in terms of $n$ and the ridge parameter $\lambda$. Technically,
our analysis relies on resolvent methods that accommodate the non-i.i.d.\
structure induced by augmentation via the coupled objects $(\varphi(\Xtt_i),\mu_\xtt(\Ztt_i))$.

Beyond prediction formulas, our results provide a mechanistic view of how
augmentation changes learning under feature misspecification:
it alters the effective covariance through a mean-feature component
and an additional, generally anisotropic variance term.
Our simulations validate the accuracy of the deterministic equivalents
and suggest that, in misspecified feature models, augmentation can improve
generalization over a wide range of $\lambda$ while not necessarily increasing
the bias, so that gains are driven primarily by variance reduction.

Yet, our results cannot be directly applied by practitioners to predict when
a DA scheme is beneficial in the finite-sample setting. A key open direction is to turn our population-level formulas into a practical,
data-driven and quantitative criterion that could be used by practitioners.

\bibliographystyle{plainnat}
\bibliography{biblio}

\appendix
\section{Sufficient conditions for \Cref{ass:artificial_data}}\label{sec:artificial_data_examples}

This section explains how \Cref{ass:artificial_data} can be verified 
in practice by reducing it to checkable sufficient conditions. 
We start with a generic proposition showing how Lipschitz concentration of the augmentation 
randomness yields bounds on the empirical second-order quantities associated with the scheme. 
We then complement this result with explicit augmentation procedures for which the assumption can be checked directly.
\begin{proposition}\label{prop:H4_bound_sufficient_cond}
    For every $n \in \N$, let $\varphi:\R^{d(n)} \to \R^{p(n)}$, $\tau_{\xtt} : \R^{d(n)} \times \R \times \mse \to \R^{d(n)}$
    and $\tau_{\ytt} : \R^{d(n)} \times \R \times \mse \to \R$. Assume that the augmentation variable $\eta$ is $1$-Lipschitz concentrated, and that there exist constants
    $\Ltt_{\xtt}, \Ltt_{\ytt} > 0$, independent of $n$, such that for all
    $z \in \R^{d(n)} \times \R$ and all $\etabf, \etabf' \in \mse$
    \begin{align}
        \|\varphi(\tau_{\xtt}(z,\etabf)) - \varphi(\tau_{\xtt}(z,\etabf'))\|_2 &\leq \Ltt_{\xtt}\|\etabf-\etabf'\|_2 \eqsp, \\
        |\tau_{\ytt}(z,\etabf) - \tau_{\ytt}(z,\etabf')| &\leq \Ltt_{\ytt}\|\etabf-\etabf'\|_2  \eqsp.
    \end{align}
    Then
    \begin{align}
        \Var(\Lambda(\Ztt)) \leq \const \dfrac{p}{n} \Ltt_{\xtt}^4 \eqsp, \quad \Var(\Omega(\Ztt)) \leq \const \dfrac{1}{n} \Ltt_{\xtt}^2 \Ltt_{\ytt}^2  \eqsp.
    \end{align}
\end{proposition}
\begin{proof}
    For all $z \in \R^{d(n)} \times \R$ we write
    \[
        X_z = \varphi(\tau_{\xtt}(z,\eta)),
        \qquad
        Y_z = \tau_{\ytt}(z,\eta).
    \]
    The Lipschitz assumptions above are with respect to the deterministic augmentation input $\etabf$.
    Combining them with the Lipschitz concentration of the random variable $\eta$,
    conditionally on any value of $z$, every projection $\bfu^\top X_z$, 
    with $\|\bfu\|_2 = 1$, is sub-Gaussian with parameter of order $\Ltt_{\xtt}$. 
    Similarly, $Y_z$ is sub-Gaussian with parameter of order $\Ltt_{\ytt}$. 
    Therefore, using the standard variance bound for sub-Gaussian random variables, 
    for every unit vector $\bfu$, we get
    \begin{align}
        \Var(\bfu^\top X_{\Ztt_1}\mid \Ztt_1)
        &\leq \const \Ltt_{\xtt}^2, \label{eq:variance_bound_xtt}\\
        \Var(Y_{\Ztt_1}\mid \Ztt_1)
        &\leq \const \Ltt_{\ytt}^2, \label{eq:variance_bound_ytt}
    \end{align}
    almost surely.
    We first control $\Lambda(\Ztt)$. By independence of the $\Ztt_i$,
    \begin{align}
        \Var(\Lambda(\Ztt))
        &= \E\left[\left\| \Lambda(\Ztt) - \E[\Lambda(\Ztt)] \right\|_\Frob^2\right] 
        = \dfrac{1}{n}
        \E\left[\left\| \Lambda(\Ztt_1) - \E[\Lambda(\Ztt_1)] \right\|_\Frob^2\right] \\
        &\leq \dfrac{\const}{n}
        \E\left[\left\| \Lambda(\Ztt_1) \right\|_\Frob^2\right] 
        \leq \dfrac{\const p}{n}
        \E\left[\left\| \Var(X_{\Ztt_1}\mid \Ztt_1) \right\|_\op^2\right].
    \end{align}
    Moreover,
    \begin{align}
        \left\| \Var(X_{\Ztt_1}\mid \Ztt_1) \right\|_\op
        &= \sup_{\|\bfu\|_2=1}
        \Var(\bfu^\top X_{\Ztt_1}\mid \Ztt_1)
        \leq \const \Ltt_{\xtt}^2,
    \end{align}
    by~\eqref{eq:variance_bound_xtt}. Therefore
    \[
        \Var(\Lambda(\Ztt))
        \leq \const \dfrac{p}{n} \Ltt_{\xtt}^4 .
    \]
    We now control $\Omega(\Ztt)$. The same independence argument gives
    \begin{align}
        \Var(\Omega(\Ztt))
        &\leq \dfrac{\const}{n}
        \E\left[
        \left\|
        \Cov(X_{\Ztt_1},Y_{\Ztt_1}\mid \Ztt_1)
        \right\|_2^2
        \right] \\
        &= \dfrac{\const}{n}
        \E\left[
        \left(
        \sup_{\|\bfu\|_2=1}
        \Cov(\bfu^\top X_{\Ztt_1},Y_{\Ztt_1}\mid \Ztt_1)
        \right)^2
        \right] \eqsp.
    \end{align}
    By Cauchy--Schwarz,~\eqref{eq:variance_bound_xtt}, and~\eqref{eq:variance_bound_ytt},
    \[
        \Cov(\bfu^\top X_{\Ztt_1},Y_{\Ztt_1}\mid \Ztt_1)^2
        \leq
        \const \Ltt_{\xtt}^2 \Ltt_{\ytt}^2
    \]
    uniformly over $\|\bfu\|_2=1$. Hence
    \[
        \Var(\Omega(\Ztt))
        \leq \const \dfrac{1}{n} \Ltt_{\xtt}^2 \Ltt_{\ytt}^2 .
    \]
\end{proof}

\paragraph{Concrete examples.}
We now present several augmentation schemes for which \Cref{ass:artificial_data} can be checked explicitly.
These examples are meant to illustrate representative mechanisms used in practice, while \Cref{prop:H4_bound_sufficient_cond} applies more broadly.
For simplicity, all three examples are written in the identity feature map setting, namely $p(n)=d(n)$ and $\varphi(x)=x$ for all $x\in\R^{d(n)}$.
This choice is made to keep the expressions of $\mu_{\xtt}$, $\mu_{\ytt}$, $\Lambda$ and $\Omega$ explicit; 
the same arguments extend to non-identity feature maps whenever the assumptions of \Cref{prop:H4_bound_sufficient_cond} are satisfied.

\begin{example}[Heteroskedastic noise injection]
    Fix $k\in\N$, independent of $n$, and let $\eta$ be a centered random vector in $\R^k$, independent of $\Ztt$, such that
    \begin{align}
        \E\left[\eta\eta^\top\right] = \Id{k} \eqsp.
    \end{align}
    Let
    \begin{align}
        s_{\xtt} &: \R^{d(n)} \times \R \to \R^{d(n)\times k} \eqsp, &
        s_{\ytt} &: \R^{d(n)} \times \R \to \R^{1\times k}
    \end{align}
    be maps satisfying, for some constants $\Ltt_s,B>0$ independent of $n$,
    \begin{align}
        \|s_{\xtt}(z)-s_{\xtt}(z')\|_\Frob + \|s_{\ytt}(z)-s_{\ytt}(z')\|_2 &\leq \Ltt_s\|z-z'\|_2 \eqsp, \\
        \|s_{\xtt}(z)\|_\op + \|s_{\ytt}(z)\|_2 &\leq B \eqsp
    \end{align}
    for all $z,z' \in \R^{d(n)}\times\R$.
    For $z=(x,y)\in\R^{d(n)}\times\R$ and deterministic $\etabf\in\R^k$, define
    \begin{align}
        \tau_{\xtt}(z,\etabf) &= x + s_{\xtt}(z)\etabf \eqsp, &
        \tau_{\ytt}(z,\etabf) &= y + s_{\ytt}(z)\etabf \eqsp,
    \end{align}
    so that the random augmentation is obtained by evaluating these maps at $\etabf=\eta$.
    Since $\eta$ is centered with identity covariance, we get
    \begin{align}
        \mu_{\xtt}(z) &= x \eqsp, &
        \mu_{\ytt}(z) &= y \eqsp, \\
        \Lambda(z) &= s_{\xtt}(z)s_{\xtt}(z)^\top \eqsp, &
        \Omega(z) &= s_{\xtt}(z)s_{\ytt}(z)^\top \eqsp.
    \end{align}
    The required regularity bounds follow directly from those of $s_{\xtt}$ and $s_{\ytt}$. Indeed,
    \begin{align}
        \|\Lambda(z)-\Lambda(z')\|_\Frob
        &\leq \|s_{\xtt}(z)\|_\op \|s_{\xtt}(z)-s_{\xtt}(z')\|_\Frob + \|s_{\xtt}(z')\|_\op \|s_{\xtt}(z)-s_{\xtt}(z')\|_\Frob \\
        &\leq 2B\Ltt_s \|z-z'\|_2 \eqsp,
    \end{align}
    and similarly $\|\Omega(z)-\Omega(z')\|_2 \leq 2B\Ltt_s\|z-z'\|_2$.
    Moreover, since $k$ is fixed,
    \begin{align}
        \|\Lambda(z)\|_\Frob \leq \sqrt{k}\,B^2 \eqsp, \qquad \|\Omega(z)\|_2 \leq B^2 \eqsp,
    \end{align}
    and therefore $\Var(\Lambda(\Ztt_1)) + \Var(\Omega(\Ztt_1)) \leq \const$.
    Therefore \Cref{ass:artificial_data} holds.
\end{example}

\begin{example}[Salt-and-pepper noise injection]
    We consider the coordinatewise salt-and-pepper scheme introduced in the main body.
    Let $\mathfrak{m}:\R^{d(n)}\times\R\to[0,1]$ and $s:\R^{d(n)}\times\R\to\R^{d(n)\times d(n)}$.
    For $z=(x,y)\in\R^{d(n)}\times\R$ and deterministic $\etabf=(\{u_i\}_{i=1}^{d(n)},\etabf_2)$, with $\etabf_2\in\R^{d(n)}$, let $m_{z,\etabf}$ be the vector with $i$-th component $\1\{u_i\leq \mathfrak{m}(z)\}$.
    We define
    \begin{align}
        \tau_{\xtt}(z,\etabf) &= x \odot m_{z,\etabf} + s(z)\left(\etabf_2 - \etabf_2 \odot m_{z,\etabf}\right) \eqsp, \\
        \tau_{\ytt}(z,\etabf) &= y \eqsp,
    \end{align}
    and obtain the random augmentation by evaluating these maps at $\etabf=\eta$.
    Here $\eta=(\{U_i\}_{i=1}^{d(n)},\eta_2)$ is independent of $\Ztt$, the variables $\{U_i\}_{i=1}^{d(n)}$ are \iid~with distribution $\mathrm{Unif}(\ccint{0,1})$, and $\eta_2\in\R^{d(n)}$ is centered, independent of $\{U_i\}_{i=1}^{d(n)}$, with
    \begin{align}
        \E\left[\eta_2\eta_2^\top\right] = \Id{d(n)} \eqsp.
    \end{align}
    Writing $m_z=m_{z,\eta}$, we have $\E[m_z]=\mathfrak{m}(z)\mathbf{1}_{d(n)}$ and $\eta_2-\eta_2\odot m_z$ is centered conditionally on $m_z$. Hence
    \begin{align}
        \mu_{\xtt}(z) &= \mathfrak{m}(z)x \eqsp, &
        \mu_{\ytt}(z) &= y \eqsp, \\
        \Lambda(z) &= \{1-\mathfrak{m}(z)\}\left[\mathfrak{m}(z)\Diag(xx^\top) + s(z)s(z)^\top\right] \eqsp, &
        \Omega(z) &= 0 \eqsp.
    \end{align}
    Indeed, by independence of $m_z$ and $\eta_2$,
    \begin{align}
        \E\left[\left(\eta_2-\eta_2\odot m_z\right)\left(\eta_2-\eta_2\odot m_z\right)^\top\right] = \{1-\mathfrak{m}(z)\}\Id{d(n)} \eqsp,
    \end{align}
    which gives the expression of $\Lambda(z)$, while $\Omega(z)=0$ because $\tau_{\ytt}(z,\eta)=y$ is deterministic conditionally on $z$.
    Assume moreover that $\mathfrak{m}$ and $s$ are Lipschitz continuous, with Lipschitz constants uniformly bounded in $n$, and that
    \begin{align}
        \sup_{n\in\N}\sup_{z\in\R^{d(n)}\times\R} \|s(z)\|_\op < \infty \eqsp.
    \end{align}
    If the law of $\Ztt_1$ is supported in a Euclidean ball of radius independent of $n$, then the maps $\mu_{\xtt}$, $\mu_{\ytt}$, $\Lambda$ and $\Omega$ are bounded and Lipschitz on this support, with constants uniformly bounded in $n$. Therefore
    \begin{align}
        \Var(\Lambda(\Ztt_1)) + \Var(\Omega(\Ztt_1)) \leq \const \eqsp,
    \end{align}
    and \Cref{ass:artificial_data} holds for this scheme.
\end{example}

\begin{example}[Label-aware mixtures of augmentations]
Fix $m\in\N$, independent of $n$.
For each $j\in\{1,\ldots,m\}$, let $\eta_j$ be an augmentation variable taking values in a space $\mse_j$, and let
\begin{align}
    \tau_{\xtt,j} &: \R^{d(n)} \times \R \times \mse_j \to \R^{d(n)} \eqsp, &
    \tau_{\ytt,j} &: \R^{d(n)} \times \R \times \mse_j \to \R
\end{align}
be an arbitrary augmentation mechanism.
The variables $\eta_1,\ldots,\eta_m$ are assumed independent of $\Ztt$.
Let $\pi_1,\ldots,\pi_m:\R^{d(n)}\times\R\to[0,1]$ be Lipschitz continuous maps satisfying
\begin{align}
    \sum_{j=1}^m \pi_j(z) = 1
\end{align}
for all $z\in\R^{d(n)}\times\R$.
Let $V$ be uniformly distributed on $[0,1]$, independent of $\Ztt$ and of $\eta_1,\ldots,\eta_m$.
For deterministic $v\in[0,1]$, define the selected augmentation index
\begin{align}
    J(z,v) = \min\left\{j\in\{1,\ldots,m\}: \sum_{s=1}^j \pi_s(z) \geq v\right\} \eqsp.
\end{align}
For $z=(x,y)\in\R^{d(n)}\times\R$ and deterministic $\etabf=(v,\etabf_1,\ldots,\etabf_m)$, define
\begin{align}
    \tau_{\xtt}(z,\etabf) &= \tau_{\xtt,J(z,v)}(z,\etabf_{J(z,v)}) \eqsp, \\
    \tau_{\ytt}(z,\etabf) &= \tau_{\ytt,J(z,v)}(z,\etabf_{J(z,v)}) \eqsp,
\end{align}
and obtain the random augmentation by evaluating these maps at $\etabf=(V,\eta_1,\ldots,\eta_m)$.
For each component, define
\begin{align}
    \mu_{\xtt,j}(z) &= \E\left[\tau_{\xtt,j}(z,\eta_j)\right] \eqsp, \qquad
    \mu_{\ytt,j}(z) = \E\left[\tau_{\ytt,j}(z,\eta_j)\right] \eqsp, \\
    \Lambda_j(z) &= \Cov\left(\tau_{\xtt,j}(z,\eta_j)\right) \eqsp, \quad \hspace{0.2em} 
    \Omega_j(z) = \Cov\left(\tau_{\xtt,j}(z,\eta_j),\tau_{\ytt,j}(z,\eta_j)\right) \eqsp.
\end{align}
Since $\Prob(J(z,V)=j)=\pi_j(z)$, the resulting quantities are
\begin{align}
    \mu_{\xtt}(z) &= \sum_{j=1}^m \pi_j(z)\mu_{\xtt,j}(z) \eqsp, \qquad
    \mu_{\ytt}(z) = \sum_{j=1}^m \pi_j(z)\mu_{\ytt,j}(z) \eqsp, \\
    \Lambda(z) &= \sum_{j=1}^m \pi_j(z)\left\{\Lambda_j(z)+\mu_{\xtt,j}(z)\mu_{\xtt,j}(z)^\top\right\} - \mu_{\xtt}(z)\mu_{\xtt}(z)^\top \eqsp, \\
    \Omega(z) &= \sum_{j=1}^m \pi_j(z)\left\{\Omega_j(z)+\mu_{\xtt,j}(z)\mu_{\ytt,j}(z)\right\} - \mu_{\xtt}(z)\mu_{\ytt}(z) \eqsp.
\end{align}
Assume that the component-level maps $\mu_{\xtt,j},\mu_{\ytt,j},\Lambda_j,\Omega_j$ are bounded and Lipschitz with constants uniformly bounded in $j$ and in $n$.
Since $m$ is fixed and the weights $\pi_j$ are Lipschitz, the maps $\mu_{\xtt},\mu_{\ytt},\Lambda$ and $\Omega$ are also bounded and Lipschitz with constants uniformly bounded in $n$. Therefore
\begin{align}
    \Var(\Lambda(\Ztt_1)) + \Var(\Omega(\Ztt_1)) \leq \const \eqsp,
\end{align}
and \Cref{ass:artificial_data} holds.
\end{example}

The previous proposition and examples should be viewed as elementary building blocks rather than isolated constructions.
They show that the regularity requirements in \Cref{ass:artificial_data} are stable under several 
common augmentation mechanisms, including sample-dependent noise injection, masking, and finite mixtures of label-aware transformations.
Consequently, the framework covers a broad family of practical data-augmentation pipelines, including schemes whose strength 
or structure adapts to the local geometry of the data or to label-dependent invariances.

\section{Technical preliminaries on concentration of measure}\label{sec:appendix}  
We work throughout the appendix under the \emph{Lipschitz concentration} framework of \cite{LouartCouillet2018ConcentrationLargeRandomMatrices, LouartCouillet2020RandomEquationsConcentration}, 
which will serve as the main probabilistic tool in the proofs of \Cref{thm:main_result} and \Cref{proposition:det_equivs} and is defined as follows:
\begin{definition} \label{def:lip_concentration}
    A random variable $Z \in \R^d$ is said to be \emph{Lipschitz concentrated} with parameter $\kappa$, if for all $1$-Lispchitz function $f : \R^d \to \R$, $f(Z)$ is a $\kappa$-sub-Gaussian random variable.
\end{definition} 
In particular, the concentration property assumed for $\Ztt$ in \Cref{ass:data_distribution} simply restate as saying that the columns of $\Ztt$ are $1$-Lispchitz concentrated.
Furthermore, \cite{LouartCouillet2018ConcentrationLargeRandomMatrices} section 2.2.2. shows that such assumption implies that the full matrix $\Ztt$ is $1$-Lipschitz concentrated, 
which will be extensively used in the next sections. 
We then collect several technical results that will be repeatedly used in the sequel. \Cref{sec:technical_preliminaries-subsection1} 
introduces arguments that disentangle the dependence between training samples, most notably we generalize the standard 'leave-one-out' argument of \cite{chouard2022quantitative} to a 'leave-two-out' argument, 
and establish a set of concentration bounds for large random matrices. \Cref{sec:technical_preliminaries-subsection2} focuses on the concentration of random quadratic forms; in particular, 
we restate the Hanson--Wright inequality \cite{Adamczak2015HansonWright} in a setting where both the covariate and the quadratic form may be random. Finally, 
\Cref{sec:technical_preliminaries-subsection3} shows that random variables satisfying \Cref{def:lip_concentration} 
retain satisfying concentration properties after composition with certain non-Lipschitz mappings.
All our concentration results will be stated in the form of a stochastic domination result, introduced in \Cref{def:stochastic_domination}.

\subsection{Concentration of large random matrices: Useful lemmas} \label{sec:technical_preliminaries-subsection1}

We begin by recalling a standard fact regarding the spectral properties of the covariance matrix of a Lipschitz concentrated random vector.
This simple observation will be useful for bounding operator norms of covariance matrices throughout the paper,
thus allowing us to derive bounds that depend only on the Lipschitz concentration parameter of the data.

\begin{lemma}[Operator norm bound for the covariance matrix of a Lipschitz concentrated vector]
    \label{lemma:Covariance_maximum_eigenvalue}
    Let $Z$ be a centered random vector in $\R^d$ which is $\kappa$-Lipschitz concentrated. Then,
    \begin{equation}
        \|\E[ZZ^\top]\|_\op \leq \kappa^2 \eqsp.
    \end{equation}
\end{lemma}
\begin{proof}
    By definition of the operator norm, we write,
    \begin{equation}
        \|\E[ZZ^\top]\|_\op = \sup_{\|u\|_2 = 1} u^\top \E[ZZ^\top] u = \sup_{\|u\|_2 = 1} \Var(u^\top Z) \eqsp.
    \end{equation}
    For any $u \in \R^d$, the mapping $z \mapsto u^\top z$ is $\|u\|_2$-Lipschitz, thus $u^\top Z$ is $\kappa \|u\|_2$-sub-Gaussian from the Lipschitz concentration assumption.
    In particular, for unit vectors $u$, we have $\Var(u^\top Z) \leq \kappa^2$. The claimed bound follows.
\end{proof}

A more general result holds for any Lipschitz concentrated random matrix, 

\begin{lemma}[Spectral norm with Lipschitz–concentrated rows]
    \label{lemma:norm_lipschitz_concentrated_rows}
    Let $A\in\mathbb{R}^{p\times d}$ have independent rows $A_1,\dots,A_p\in\mathbb{R}^d$.
    Assume each row is $\kappa$–Lipschitz concentrated.
    Then there exists $\const>0$ such that for all $s>0$,
    \begin{equation}\label{eq:lemma-centered}
        \Prob\Big(\,\|A-\E[A]\|_{\op}\ \geq\ \consta\,\kappa\,(\sqrt{p}+\sqrt{d}) \;+\; s\,\Big)
        \ \leq  2\,\exp(-\,{\const\,s^2}/{\kappa^2}) \eqsp,
    \end{equation}
    and,
    \begin{equation}
        \label{eq:lemma-consequence}
        \Prob\Big(\,\|A - \E\left[A\right]\|_{\op} \geq 2\consta\,\kappa\,(\sqrt{p}+\sqrt{d}) \Big)
        \ \leq  2\,\exp(-\,{\const \consta_1^2 (p+d)}) \eqsp.
    \end{equation}
    with $\consta$ defined as,
    \begin{equation}
        \label{eq:def_a_1}
        \consta = \sqrt{2 \log(9)} \eqsp.
    \end{equation}
    Furthermore, for all $k \in \N^*$, and for three constants $\const_k, \const_k'$, we have,
    \begin{equation}\label{eq:lemma_concentrated_rows_operator_norm_moments}
        \E\left[\|A - \E\left[A\right]\|_{\op}^k\right] \leq \const_k \kappa^k (\sqrt{p} + \sqrt{d} + 1)^k + \const_k' \kappa^{k-1} (\sqrt{p} + \sqrt{d})^{k-1} \eqsp.
    \end{equation}
\end{lemma}

\begin{proof}
    The proof follows the standard $\varepsilon$–net argument (see, e.g., Lemma 5.2 and Lemma 5.4 of \cite{vershynin2018hdp}).
    We consider $\varepsilon$-nets of the unit spheres $\mathbb{S}^{p-1}$ and $\mathbb{S}^{d-1}$, i.e., sets
    $\msn\subset\mathbb{S}^{p-1}$ and $\msm\subset\mathbb{S}^{d-1}$ such that for any $x\in\mathbb{S}^{p-1}$,
    there exists $x'\in\msn$ with $\|x-x'\|_2\leq\varepsilon$, and similarly for $\msm$.
    For $\varepsilon=1/4$, Corollary 4.2.13 of \cite{vershynin2018hdp} guarantees the existence of $\varepsilon$–nets with cardinalities:
    \begin{equation}
        \label{eq:cardinal_eps_net}
        |\msn|\ \leq  (1+2/\varepsilon)^p\leq 9^p,\qquad
        |\msm|\ \leq  (1+2/\varepsilon)^d\leq 9^d\eqsp.
    \end{equation}
    For any matrix $B\in\mathbb{R}^{p\times d}$, the operator norm can be approximated by the maximum over the $\varepsilon$-nets:
    \begin{equation}\label{eq:net-to-op}
        \|B\|_{\op} \;=\; \max_{x\in\mathbb{S}^{p-1},\,y\in\mathbb{S}^{d-1}} x^\top B y
        \;\leq  \frac{1}{1-2\varepsilon}\,\max_{x\in\msn,\,y\in\msm} x^\top B y
        \;\leq  2\,\max_{x\in\msn,\,y\in\msm} x^\top B y \eqsp.
    \end{equation}
    For any $x\in\msn$ and $y\in\msm$, define
    \begin{equation}
        S_{x,y}\ :=\ x^\top (A-\E[A])\,y \;=\; \sum_{i=1}^p x_i\,\langle A_i-\E[A_i], y\rangle \eqsp.
    \end{equation}
    Since the map $r\mapsto \langle r,y\rangle$ is $1$–Lipschitz on $\mathbb{R}^d$, and each row $A_i$ is $\kappa$–Lipschitz concentrated,
    each summand $X_i:=\langle A_i-\E[A_i], y\rangle$ is zero-mean sub-Gaussian with parameter $\kappa$.
    Since the $\{A_i\}_{i=1}^p$ are independent, $\{X_i\}_{i=1}^p$ are also independent, and therefore
    $S_{x,y}=\sum_i x_i X_i$ is sub-Gaussian with parameter $\kappa\,\|x\|_2=\kappa$. Therefore, for any $t \geq 0$:
    \begin{equation}\label{eq:fixed-pair-tail}
        \Prob\big(|S_{x,y}|\geq t\big)\ \leq  2\exp(-\,t^2/(2\kappa^2)) \eqsp.
    \end{equation}
    By the union bound over $\msn\times\msm$ and using \eqref{eq:cardinal_eps_net}:
    \begin{equation}
        \Prob\left(\max_{x\in\msn,\,y\in\msm} |S_{x,y}|\ \geq\ t\right) \leq 2 \cdot 9^{p+d} \exp\left(-\frac{t^2}{2 \kappa^2}\right) \eqsp.
    \end{equation}
    Setting $t = \sqrt{2 \log(9)}\,\kappa\,(\sqrt{p}+\sqrt{d}) + s$ and using $\log(9^{p+d}) = (p+d)\log(9) \leq (\sqrt{p}+\sqrt{d})^2\log(9)$:
    \begin{equation}
        \Prob\left(\max_{x\in\msn,\,y\in\msm} |S_{x,y}|\ \geq\ \sqrt{2 \log(9)}\,\kappa\,(\sqrt{p}+\sqrt{d}) + s\right) \leq 2 \exp\left(-\frac{s^2}{2 \kappa^2}\right) \eqsp.
    \end{equation}
    Using \eqref{eq:net-to-op} with $B=A-\E[A]$ and defining $\consta_1 = \sqrt{2 \log(9)}$ completes the proof of \eqref{eq:lemma-centered}. \\
    \noindent The second equation \eqref{eq:lemma-consequence} follows by setting $s = \consta_1\,\kappa\,(\sqrt{p}+\sqrt{d})$ and using $(\sqrt{p} + \sqrt{d})^2 \geq (p+d)$.
    For the moment bound \eqref{eq:lemma_concentrated_rows_operator_norm_moments}, we use the integral representation:
    \begin{align}
        \E\left[\|A - \E\left[A\right]\|_\op^k\right] &= \int_0^\infty k t^{k-1} \Prob\left(\|A - \E\left[A\right]\|_\op \geq t\right) dt \eqsp.
    \end{align}
    We split the integral at $t_0 = 2\consta_1 \kappa (\sqrt{p} + \sqrt{d})$:
    \begin{align}
        \E\left[\|A - \E\left[A\right]\|_\op^k\right] &= \int_0^{t_0} k t^{k-1} dt + \int_{t_0}^\infty k t^{k-1} \Prob\left(\|A - \E\left[A\right]\|_\op \geq t\right) dt \\
        &= t_0^k + \int_{t_0}^\infty k t^{k-1} \Prob\left(\|A - \E\left[A\right]\|_\op \geq t\right) dt \eqsp.
    \end{align}
    For the second integral, we use the concentration bound from \eqref{eq:lemma-centered} with $s = t - \consta_1 \kappa (\sqrt{p} + \sqrt{d})$:
    \begin{align}
        \int_{t_0}^\infty k t^{k-1} \Prob\left(\|A - \E\left[A\right]\|_\op \geq t\right) dt &\leq 2k \int_{t_0}^\infty t^{k-1} \exp\left(-\frac{(t - \consta_1 \kappa (\sqrt{p} + \sqrt{d}))^2}{2 \kappa^2}\right) dt \\
        &= 2k \int_0^\infty (s + \consta_1 \kappa (\sqrt{p} + \sqrt{d}))^{k-1} \exp\left(-\frac{s^2}{2 \kappa^2}\right) ds \eqsp,
    \end{align}
    where we made the substitution $s = t - \consta_1 \kappa (\sqrt{p} + \sqrt{d})$. Using $(a + b)^{k-1} \leq 2^{k-2}(a^{k-1} + b^{k-1})$ for $a,b \geq 0$:
    \begin{align}
        &\leq 2^{k-1} k \left[(\consta_1 \kappa (\sqrt{p} + \sqrt{d}))^{k-1} \int_0^\infty \exp\left(-\frac{s^2}{2 \kappa^2}\right) ds + \int_0^\infty s^{k-1} \exp\left(-\frac{s^2}{2 \kappa^2}\right) ds\right] \\
        &= 2^{k-1} k \left[(\consta_1 \kappa (\sqrt{p} + \sqrt{d}))^{k-1} \Gamma(1) \kappa + \left(\sqrt{2} \kappa\right)^{k-1} \Gamma\left(\frac{k}{2}\right)\right] \eqsp.
    \end{align}
    Combining both terms and using that $t_0^k = (2\consta_1 \kappa (\sqrt{p} + \sqrt{d}))^k$:
    \begin{align}
        \E\left[\|A - \E\left[A\right]\|_\op^k\right] &\leq (2\consta_1 \kappa (\sqrt{p} + \sqrt{d}))^k + 2^{k-1} k \sqrt{\pi/2} \kappa (\consta_1 \kappa (\sqrt{p} + \sqrt{d}))^{k-1} \\
        &\quad + 2^{k-1} k \left(\sqrt{2} \kappa\right)^{k-1} \Gamma\left(\frac{k}{2}\right) \\
        &\leq \const_k \kappa^k (\sqrt{p} + \sqrt{d} + 1)^k + \const_k' \kappa^{k-1} (\sqrt{p} + \sqrt{d})^{k-1} \eqsp,
    \end{align}
    for appropriate constants $\const_k, \const_k' > 0$ that depend only on $k$.
\end{proof}

Now, recall the definition of the variance of a random matrix. For any $A \in \R^{p\times p}$,
\begin{align} \label{eq:matrix_variance_definition}
    \Var(A) = \E\left[\|A - \E\left[A\right]\|_\Frob^2\right] \eqsp.
\end{align}
Note that this definition is consistent with the one of the variance of a random variable living in a general Hilbert space. 
We prove the following standard result about the variance of a product of random matrices, 
although it holds in any Hilbert space, we state it here only in the case of $A \in \R^{p\times p}$.
\begin{lemma}[Variance of the product of random matrices]
    \label{lemma:matrix_variance_product_bound}
    Let $A_1, A_2 \in \R^{p\times p}$ be two random matrices. Then,
    \begin{align}
        \Var(A_1 A_2) &\leq 2\| \|A_1\|_\op \|_\infty^2 \Var(A_2) + 2\|\E\left[A_2\right]\|_\op^2 \Var(A_1), \\
        \Var(A_1 A_2) &\leq 2\| \|A_2\|_\op \|_\infty^2 \Var(A_1) + 2\|\E\left[A_1\right]\|_\op^2 \Var(A_2), \\
        \text{and} \quad \Var(A_1 A_2 A_1) &\leq 2 \| \|A_1\|_\op \|_\infty^4 \Var(A_2) + 8 \|\E\left[A_1\right]\|_\op^2 \| \|A_1\|_\op \|_\infty^2 \Var(A_2) \eqsp.
    \end{align}
\end{lemma}
\begin{proof}
    We prove the first inequality, the second one follows from the same arguments. We have that,
    \begin{align}
        \Var(A_1 A_2) = \Var(A_1 \{A_2 - \E\left[A_2\right]\} + A_1 \E\left[A_2\right]) \leq 2 \Var(A_1 \{A_2 - \E\left[A_2\right]\}) + 2 \Var(A_1 \E\left[A_2\right]) \eqsp.
    \end{align}
    Using $\Var(A) \leq \E\left[\|A\|_\Frob^2\right]$, it follows that,
    \begin{align}
        \Var(A_1 \{A_2 - \E\left[A_2\right]\}) \leq \E\left[\|A_1 \{A_2 - \E\left[A_2\right]\}\|_\Frob^2\right] \leq \| \|A_1\|_\op \|_\infty^2 \Var(A_2) \eqsp.
    \end{align}
    Respectively, and by definition of $\Var$, we have,
    \begin{align}
        \Var(A_1 \E\left[A_2\right]) = \E\left[\|\{A_1 - \E\left[A_1\right]\} \E\left[A_2\right]\|_\Frob^2\right] \leq \|\E\left[A_2\right]\|_\op^2 \Var(A_1) \eqsp.
    \end{align}
    The two first inequalities follow. 
    We proove the last inequality in a similar way, indeed we have,
    \begin{align}
        \Var(A_1 A_2 A_1) \leq  2 \Var(A_1 \{A_2 - \E\left[A_2\right]\} A_1) + 2 \Var(A_1 \E\left[A_2\right] A_1) \eqsp,
    \end{align}
    Using the previously derived upper bounds, we have,
    \begin{align}
        \Var(A_1 \E\left[A_2\right] A_1) &\leq 2 \|\|A_1\|_\op\|_\infty^2 \Var\left(\E\left[A_2\right] A_1\right) + 2 \|\E\left[A_1\right]\|_\op^2 \|\E\left[A_2\right]\|_\op^2 \Var(A_1) \\
        &\leq 4 \|\|A_1\|_\op\|_\infty^2 \|\E\left[A_2\right]\|_\op^2 \Var(A_1)
    \end{align}
    in addition,
    \begin{align}
        \Var(A_1 \{A_2 - \E\left[A_2\right]\} A_1) \leq \E\left[\|A_1 \{A_2 - \E\left[A_2\right]\} A_1\|_\Frob^2\right] \leq \|\|A_1\|_\op\|_\infty^4 \Var(A_2) \eqsp.       
    \end{align}
    The statement follows.
\end{proof}

\subsection{Concentration of random quadratic forms} \label{sec:technical_preliminaries-subsection2}

Moving on, we now restate a few usual results about the concentration of quadratic forms in the context of stochastic domination \Cref{def:stochastic_domination} and Lipschitz concentration \Cref{def:lip_concentration}. 
We start by highlighting a few standard properties of stochastic domination, and how it relates to non-asymptotic concentration bounds.
Our first result shows the additivity and multiplicative properties of \Cref{def:stochastic_domination}.

\begin{lemma}[Stability of the stochastic domination]
    \label{lemma:additivity_multiplicativity_stochdom}
    Let $(U_n)_{n\in\N}, (V_n)_{n\in\N}, (U_n')_{n\in\N}$, and $(V_n')_{n\in\N}$ be four sequences of non-negative random variables.
    If $U_n \stochdom V_n$ and $U_n' \stochdom V_n'$, then it holds,
    \begin{equation}
        U_n + U_n' \stochdom V_n + V_n' \eqsp, \quad \text{and} \quad U_n U_n' \stochdom V_n V_n' \eqsp.
    \end{equation}
\end{lemma}
\begin{proof}
    Fix $\delta>0$ and $D>0$. From $U_n\stochdom V_n$ and $U_n'\stochdom V_n'$, by definition of stochastic domination,
    for any choice of $\delta>0$ and $D>0$ there exists $N(\delta,D)$ such that for all $n\geq N(\delta,D)$,
    \[
    \Prob\!\left(U_n>n^{\delta}V_n\right)\leq n^{-D}
    \quad\text{and}\quad
    \Prob\!\left(U_n'>n^{\delta}V_n'\right)\leq n^{-D}.
    \]

    \smallskip
    \noindent\textit{Additivity.} Applying the union-bound, we have,
    for $n \geq N(\delta, D+1)$,
    \[
    \Prob\!\left(U_n+U_n'>n^\delta(V_n+V_n')\right)
    \;\leq\;
    \Prob\!\left(U_n>n^{\delta}V_n\right)+\Prob\!\left(U_n'>n^{\delta}V_n'\right)
    \;\leq\; 2n^{-D+1}\leq n^{-D},
    \]
    where the last inequality holds for sufficiently large $n$.
    Thus $U_n+U_n'\stochdom V_n+V_n'$.

    \smallskip
    \noindent\textit{Multiplicativity.}
    Similarly as in the additive case, we have thanks to the union bound,
    \[
    \Prob\!\left(U_nU_n'>n^\delta V_nV_n'\right)
    \leq
    \Prob\!\left(U_n>n^{\delta/2}V_n\right)+\Prob\!\left(U_n'>n^{\delta/2}V_n'\right)
    \leq 2n^{-(D+1)}\leq n^{-D},
    \]
    for all sufficiently large $n$. Hence $U_nU_n'\stochdom V_nV_n'$.

    \smallskip
    Both claims hold for arbitrary $\delta>0$ and $D>0$, which completes the proof.
\end{proof}

We now clarify the connection between stochastic domination and classical concentration of measure results. 
To this end, introduce the following definition, which can be seen as an extension of the definition of a sub-Gaussian random variable to sequences 
of random variables, such as in \Cref{def:stochastic_domination}.

\begin{definition}
    We say that a sequence of real random variables $(U_n)_{n\in\N}$ is a sub-Gaussian sequence with parameters $(\sigma_n)_{n\in\N}$ if for there exists $\const \geq 0 $ and $\lambda_0 >0$ such that for any $n\in\N$ and $\lambda \leq \lambda_0$, $\E[\rme^{\lambda U_n}] \leq \exp(\const \lambda^2 \sigma_n^2)$.
    If $(U_n)_{n\in\N}$ is complex valued, we say similarly that it is a sub-Gaussian sequence with parameters $(\sigma_n)_{n\in\N}$ if its real and imaginary parts are sub-Gaussian sequences  with parameters $(\sigma_n)_{n\in\N}$.
\end{definition}

We observe that several common tail bounds imply stochastic domination in the appropriate regimes.

\begin{lemma}[Stochastic domination from sub-Gaussian and sub-Exponential concentration]
    \label{lemma:Concentration_to_stochdom}
    Let $(U_n)_n$ be a sequence of complex random variables. Suppose there exist constants $\const_1, \const_2, \alpha > 0$ and a sequence $(\sigma_n)_n$ of positive real numbers such that, for all $t > 0$,
    \begin{equation}
        \Prob\left(|U_n| > t\right) \leq \const_1 \exp\left(-\const_2 \dfrac{t^\alpha}{\sigma_n^\alpha}\right)\eqsp.
    \end{equation}
    Then, it follows that
    \begin{equation}
        |U_n| \stochdom \sigma_n \eqsp.
    \end{equation}
\end{lemma}

\begin{proof}
    Let $\varepsilon, D > 0$. For all $n \geq 1$, we may write
    \begin{align}
        \Prob\left(|X_n| > n^\varepsilon \sigma_n\right) \leq \const_1 \exp\left(-\const_2 \dfrac{(n^{\varepsilon}\sigma_n)^\alpha}{\sigma_n^{\alpha}}\right) = \const_1 \exp\left(-\const_2 n^{\alpha \varepsilon}\right).
    \end{align}
    The exponential decay ensures that, for sufficiently large $n$, $\const_1 \exp\left(-\const_2 n^{\alpha \varepsilon}\right)\leq n^{-D}$. Therefore, the conclusion follows by the definition of stochastic domination.
\end{proof}

To conclude this subsection, we provide a version of the Hanson-Wright inequality in the framework of stochastic domination.
To this end, we first recall the usual statement of the Hanson-Wright inequality:
\begin{lemma}[Hanson-Wright inequality]
    \label{lemma:HansonWright}
    Let $(Z_n)_n$ be a sequence of centered $d$ dimensional random vector such that for all $n \in \nset$, $Z_n$ is $\sigma_n$-Lipschitz concentrated.
    For a sequence of deterministic matrix $(\Abf_n)_n$,
    it holds for all $t > 0$,
    \begin{equation} \label{eq:Vanilla_HansonWright}
        \Prob\left(\left|Z_n^\top \Abf_n Z_n - \E[Z_n^\top \Abf_n Z_n]\right| > t\right) \leq 2\exp\left(-c \min\left\{\dfrac{t^2}{\sigma_n^4 \|\Abf_n\|_F^2}, \dfrac{t}{\sigma_n^2 \|\Abf_n\|_\op}\right\}\right) \eqsp,
    \end{equation}
    Stated as a stochastic domination result, we have,
    \begin{equation}
        \left|Z_n^\top \Abf_n Z_n - \E[Z_n^\top \Abf_n Z_n]\right| \stochdom \sigma_n^2 \|\Abf_n\|_F \eqsp.
    \end{equation}
    And, the moments of such quadratic forms are bounded by,
    \begin{equation}
        \E\left[|Z_n^\top \Abf_n Z_n - \E[Z_n^\top \Abf_n Z_n]|^k\right] \leq \const_k \sigma_n^{2k} \|\Abf_n\|_F^k \eqsp.
    \end{equation}
    Where $c, \eqsp (c_k)_{k\in \N} > 0$ are universal constants, that don't depend on $n$, $Z_n$ nor $\Abf_n$.
\end{lemma}
\begin{proof}
    The first part of the theorem is a direct application of Theorem 2.4. of \cite{Adamczak2015HansonWright}.
    The second part follows from \Cref{lemma:Concentration_to_stochdom}.
    We prove the third part by writting the moments in integral form as, and using \eqref{eq:Vanilla_HansonWright},
    \begin{align}
        \E\left[|Z_n^\top \Abf_n Z_n - \E[Z_n^\top \Abf_n Z_n]|^k\right] &= \int_0^\infty k t^{k-1} \Prob\left(\left|Z_n^\top \Abf_n Z_n - \E[Z_n^\top \Abf_n Z_n]\right| > t\right) dt \\
        &\leq 2k \int_0^\infty t^{k-1} \exp\left(-c \min\left\{\dfrac{t^2}{\sigma_n^4 \|\Abf_n\|_F^2}, \dfrac{t}{\sigma_n^2 \|\Abf_n\|_\op}\right\}\right) dt \eqsp.
    \end{align}
    Let $u = t/(\sigma_n^2 \|\Abf_n\|_F)$ and $v = \|\Abf_n\|_\op/\|\Abf_n\|_F$. The integral becomes:
    \begin{align}
        &\leq 2k (\sigma_n^2 \|\Abf_n\|_F)^k \int_0^\infty u^{k-1} \exp\left(-c \min\left\{u^2, \dfrac{u}{v}\right\}\right) du \eqsp.
    \end{align}
    Splitting the integral into two parts, we have:
    \begin{align}
        \int_0^\infty u^{k-1} e^{-c \min\{u^2, u/v\}} du &\leq \int_0^\infty u^{k-1} e^{-c u^2} du + \int_0^\infty u^{k-1} e^{-c u/v} du = \dfrac{\Gamma(k/2)}{2c^{k/2}} + \dfrac{\Gamma(k)}{c^k} \eqsp.
    \end{align}
    Where we have used $v < 1$. We can conclude that:
    \begin{align}
        \E\left[|Z_n^\top \Abf_n Z_n - \E[Z_n^\top \Abf_n Z_n]|^k\right] &\leq \const_k \sigma_n^{2k} \|\Abf_n\|_F^k \eqsp,
    \end{align}
    where $\const_k = 2k \left(\Gamma(k/2) / (2c^{k/2}) + \Gamma(k)/(c^k)\right)$.
\end{proof}

In what follows, it will also be necessary to address the scenario in which $\Abf_n$ is a random matrix that is independent of $X_n$.
Specifically, we establish the following result:

\begin{corollary}[Concentration of random quadratic forms]
    \label{corollary:randomHansonWright}
    Consider a sequence of centered $d$-dimensional random vectors $(Z_n)_{n \in \mathbb{N}}$ that are $\sigma_n$-Lipschitz concentrated,
    and a sequence of random matrices $(A_n)_{n \in \mathbb{N}}$ that are $\kappa_n$-Lipschitz concentrated and independent of $(Z_n)_{n \in \mathbb{N}}$.
    Then, the deviation from its expectation is stochastically dominated by,
    \begin{equation} \label{eq:randomHansonWright-deviation}
        |Z_n^\top A_n Z_n - \E\left[Z_n^\top A_n Z_n\right]| \stochdom \sigma_n^2 \|A_n\|_\Frob + \kappa_n \|\E\bigl[Z_n Z_n^\top\bigr]\|_\Frob \eqsp.
    \end{equation}
    And for $k \in \N$, its $k$-th moment satisfies
    \begin{equation}
        \E\left[\left|Z_n^\top A_n Z_n - \E\left[Z_n^\top A_n Z_n\right]\right|^k\right] \leq \const_k \left(\sigma_n^{2k} \E\left[\|A_n\|_\Frob^k\right] + \kappa_n^k \|\E\bigl[Z_n Z_n^\top\bigr]\|_\Frob^k\right) \eqsp,
    \end{equation}
    where $(\const_k)_{k\in \N} > 0$ are universal constants.
    \end{corollary}
\begin{proof}
    We establish the result by decomposing the deviation into two components and analyzing each separately.
    First, we write the triangular inequality:
    \begin{equation} \label{eq:randomHansonWright-triangular-inequality}
        |Z_n^\top A_n Z_n - \E\left[Z_n^\top A_n Z_n\right]| \leq |Z_n^\top A_n Z_n - \E\left[Z_n^\top A_n Z_n \mid A_n\right]| + |\E\left[Z_n^\top A_n Z_n \mid A_n\right] - \E\left[Z_n^\top A_n Z_n\right]| \eqsp.
    \end{equation}
    By applying \Cref{lemma:HansonWright} conditionally on $A_n$, we obtain the high probability bound:
    \begin{equation}
        \Prob\left(|Z_n^\top A_n Z_n - \E\left[Z_n^\top A_n Z_n \mid A_n\right]| \geq t \mid A_n\right) \leq 2 \exp\left(- \const \min\left\{\dfrac{t^2}{2 \sigma_n^4 \|A_n\|_\Frob^2}, \dfrac{t}{\sigma_n\|A_n\|_\op}\right\}\right) \eqsp.
    \end{equation}
    Since $\|A_n\|_\op \leq \|A_n\|_\Frob$ almost surely, it follows from the definition of stochastic domination that:
    \begin{equation}
        |Z_n^\top A_n Z_n - \E\left[Z_n^\top A_n Z_n \mid A_n\right]| \stochdom \sigma_n^2 \|A_n\|_\Frob \eqsp.
    \end{equation}
    Now turning to the second term in \eqref{eq:randomHansonWright-triangular-inequality}, we express
    the conditional expectation $\E\left[Z_n^\top A_n Z_n \mid A_n\right]$ as a function of $A_n$.
    Specifically, defining $f_n : \Abf \mapsto \tr\bigl(\E\bigl[Z_n Z_n^\top\bigr] \Abf\bigr)$, we get,
    \begin{align}
        |\E\left[Z_n^\top A_n Z_n \mid A_n\right] - \E\left[Z_n^\top A_n Z_n\right]| = |f_n(A_n) - \E\left[f_n(A_n)\right]| \eqsp.
    \end{align}
    Since $f_n$ is $\|\E\bigl[Z_n Z_n^\top\bigr]\|_\Frob$-Lipschitz, and $A_n$ is $\kappa_n$-Lipschitz concentrated, thus from \Cref{lemma:Concentration_to_stochdom} we get,
    \begin{equation}
        |f_n(A_n) - \E\left[f_n(A_n)\right]| \stochdom \kappa_n \|\E\bigl[Z_n Z_n^\top\bigr]\|_\Frob \eqsp.
    \end{equation}
    The first statement of the corollary follows from the additivity property of stochastic domination (\Cref{lemma:additivity_multiplicativity_stochdom}).
    For the second statement, use $(a+b)^k \leq 2^{k-1} (a^k + b^k)$ and take the expectation in \eqref{eq:randomHansonWright-triangular-inequality}, to get,
    \begin{align}
        \E\left[\left|Z_n^\top A_n Z_n - \E\left[Z_n^\top A_n Z_n\right]\right|^k\right] \leq 2^{k-1} \bigg(&\E\left[\E\left[\left|Z_n^\top A_n Z_n - \E\left[Z_n^\top A_n Z_n \mid A_n\right] \right|^k \mid A_n\right]\right] \\
        &+ \E\left[\left|f_n(A_n) - \E\left[f_n(A_n)\right]\right|^k\right]\bigg) \eqsp.
    \end{align}
    The previous function $f_n$ being $\|\E\left[Z_n Z_n^\top\right]\|_\Frob$-Lipschitz, we get from the Lipschitz concentration assumption on $A_n$, as well as the well-known moment bound on sub-Gaussian random variables that:
    \begin{equation} \label{eq:randomHansonWright-var-A}
        \E\left[\left|f_n(A_n) - \E\left[f_n(A_n)\right]\right|^k\right] \leq \const_k \kappa_n^k \|\E\left[Z_n Z_n^\top\right]\|_\Frob^k \eqsp.
    \end{equation}
    Moreover, applying the moment bound of \Cref{lemma:HansonWright} conditionally on $A_n$ yields the almost sure bound,
    \begin{equation}
        \E\left[\left|Z_n^\top A_n Z_n - \E\left[Z_n^\top A_n Z_n \mid A_n\right] \right|^k \mid A_n\right] \leq \const_k \sigma_n^{2k} \|A_n\|_\Frob^k \eqsp,
    \end{equation}
    for some large enough universal constants $(\const_k)_{k\in \N} > 0$. The second statement follows.
\end{proof}

\subsection{Concentration through a non Lipschitz mapping} \label{sec:technical_preliminaries-subsection3}

This final subsection is dedicated to showing that random variables satisfying \Cref{def:lip_concentration} preserve satisfying concentration properties when composed with a non Lipschitz mapping.

More precisely, we will consider a differentiable map $f:\mathbb{R}^k\to\mathbb{R}$, and notice that on any compact set $\msk\subset\mathbb{R}^k$, the function $f_{ \mid \msk}$ is Lipschitz.
The concentration of $f(Z)$ will then follow from finding a trade-off between the grows of the Lipschitz constant of $f_{\mid \msk}$ and the probability of the event $\{Z \notin \msk\}$.
We will denote $\Ltt_{\msk}^f$ the Lipschitz constant of $f$ on $\msk$:
\begin{equation}
  \label{eq:def_const_loc_lip}
  \Ltt_{\msk}^f :=\ \sup_{(x,y)\in\msk^2} \{|f(x) - f(y)|/\|x-y\|_2\}\eqsp.
\end{equation}
In addition, if $\Ltt_{\msk}^f < \plusinfty$, consider the Moreau envelop $g_{\msk}^f$ of and the Moreau residual $\delta_{\msk}^f$ defined as,
\begin{equation} \label{eq:def_moreau_envelop}
    g_{\msk}^f(z) = \inf_{x \in \msk} \{f(x) + \Ltt_{\msk}^{f} \|x - z\|_2\}\eqsp, \quad \text{and} \quad \delta_{\msk}^f(Z) \ :=\ \E\big[\,|f(Z)-g_{\msk}^f(Z)|\,\big] \eqsp.
\end{equation}
If $\Ltt_{\msk}^f = \plusinfty$, we use the convention $\delta_{\msk}^f= +\infty$.

\begin{lemma}
    \label{lemma:MoreauExtension}
    Let $f :\rset^k \to \C$ and $\msk \subset \rset^k$ be a
    compact set.  Assume that $\Ltt_{\msk}^f <\plusinfty$. Then,
    $g_{\msk}^f$ defined in \eqref{eq:def_moreau_envelop} is a Lipschitz extension of $f$ restricted to
    $\msk$, i.e., it is a Lipschitz function from $\rset^k$ to
    $\rset$ such that for any $x\in\msk$, $f(x) = g_{\msk}^f(x)$.
    Furthermore, for any random variable $Z \in \R^k$, and any $z_0 \in \msk$, it holds,
    \begin{align}
        \delta_{\msk}^f(Z) \leq \left(\E\left[|f(Z)|^2\right]^{1/2} +|f(z_0)| + \Ltt_\msk^f \E\left[\|Z - z_0\|_2^2\right]^{1/2} \right) \Prob\left(Z \notin \msk\right)^{1/2}
    \end{align}
\end{lemma}
\begin{proof}
    Fix $z_1,z_2\in\R^k$ and any $x\in\msk$. By the triangle inequality,
    \[
    f(x)+\Ltt_\msk^f\|z_1-x\|
    \;\leq\; f(x)+\Ltt_\msk^f\big(\|z_2-x\|+\|z_1-z_2\|\big)
    \;=\; \big(f(x)+\Ltt_\msk^f\|z_2-x\|\big)+\Ltt_\msk^f\|z_1-z_2\|.
    \]
    Taking the infimum over $x\in\msk$ yields
    \[
    g_{\msk}^f(z_1)\;\leq\; g_{\msk}^f(z_2)+\Ltt_\msk^f\|z_1-z_2\|.
    \]
    Exchanging the roles of $z_1$ and $z_2$ gives
    \[
    |g_{\msk}^f(z_1)-g_{\msk}^f(z_2)|\;\leq\;\Ltt_\msk^f\|z_1-z_2\|,
    \]
    so $g_{\msk}^f$ is $\Ltt_\msk^f$-Lipschitz on $\R^k$. We now show that $g_{\msk}^f$ coincides with $f$ on $\msk$.
    Let $z \in \msk$, since $f_{\mid \msk}$ is $\Ltt^f_\msk$-Lipschitz, we have for any $x \in \msk$,
    \begin{align}
        f(z) \leq f(x) + \Ltt_\msk^f \|z - x\|_2 \eqsp.
    \end{align}
    taking the infimum over $x \in \msk$ yields $f(z) \leq g_{\msk}^f(z)$.
    Furthermore, taking $x = z$ in the definition of $g_{\msk}^f(z)$ gives $g_{\msk}^f(z) \leq f(z)$, proving the equality.
    Finally, let $Z \in \R^k$ to be any random vector, without loss of generality (and up to replacing $\msk$ by $\msk - z_0$ and $f(\cdot)$ by $f(\cdot + z_0)$), we assume $z_0=0$.
    Leveraging that $f - g_\msk^f$ is zero on $\msk$, we write,
    \begin{align}
        \delta_\msk^f(Z) = \E\left[|f(Z)-g_{\msk}^f(Z)| \1_{\msk^\complement}(Z)\right] \leq \E\left[|f(Z)|\1_{\msk^\complement}(Z)\right] + \E\left[|g_\msk^f(Z)|\1_{\msk^\complement}(Z)\right]
    \end{align}
    We now bound the two terms in the right-hand side. For the first term, we have by the Cauchy-Schwarz inequality,
    \begin{align}
        \E\left[|f(Z)|\1_{\msk^\complement}(Z)\right] & \leq \E\left[|f(Z)|^2\right]^{1/2} \Prob(Z \notin \msk)^{1/2} \eqsp.
    \end{align}
    The second term can be bounded thanks to the Lipschitz property of $g_{\msk}^f$,
    \begin{align}
        \E\left[|g_\msk^f(Z)|\1_{\msk}(Z)\right] & \leq |g_\msk^f(0)| \Prob(Z \notin \msk) + \Ltt_\msk^f \E\left[\|Z\|_2 \1_{\msk^\complement}(Z)\right] \\
        & \leq |f(0)| \Prob(Z \notin \msk) + \Ltt_\msk^f \E\left[\|Z\|_2^2\right]^{1/2} \Prob(Z \notin \msk)^{1/2} \eqsp,
    \end{align}
    where the last inequality followed from the Cauchy-Schwarz inequality, as well as the fact that $0 \in \msk$.
    The claimed upper bound follows.
\end{proof}

\begin{lemma}[Concentration through locally Lipschitz maps, all $t>0$]
    \label{lemma:concentration_lipschitz_maps_all_t}
    Let $Z\in\mathbb{R}^k$ be zero-mean and $\kappa$–Lipschitz concentrated random vector, and a compact set $\msk \subset \R^k$.
    We have for every $t>0$,
    \begin{equation}\label{eq:all-t}
        \Prob\Big(\,|f(Z)-\E f(Z)|\geq t\,\Big)
        \ \leq\ 2\exp\Big(-\,\frac{t^2}{8\,\kappa^2 (\Ltt_{\msk}^f)^2}\Big)\;+\;\frac{4\,\delta_\msk^f(Z)}{t}.
    \end{equation}
\end{lemma}

\begin{proof}
    The case $\Ltt_{\msk}^f = \plusinfty$ is trivial. We now suppose that $\Ltt_{\msk}^f < \plusinfty$.

    Define the function $\Delta_{\msk}^f:=f-g_{\msk}^f$, for any $t>0$, it holds
    \begin{equation}\label{eq:split}
        \Prob\big(|f(Z)-\E f(Z)|\geq t\big)
        \ \leq  \Prob\Big(|g_{\msk}^f(Z)-\E g_{\msk}^f(Z)|\geq t/2\Big)
        \ +\ \Prob\Big(|\Delta_{\msk}^f(Z)-\E\Delta_{\msk}^f(Z)|\geq t/2\Big).
    \end{equation}
    We now bounds the two terms in the right-hand side. Since $g_{\msk}^f$ is $\Ltt_{\msk}^f$–Lipschitz, and $Z$ is supposed to be $\kappa$–Lipschitz concentrated, we get
    \begin{equation}\label{eq:bounds_g_msk_0_0}
        \Prob\Big(|g_{\msk}^f(Z)-\E g_{\msk}^f(Z)|\geq t/2\Big)
        \leq  2\exp \{-\,{t^2}/{(8\,\kappa^2 (\Ltt_{\msk}^f)^2)}\}\eqsp.
    \end{equation}
    For the second term
    by the triangle inequality and Markov’s inequality, we get
    \begin{equation}\label{eq:bounds_g_msk_0}
        \Prob\Big(|\Delta_{\msk}^f(Z)-\E\Delta_{\msk}^f(Z)|\geq t/2\Big)
        \leq  \frac{2\,\E|\Delta_{\msk}^f(Z)-\E\Delta_{\msk}^f(Z)|}{t}
        = \frac{4\,\delta_\msk^f}{t}\eqsp.
    \end{equation}
\end{proof}

Combining the results of \Cref{lemma:MoreauExtension} and \Cref{lemma:concentration_lipschitz_maps_all_t}, 
one can see that the concentration of $f(Z)$ can be derived from $\Ltt_f^\msk$ and $\Prob(Z \notin \msk)$ being both small, in particular, it holds,
\begin{proposition}[Stochastic domination for non-Lipschitz transforms of a Lipschitz concentrated random matrix] 
    \label{proposition:Stochdom_local_lipschitz_transforms}
    Let $(Z_n)_{n \in \N}$ be a sequence of random matrices, with for all $n \in \N$, $Z_n \in \R^{d_n \times n}$ being zero-mean and $\kappa_n$-Lipschitz concentrated.
    Consider the compact set of $\R^{d_n \times n}$ given by,
    \begin{equation}
        \msk_n = \left\{\Zbf \in \R^{d_n \times n} : \|\Zbf\|_{\op} \leq 2 \consta_1 \kappa_n \left(\sqrt{n} + \sqrt{d_n}\right)\right\} \eqsp,
    \end{equation} 
    where $\consta_1 > 0$ is defined in \eqref{eq:def_a_1}.
    In addition, let $f_n : \R^{d_n \times n} \to \C$ be a function holomorphic in its second argument.
    For $\zeta \in \C$, we denote by $\Ltt_n(\zeta_n)$ the Lipschitz constant of $f_n(\cdot, \zeta_n)$ on $\msk_n$ defined in \eqref{eq:def_const_loc_lip}.
    For some fixed $\zeta \in \C$ and $r_n > 0$, we set,
    \begin{equation}
        \Ltt_n^\star(\zeta, r_n) = \sup_{\zeta' \in \sphere(\zeta, r_n)} \Ltt_n(\zeta') \eqsp, \quad M_n^{(2)}(\zeta) = \E\left[\big|f_n(Z_n, \zeta)\big|^2\right]^{1/2} \eqsp, \quad M_n^{(2)}(\zeta, r_n) = \sup_{\zeta' \in \sphere(\zeta, r_n)} M_n^{(2)}(\zeta') \eqsp,
    \end{equation}
    In addition, we assume that:
    \begin{enumerate}
        \item \label{proposition:Stochdom_local_lipschitz_transforms_a} The shapes of the matrices satisfy,
              $$\sup_{n \in \nset} \frac{d_n}{n} < + \infty\eqsp.$$
        \item \label{proposition:Stochdom_local_lipschitz_transforms_b} For any $n\in \nset$ and $\zeta \in\cset$ and $r \geq 0$, there exists $\mathtt{D}_n(\zeta, r)$ such that
        $$\displaystyle \sup_{\zeta' \in \sphere(\zeta,r)}\left|\frac{\partial f_n}{\partial \zeta^\prime}\big(Z_{n},\zeta^\prime \big) - \E\Big[\frac{\partial f_n}{\partial \zeta^\prime}\big(Z_{n},\zeta^\prime \big)\Big]\right| \stochdom \mathtt{D}_n(\zeta, r) \eqsp,$$
        which satisfies for some $\beta_2>0$ (independent of $\zeta$ and $r$ and $n$),
        $$ \sup_{n\in\nset} \frac{\mathtt{D}_n(\zeta, r)}{\kappa_n\,\Ltt_n^\star(\zeta,r)n^{\beta_2}}< \plusinfty \eqsp.$$
    \end{enumerate}
    Then, for some constants $\const>0$, it holds for all $z \in \C$ and $r > 0$,
    \begin{align}
        \sup_{\zeta' \in \sphere(\zeta, r_n)}\Big|f_n\big(Z_{n},\zeta'\big)-\E \left[ f_n\big(Z_{n},\zeta'\big) \right]\Big|
        \;\stochdom\;& \kappa_n \Ltt_n^\star(\zeta, r)(1 + r) + \Big(M_n^{(2)}(\zeta, r) +\sup_{\sphere(\zeta, r_n)} |f_n(\mathbf{0}, \zeta')|\Big) \rme^{-\constb n}.
    \end{align}
\end{proposition}

\begin{proof}
    We first control the probability $\Prob(Z_{n} \in \msk_n)$.
    To this end, we leverage that $Z_{n}$ is centered and has independant and $\kappa_n$-Lipschitz concentrated columns, 
    hence by \Cref{lemma:norm_lipschitz_concentrated_rows},
    we have for some $\const >0$,
    \begin{align}
        \Prob\left(\|Z_{n}\|_{\op}> 2 \consta_1 \kappa_n (\sqrt{n} + \sqrt{d_n}) \right) \leq 2 e^{-\const n} \eqsp.
    \end{align}
    where we used assumption \ref{proposition:Stochdom_local_lipschitz_transforms_a}  to simplify the inequality.
    Therefore, we get,
    \begin{align} \label{eq:high_prob_good_event_i}
        \Prob(Z_{n} \in \msk_n) \leq 2 \rme^{-\const n} \eqsp.
    \end{align}
    The proof of the proposition relies on the similar idea as the one used in \Cref{lemma:concentration_lipschitz_maps_all_t},
    throughout the proof, we fix $\zeta\in \C$ and $r>0$, and for $\zeta' \in \sphere(\zeta,r)$,  we define
    \begin{align}
        g_n(\cdot, \zeta') = g_{\msk_n}^{f_n(\cdot, \zeta')}, \quad \text{and} \quad \Delta_n = f_n - g_n \eqsp,\eqsp,
    \end{align}
    where $g_{\msk_n}^{f_n(\cdot, \zeta')}$was introduced in \eqref{eq:def_moreau_envelop}.
    Using the triangle inequality and denoting $\bar{f}_n = f_n(Z_{n},\zeta')-\E \left[ f_n(Z_{n},\zeta')\right]$, for any $\zeta'\in\sphere(\zeta,r)$
    \begin{align}
            \label{eq:ProbControlDelta_n_0}
            \Big|\bar{f}_n\big(Z_{n},\zeta'\big) \Big| 
    &\leq \Big|g_n\big(Z_{n},\zeta'\big) -\E \left[g_n\big(Z_n,\zeta'\big)\right]\Big| + \Bigl| \Delta_n\big(Z_n,\zeta'\big) \Bigr| + \E\left[\Bigl| \Delta_n\big(Z_n,\zeta'\big) \Bigr|\right]
    \end{align}
    and we recall from \Cref{lemma:MoreauExtension} that $g_n(\cdot, \zeta')$ is $\Ltt_n(\zeta')$-Lipschitz. Since $\Delta_n(\mathbf{Z}_{n},\zeta') = 0$ by definition, for any   $\mathbf{Z}_{n} \in \msk_n$, 
    relying on \eqref{eq:high_prob_good_event_i}, we have for any fixed $\delta > 0$,
    \begin{align}\label{eq:ProbControlDelta_n}
        \Prob\left(\Bigl| \Delta_n\big(Z_{n},\zeta'\big)\Bigr| > n^\delta \E\left[\Bigl| \Delta_n\big(Z_n,\zeta'\big) \Bigr|\right]\right) \leq \Prob\left(\Bigl| \Delta_n\big(Z_{n},\zeta'\big)\Bigr| > 0\right) \leq \Prob\left(Z_{n} \notin \msk_n\right) \leq 2 \rme^{-\const n}\eqsp,
    \end{align}
    this ensures for any $\zeta'\in\sphere(\zeta,r)$,
    \begin{equation}
      \label{eq:ProbControlDelta_n_i}
      \Bigl|\Delta_n\big(Z_{n},\zeta'\big) \Bigr| \stochdom \E\left[\Bigl| \Delta_n\big(Z_n,\zeta'\big) \Bigr|\right] \eqsp.
    \end{equation}
    Furthermore, from \Cref{lemma:MoreauExtension}, and because $\mathbf{0} \in \msk_n$ we have for any $\zeta'\in\sphere(\zeta,r)$
    \begin{align}
        \E\left[\Bigl| \Delta_n\big(Z_n,\zeta'\big) \Bigr|\right] \leq &\left(M_n^{(2)}(\zeta') +|f_n(\mathbf{0}, \zeta')| + \Ltt_n(\zeta') \|Z_{n}\|_\Frob \right) \Prob\left(Z_{n} \notin \msk_n\right)^{1/2}
    \end{align}
    Using our assumptions on the $Z_{n}$'s, we have,
    \begin{align}
        \E\left[\|Z_{n}\|_\Frob^2\right] & = n \tr\left(\E\left[\Sigma_{n}\right]\right)\eqsp,
    \end{align}
    where $\Sigma_{n} = n^{-1}\E\left[Z_{n} Z_{n}^\top\right] \in \R^{d_n \times d_n}$
    is the covariance matrix of the columns of $Z_{n}$.
    From the concentration assumption, as well as \Cref{lemma:Covariance_maximum_eigenvalue}
    we have $\tr(\Sigma_{n}) \leq d_n \|\Sigma_{n}\|_\op \leq d_n \kappa_n^2$ a.s.,
    therefore, by condition \ref{proposition:Stochdom_local_lipschitz_transforms_a}, for some $\const>0$, it holds
    \begin{align}
        \E\left[\|Z_{n}\|_{\Frob}^2\right] \leq n d_n \kappa_n^2 \leq \const^2 n^2 \kappa_n^2 \eqsp.
    \end{align}
    Hence, it holds for any $\zeta'\in\sphere(\zeta,r)$
    \begin{align}
            \label{eq:ProbControlDelta_n_ii}
            \E\left[\Bigl| \Delta_n\big(Z_n,\zeta'\big) \Bigr|\right] &\leq 2 \left(M_n^{(2)}(\zeta') +|f_n(\mathbf{0}, \zeta')| + \const \Ltt_n(\zeta') n \kappa_n \right) \rme^{-\constb n / 2}\\
        \end{align}
        where we used \eqref{eq:high_prob_good_event_i}.
    By additivity of $\stochdom$ \Cref{lemma:additivity_multiplicativity_stochdom}, 
    \eqref{eq:ProbControlDelta_n_0}-\eqref{eq:ProbControlDelta_n_i}-\eqref{eq:ProbControlDelta_n_ii} imply  
    that  for all $\delta, D > 0$, there exists $N(\delta, D)$ such that for all $n \geq N(\delta, D)$ and and $\zeta'\in\sphere(\zeta,r)$, and with probability at least $1-n^{-D}$,
    \begin{align}
        \Big|\bar{f}_n\big(Z_{n},\zeta'\big) \Big| \leq& \eqsp n^\delta\Big|g_n\big(Z_{n},\zeta'\big) - \E \left[ g_n\big(Z_{n},\zeta'\big) \right] \Big|
         & + 2  n^\delta\left(M_n^{(2)}(\zeta') +|f_n(\mathbf{0}, \zeta')| + \const \Ltt_n(\zeta') n \kappa_n \right) \rme^{-\constb n / 2} \eqsp.
        \label{eq:Control_Indep_n}
    \end{align}
    We further upper bound the term involving $g_n$, by leveraging the Lipschitz property of $g_n(\cdot, \zeta')$ and the Lipschitz-concentration of $Z_n$, we have for all $t \geq 0$,
    \begin{align}
        \Prob\left(\Big|g_n\big(Z_{n},\zeta'\big) - \E \left[ g_n\big(Z_{n},\zeta'\big) \right] \Big| \geq t\right) \leq 2 \exp \left(\dfrac{-t^2}{2 \kappa_n^2 \Ltt_n(\zeta')^2}\right)\eqsp,
    \end{align}
    thus taking $t = n^\delta \kappa_n \Ltt_n(\zeta')$, we get,
    \begin{align}
        \Prob\left(\Big|g_n\big(Z_{n},\zeta'\big) - \E \left[ g_n\big(Z_{n},\zeta'\big) \right] \Big| \geq n^\delta \kappa_n \Ltt_n(\zeta')\right) \leq 2 \exp \left(\dfrac{-n^{2\delta}}{2}\right)
    \end{align}
    and it holds with probability at least $1 - 2 \exp(-n^{2\delta} / 2) - n^{-D'}$ that
    \begin{align} \label{eq:fixed_zeta}
        \Big|\bar{f}_n\big(Z_{n},\zeta'\big) \Big| \leq n^\delta \left\{\kappa_n \Ltt_n(\zeta') + \left(M_n^{(2)}(\zeta') +|f_n(\mathbf{0}, \zeta')| + \const \Ltt_n(\zeta') n \kappa_n \right) \rme^{-\constb n / 2}\right\}\eqsp.
    \end{align}
    This concludes the proof in the case $r_n = 0$, for the general case, we use a $\varepsilon$-net technique to derive a control uniform over $\sphere(\zeta, r_n)$.
    Consider for $\varepsilon > 0$, the set
    \begin{equation}
    \label{eq:def_neps}
        \mathsf{N}_\varepsilon(\zeta,r) := \left\{\zeta + r \rme^{2 \rmi \pi \ell / k } : 0 \leq \ell \leq \lceil 1/\varepsilon \rceil-1 \right\}\eqsp,
    \end{equation}
    and for any $\zeta'\in\cset$,  $\proj(\zeta') \in \argmin_{\tilde{\zeta} \in \msn_{\varepsilon}(\zeta,r)} \abs{\zeta'-\tilde{\zeta}}$. Therefore,
    considering the parameterization $\theta \in [0,2\pi] \mapsto \zeta'(\theta)=z+r \rme^{i\theta}$, using a Taylor expansion and condition \ref{proposition:Stochdom_local_lipschitz_transforms_b}, we have for any $\zeta' \in\sphere(\zeta,r)$,
    \begin{align}
        &\sup_{\zeta'\in\sphere(\zeta,r)}\Big|\bar{f}_n(Z_{n},\zeta') -\bar{f}_n(Z_{n},\proj(\zeta')) \Big| \stochdom r \mathtt{D}_n(\zeta, r) \varepsilon \eqsp,
    \end{align}
    It yields
    \begin{align}
        &\sup_{\zeta'\in\sphere(\zeta,r)}\Big|f_n(Z_{n},\zeta')-\E \left[f_n(Z_{n},\zeta')\right]\Big| \\
        &\hspace{1.5cm} \stochdom\;
        \max_{\zeta'\in\mathsf N_\varepsilon}\Big|f_n(Z_{n},\zeta')-\E \left[f_n(Z_{n},\zeta') \right]\Big|         + r \,\mathtt{D}_n(\zeta, r)\,\varepsilon \eqsp. \label{eq:net-to-sup}
    \end{align}
    By a union bound, using \eqref{eq:fixed_zeta}, and setting
    \begin{align}
    t_n = \kappa_n \Ltt_n^\star(\zeta, r_n) + \left(M_n^{(2)}(\zeta, r_n) +|f_n(\mathbf{0}, \zeta')| + \const \Ltt_n^\star(\zeta, r_n) n \kappa_n \right) \rme^{-\constb n / 2} \eqsp,
    \end{align}
    we get for all $\delta > 0$, and $D > 0$, there exists $N(\delta,D)$ such that for  any $n\geq N(\delta,D)$,
    \begin{align}
        \Prob\!\left(\max_{\zeta'\in\mathsf N_\varepsilon}\Big|f_n(Z_{n},\zeta')-\E \left[f_n(Z_{n},\zeta')\right]\Big|\geq n^\delta t_n\right)
&        \leq \sum_{\zeta' \in \mathsf{N}_\varepsilon} \Prob\!\left(\Big|f_n(Z_{n},\zeta')-\E \left[ f_n(Z_{n},\zeta') \right]\Big|\geq n^\delta t_n\right) \\
        & \leq \left|\mathsf{N}_\varepsilon\right|  \left\{ 2\exp\left(-n^{2\delta} / 2\right) + n^{-D} \right\} \eqsp. \\
    \end{align}
    Finally, we choose $\varepsilon = \kappa_n  \Ltt_n^\star(\zeta, r_n) / \mathtt{D}_n(\zeta,r_n)$, so that \eqref{eq:net-to-sup} implies that for all $\delta > 0$, and $D > 0$, there exists $N(\delta,D)$ such that for  any $n\geq N(\delta,D)$,
    \begin{align}
        &\Prob\left(\sup_{\zeta' \in \sphere(\zeta,r_n)}\Big|f_n(Z_{n},\zeta')-\E f_n(Z_{n},\zeta')\Big| \geq n^\delta t_n + r_n \kappa_n \right)  \leq \dfrac{r \mathtt{D}_n(z,r)}{\kappa_n \Ltt_n^\star(z,r)} \left\{2k \exp\left(-n^{2\delta} / 2\right) + n^{-D} \right\}
    \end{align}
 Condition \ref{proposition:Stochdom_local_lipschitz_transforms_b} implies that for any $\delta,D' >0$, there exists $N(\delta,D')$ such that setting  $D = D' + 1 + \beta_2$, it holds
    \begin{align}
 \dfrac{r \mathtt{D}_n(z,r)}{\Ltt_n^\star(\zeta, r_n) \kappa_n} \left\{2k \exp\left(-n^{2\delta} / 2\right) + n^{-D'} \right\}
        \leq n^{-D'} \eqsp.
    \end{align}
 By the definition of stochastic domination, this shows, for any $\delta,D' >0$, there exists $N(\delta,D')$ such that setting  $D = D' + 1 + \beta_2$, it holds
    \begin{align}
        \sup_{\zeta'\in\sphere(\zeta,r)}\Big|f_n(Z_{n},\zeta')-\E f_n(Z_{n},\zeta')\Big| \stochdom t_n + r_n \Ltt_n^\star(\zeta, r_n) \kappa_n\eqsp.
    \end{align}
Which simplifies into
    \begin{align}
        \sup_{\zeta' \in\sphere(\zeta,r_n)}\Big|f_n(Z_{n},\zeta')-\E f_n(Z_{n},\zeta')\Big| \stochdom &\kappa_n (1+r_n)\Ltt_n^\star(\zeta, r_n) + \left(M_n^{(2)}(\zeta, r_n) +|f_n(\mathbf{0}, \zeta')| \right) \rme^{-\constb n / 2}\eqsp.
   \end{align}
   and concludes the proof.
\end{proof}

\section{Concentration of $\mathcal{G}(\alpha, \lambda)$ in the proportional regime} \label{section:concentration_generalization_error}

Building upon the preceding result, we are able to establish sharp
concentration inequalities for the augmented resolvent matrix and
related quantities.  Specifically, we demonstrate that this result
allows to obtain high-probability bounds for a range of transformations of $\hat{\theta}_{\alpha, \lambda}$, ultimately yielding the concentration of
$\mathcal{G}(\alpha, \lambda)$ as defined in
\eqref{eq:generalization_error}. 
We will consider quantities of the form,
\begin{equation}
    \Resolvent_{\alpha, \lambda}(\zeta) = \left((1-\alpha) \SampleCov + \alpha \SampleCov' + \alpha \Lambda(\Xtt,\Ytt) + \lambda \Idd + \zeta \Sigma\right)^{-1} \eqsp,
\end{equation}
\begin{equation}
    \mathrm{H}_\alpha = \left((1-\alpha) \SampleCrossCov  + \alpha \SampleCrossCov' + \alpha \Omega(\Xtt,\Ytt)\right)
\end{equation}
\begin{equation} \label{eq:def_xi_alpha_lambda_zeta}
 \text{and} \quad \xi_{\alpha, \lambda}(\zeta) = \SampleCrossCov_\alpha^\top \Resolvent_{\alpha, \lambda}(\zeta) \SampleCrossCov_\alpha \eqsp,
\end{equation}
for an $\zeta \in \msu_\lambda = \{\zeta \in \C : \Re(\zeta) > - \lambda / \Ltt_\varphi^2\}$. Note that $\zeta \in \msu_\lambda$,  \Cref{lemma:Covariance_maximum_eigenvalue} and \Cref{ass:data_distribution} 
ensure that $\min\{\Re(\lambda + \rho) : \rho \in \mathrm{Sp}(\zeta\Sigma)\} > 0$, thus garanteeing that $\xi_{\alpha,\lambda}(\zeta)$ is well-defined.
With these notations, $\mathcal{G}(\alpha, \lambda)$ in \eqref{eq:generalization_error_rewritten} rewrites as
\begin{equation} \label{eq:generalization_error_rewritten2}
    \mathcal{G}(\alpha, \lambda) = \theta_\star^\top \Sigma_{\star\star}\theta_\star + \sigma^2 - 2 \theta_\star^\top \Sigma_\star \hat{\theta}_{\alpha, \lambda} - \partial_\zeta \xi_{\alpha, \lambda}(0) \eqsp,
\end{equation} 
for any $\alpha \in [0,1]$ and $\lambda > 0$. Our objective is to show that, with high probability, $\mathcal{G}(\alpha, \lambda)$ remains close to its mean. To this end,
we will invoke \Cref{proposition:Stochdom_local_lipschitz_transforms} on $\theta_\star^\top \Sigma_{\star} \hat{\theta}_{\alpha, \lambda}$ and $\xi_{\alpha, \lambda}(\zeta)$ respectively. 
The concentration of the derivative $\partial_\zeta \xi_{\alpha, \lambda}(0)$ then follows from an appropriate differentiation lemma,
formulated below in the formalism of stochastic domination.

\begin{lemma}[Differentiation under uniform stochastic control] \label{lemma:differentiation_lemma}
    Let $(U_n)_{n \in \mathbb{N}}$ and $(V_n)_{n \in \mathbb{N}}$ be sequences of random holomorphic functions from $\C$ to $\C$. Fix $\zeta \in \C$ and $r > 0$. Suppose there exists $\xi_n(\zeta, r) > 0$ such that
    \begin{equation}\label{eq:uniform_disc_control}
        \sup_{\zeta'\in\sphere(\zeta,r)} \big|U_n(\zeta') - V_n(\zeta')\big| \;\stochdom\; \xi_n(\zeta, r) \eqsp.
    \end{equation}
    Then, the derivatives at $\zeta$ satisfy
    \begin{equation}\label{eq:derivative_bound}
        \left| \dfrac{\partial U_n}{\partial \zeta}(\zeta) - \dfrac{\partial V_n}{\partial \zeta}(\zeta) \right| \;\stochdom\; \frac{\xi_n(\zeta, r)}{r} \eqsp.
    \end{equation}
\end{lemma}
\begin{proof}
 By Cauchy's integral formula for the derivative of a holomorphic function,
    \begin{equation}
        \frac{\partial}{\partial \zeta} (U_n - V_n)(\zeta) = \frac{1}{2\uppi \rmi} \oint_{\zeta'\in\sphere(\zeta,r)} \frac{U_n(\zeta') - V_n(\zeta')}{(\zeta' - \zeta)^2} \, \rmd \zeta' \eqsp.
    \end{equation}
    Taking the module and  noting that the maximum of $|U_n(\zeta') - V_n(\zeta')|$ over the circle $\sphere(\zeta,r)$ controls the integral, we obtain
    \[
        \left| \frac{\partial}{\partial \zeta} (U_n - V_n)(\zeta) \right|
        \leq  \frac{1}{r} \sup_{\zeta'\in\sphere(\zeta,r)} |U_n(\zeta') - V_n(\zeta')| \eqsp.
    \]
    By the stochastic domination assumption \eqref{eq:uniform_disc_control}, we obtain the desired bound \eqref{eq:derivative_bound}.
\end{proof}

In order to apply this result to the sequence $\partial_\zeta \xi_{\alpha, \lambda_n}(0)$, we first prove the following result:
\begin{proposition}
    \label{proposition:Concentration_resolvent_around_mean}
    For all $n \in \N$, let $\lambda_n > 0$, $0 \leq r_n < \lambda_n / \Ltt_\varphi^2$ and $\mathbf{a}_n \in \R^p$. Assume that \Cref{ass:porportionality}, \Cref{ass:data_distribution} and \Cref{ass:artificial_data} hold.
    Then, we have,
    \begin{align}
        \left|\mathbf{a}_n^\top \hat{\theta}_{\alpha,\lambda_n} - \E\left[\mathbf{a}_n^\top \hat{\theta}_{\alpha,\lambda_n}\right]\right| \stochdom \dfrac{\|\mathbf{a}_n\|_2}{\lambda_n^{3/2} \sqrt{n}}\left(1 + \sqrt{\lambda_n}\right) \eqsp,
    \end{align}
    and,
    \begin{align}
        \sup_{t \in [0,2\uppi]}\left|\xi_{\alpha, \lambda_n}(r_n \rme^{\rmi t}) - \E\left[\xi_{\alpha, \lambda_n}(r_n \rme^{\rmi t})\right]\right| \stochdom \dfrac{1 + r_n}{(\lambda_n - r_n \Ltt_\varphi^2)^{3/2}\sqrt{n}}\left(1 + \sqrt{\lambda_n - r_n \Ltt_\varphi^2}\right) \eqsp.
    \end{align}
\end{proposition}

Before proving \Cref{proposition:Concentration_resolvent_around_mean}, we first introduce the following notations for the smallest/largest
real parts of eigenvalues associated to any complex matrix $\mathbf{D} \in \C^{d \times d}$:
\begin{equation}
    \uplambda_{\min}^\Re(\mathbf{D}) := \min_{\uplambda \in \mathrm{Sp}(\mathbf{D})} \Re(\uplambda) \eqsp,
    \quad \text{and} \quad
    \uplambda_{\max}^\Re(\mathbf{D}) := \max_{\uplambda \in \mathrm{Sp}(\mathbf{D})} \Re(\uplambda) \eqsp,
\end{equation}
which are used in the following lemma on the Lipschitz property of the resolvent matrix with complex regularizations:
\begin{lemma}
    \label{lemma:Lipschitz_resolvent_map}
    For any $n,p \in \nset$, let $\mathbf{D} \in \C^{p \times p}$ be a  complex matrix, and $\mathbf{X}_1, \mathbf{X}_2 \in \R^{p \times n}$
    be two matrices. Assume that $\mathrm{Sp}(\mathbf{D}) \subset \{\zeta \in \C : \Re(\zeta) > 0\}$, then  we have,
    \begin{align}
        \left\|\left(n^{-1} \mathbf{X}_1 \mathbf{X}_1^\top + \mathbf{D}\right)^{-1} - \left(n^{-1} \mathbf{X}_2 \mathbf{X}_2^\top + \mathbf{D}\right)^{-1}\right\|_\Frob \leq \dfrac{1}{\sqrt{n} \uplambda_{\min}^{\Re}(\mathbf{D})^{3/2}}\|\mathbf{X}_1 - \mathbf{X}_2\|_\Frob \eqsp,
    \end{align}
    and
    \begin{align}
        \left\|\left(n^{-1} \mathbf{X}_1 \mathbf{X}_1^\top + \mathbf{D}\right)^{-1} - \left(n^{-1} \mathbf{X}_2 \mathbf{X}_2^\top + \mathbf{D}\right)^{-1}\right\|_\op \leq \dfrac{1}{\sqrt{n} \uplambda_{\min}^{\Re}(\mathbf{D})^{3/2}}\|\mathbf{X}_1 - \mathbf{X}_2\|_\op \eqsp.
    \end{align}
\end{lemma}
\begin{proof}
    Write $\|\cdot\|$ to represent either $\|\cdot\|_\op$ or $\|\cdot\|_\Frob$. We first use the resolvent identity to write,
    \begin{align}
        &\left\|\left(n^{-1} \mathbf{X}_1 \mathbf{X}_1^\top + \mathbf{D}\right)^{-1} - \left(n^{-1} \mathbf{X}_2 \mathbf{X}_2^\top + \mathbf{D}\right)^{-1}\right\|\\
        =& p^{-1}\left\|\left(n^{-1} \mathbf{X}_1 \mathbf{X}_1^\top + \mathbf{D}\right)^{-1}\left\{ \mathbf{X}_1\mathbf{X}_1^\top - \mathbf{X}_2\mathbf{X}_2^\top\right\}\left(n^{-1} \mathbf{X}_2 \mathbf{X}_2^\top + \mathbf{D}\right)^{-1}\right\|\\
        =& p^{-1}\left\|\left(n^{-1} \mathbf{X}_1 \mathbf{X}_1^\top + \mathbf{D}\right)^{-1}\left\{ \mathbf{X}_1(\mathbf{X}_1 - \mathbf{X}_2)^\top + (\mathbf{X}_1 - \mathbf{X}_2) \mathbf{X}_2^\top \right\}\left(n^{-1} \mathbf{X}_2 \mathbf{X}_2^\top + \mathbf{D}\right)^{-1}\right\| \\
        \leq& \dfrac{\|(\mathbf{C}_2 + \mathbf{D})^{-1}\|_\op\|\left(\mathbf{C}_1 + \mathbf{D}\right)^{-1} \mathbf{X}_1\|_\op + \|(\mathbf{C}_1 + \mathbf{D})^{-1}\|_\op\|\left(\mathbf{C}_2 + \mathbf{D}\right)^{-1} \mathbf{X}_2\|_\op}{n} \|\mathbf{X}_1 - \mathbf{X}_2\| \eqsp,
        \label{eq:Apply_Resolvent_identity}
    \end{align}
    where we have introduced the notation $\mathbf{C}_i = n^{-1}\mathbf{X}_i \mathbf{X}_i^\top$ for $i \in \{1,2\}$.
    Furthermore, one can check that $\|(\mathbf{C}_i + \mathbf{D})^{-1}\|_\op = \max_{\uplambda \in \mathrm{Sp}(\mathbf{C}_i + \mathbf{D})} |\uplambda|^{-1} \leq 1/\uplambda_{\min}^\Re(\mathbf{D})$.
    Additionally, we bound the remaining operator norm, for both $i \in \{1,2\}$, it holds,
    \begin{align}
        \big\|\,(\mathbf{C}_i + \mathbf{D})^{-1}\mathbf{X}_i^\top\big\|_{\op}^2
        &= n\,\uplambda_{\max}\!\left((\mathbf{C}_i+\mathbf D)^{-1}\,\mathbf{C}_i\,\big((\mathbf{C}_i+\mathbf D)^{-1}\big)^{\!\dagger}\right) \\
        &= n \sup_{v \in \C^n\backslash \{0\}} \frac{v^\dagger (\mathbf{C}_i+\mathbf{D})^{-1} \mathbf{C}_i (\mathbf{C}_i+\mathbf{D}^\dagger)^{-1} v}{\|v\|^2_2} \\
        &= n \sup_{w \in \C^n\backslash \{0\}} \frac{w^\dagger \mathbf{C}_i w}{\|(\mathbf{C}_i + \mathbf{D}^\dagger)w\|^2_2} \eqsp.
    \end{align}
    where we used the change of variable $w = (\mathbf{C}_i + \mathbf{D}^\dagger)^{-1} v$ in the last line.
    Now remark that from Cauchy-Schwarz inequality, it holds,
    \begin{align}
        \|(\mathbf{C}_i + \mathbf{D}^\dagger) w\|_2 &\geq \dfrac{\left|\left\langle (\mathbf{C}_i + \mathbf{D}^\dagger) w, w \right\rangle\right|}{\|w\|_2} \geq \dfrac{\Re\left(\left\langle (\mathbf{C}_i + \mathbf{D}^\dagger) w, w \right\rangle\right)}{\|w\|_2} \\
        &\geq \dfrac{w^\dagger \mathbf{C}_i w + \uplambda_{\min}^{\Re}(\mathbf{D})\|w\|_2^2}{\|w\|_2} \eqsp.
    \end{align}
    Hence, defining $c = w^\dagger \mathbf{C}_i w / \|w\|_2^2 \geq 0$, we have,
    \begin{align}
        \big\|\,(\mathbf{C}_i+\mathbf D)^{-1}\mathbf X_i^\top\big\|_{\op}^2 &\leq \sup_{c \geq 0} \dfrac{c}{(c + \uplambda_{\min}^\Re(\mathbf{D}))^2} \eqsp.
    \end{align}
    To conclude the proof, we simply remark that the function
    $c \mapsto c / (c + \uplambda_{\min}^\Re(\mathbf{D}))^2$ is maximized at $c = \uplambda_{\min}^\Re(\mathbf{D})$, which yields,
    \begin{align}
        \big\|\,(\mathbf{C}_i + \mathbf D)^{-1}\mathbf X_i^\top\big\|_{\op}^2 &\leq \dfrac{1}{4\uplambda_{\min}^\Re(\mathbf{D})} \eqsp.
    \end{align}
    Plugging this result in \eqref{eq:Apply_Resolvent_identity} concludes the proof.
\end{proof}

\begin{proof}[Proof of \Cref{proposition:Concentration_resolvent_around_mean}] 
    Fir $\mathbf{a}_n \in \R^p$.
    First note that by assumption $\msu_{\lambda_n}$ contains $\{\lambda_n + r_n \rme^{\rmi t} : t \in [0, 2 \uppi)\}$ for all $n \in \N$, thus $\xi_{\alpha, \lambda_n}(r_n \rme^{\rmi t})$ is well-defined for all $t \in [0, 2 \uppi]$.
    The proof heavily relies on \Cref{proposition:Stochdom_local_lipschitz_transforms}, we first exhibit two functions $q : \R^{d \times n} \to \R$ and $s: \R^{d \times n} \times \C \to \C$ such that 
    \begin{align}
        \mathbf{a}_n^\top \hat{\theta}_{\alpha, \lambda_n} = q(\Xtt, \Ytt) \eqsp, \quad \text{and} \quad \xi_{\alpha, \lambda_n}(r_n \rme^{\rmi t}) = s(\Xtt, \Ytt, r_n \rme^{\rmi t}) \eqsp.
    \end{align}
    To this end, we define for any $\bfX \in \R^{d \times n}$, $\bfY \in \R^{d \times n}$, and $\zeta \in \C$,
    \begin{align} \label{eq:def_r_and_h}
        r(\bfX, \bfY, \zeta) &= (n^{-1} \varphi(\bfX) \varphi(\bfX)^\top + n^{-1} \mu_x(\bfX, \bfY) \mu_x(\bfX, \bfY)^\top + \alpha\Lambda(\bfX, \bfY) + \lambda_n \Idd + \zeta \Sigma)^{-1} \eqsp, \\
        h(\bfX, \bfY, \zeta) &= n^{-1} \varphi(\bfX) \bfY^\top + n^{-1} \mu_x(\bfX, \bfY) \mu_y(\bfX, \bfY)^\top + \alpha \Omega(\bfX,\bfY) \eqsp.
    \end{align}
    where $\mu_x, \mu_y$, $\Lambda$ and $\Omega$ were introduced in \Cref{ass:artificial_data}. $r$ and $h$ allow to constructions $q$ and $s$ as follows:
    \begin{align}
        \forall (\bfX, \bfY, \zeta) \in \R^{d \times n} \times \R^{d \times n} \times \C, \quad q(\bfX, \bfY) &= \mathbf{a}_n^\top r(\bfX, \bfY, 0) h(\bfX, \bfY) \eqsp,\\
        \text{and} \quad s(\bfX, \bfY, \zeta) &=  h(\bfX,\bfY)^\top r(\bfX, \bfY, \zeta) h(\bfX, \bfY) \eqsp.
    \end{align}
    In line with the hypothesis of \Cref{proposition:Stochdom_local_lipschitz_transforms}, we now show that $q(\cdot, \cdot)$ and $s(\cdot, \cdot, \zeta)$ are Lipschitz continuous on the compact set
    \begin{align}
        \msk_n = \left\{(\bfX, \bfY) \in \R^{d \times n} \times \R^{1 \times n} : \|(\bfX^\top, \bfY^\top)\|_{\op} \leq 2 \consta (\sqrt{n} + \sqrt{d+1}) \right\} \eqsp,
    \end{align}
    where $\consta > 0$ is defined in \eqref{eq:def_a_1}.
    To this end, let us first derive $\Ltt_{\msk_n}^r(\zeta)$ and $\Ltt_{\msk_n}^h$, the Lipschitz constants of $r(\cdot, \cdot, \zeta)$ and $h(\cdot, \cdot, \zeta)$ on $\msk_n$. 
    Fix any $\zeta \in \msu_{\lambda_n}$ and $(\bfX, \bfY)$ and $(\widetilde{\bfX}, \widetilde{\bfY})$ in $\msk_n$, we have from \Cref{lemma:Lipschitz_resolvent_map}, 
    the Lipschitz property assumption on $\Lambda$ and $\mu_x$, from \Cref{ass:artificial_data}, and the lipschitz property of $\mathbf{M} \mapsto (\Idd + \mathbf{M})^{-1}$ (provided $\Re(\mathrm{Sp}(\mathbf{M})) \subset \R_+$),
    \begin{align}
        &\left\|r(\bfX, \bfY, \zeta) - r(\widetilde{\bfX}, \widetilde{\bfY}, \zeta)\right\|_\op \\
        \leq&\eqsp \dfrac{\sqrt{2}}{\uplambda_{\min}^{\Re}(\Lambda(\bfX,\bfY) + \lambda \Idd + \zeta \Sigma)^{3/2} \sqrt{n}} \left\|(\varphi(\bfX), \mu_x(\bfX,\bfY)) - (\varphi(\widetilde{\bfX}), \mu_x(\widetilde{\bfX},\widetilde{\bfY}))\right\|_{\op} \\
        &+ \dfrac{1}{\uplambda_{\min}^{\Re}(\Lambda(\bfX,\bfY) + \lambda \Idd + \zeta \Sigma)} \left\|\Lambda(\bfX,\bfY) - \Lambda(\widetilde{\bfX},\widetilde{\bfY})\right\|_\op \\
        \leq&\eqsp \dfrac{\const(1 + \uplambda_{\min}^{\Re}(\Lambda(\bfX,\bfY) + \lambda \Idd + \zeta \Sigma)^{1/2})}{\uplambda_{\min}^{\Re}(\Lambda(\bfX,\bfY) + \lambda \Idd + \zeta \Sigma)^{3/2} \sqrt{n}} \left\|(\bfX, \bfY) - (\widetilde{\bfX}, \widetilde{\bfY})\right\|_{\op} \\
        \leq&\eqsp \dfrac{\const(1 + (\lambda_n - |\zeta| \Ltt_\varphi^2)^{1/2})}{(\lambda_n - |\zeta| \Ltt_\varphi^2)^{3/2} \sqrt{n}} \left\|(\bfX, \bfY) - (\widetilde{\bfX}, \widetilde{\bfY})\right\|_{\op}\eqsp,
    \end{align}
    where we have used $\|\Sigma\|_\op \leq \Ltt_\varphi^2$, which follows from \Cref{ass:data_distribution} and \Cref{lemma:Covariance_maximum_eigenvalue}. It results,
    \begin{align} \label{eq:Lipschitz_constant_r}
        \Ltt_{\msk_n}^r(\zeta) \leq \dfrac{\const(1 + (\lambda_n - r_n \Ltt_\varphi^2)^{1/2})}{(\lambda_n - r_n \Ltt_\varphi^2)^{3/2} \sqrt{n}} \eqsp.
    \end{align}
    Similarly for $h$, we use the fact that $(\bfX,\bfY) \in \msk_n$ and $(\widetilde{\bfX}, \widetilde{\bfY}) \in \msk_n$ to write,
    \begin{align}
        \left\|h(\bfX, \bfY) - h(\widetilde{\bfX}, \widetilde{\bfY})\right\|_\op 
        \leq&\eqsp \dfrac{\const \left(\sqrt{n} + \sqrt{d}\right)}{n}\biggl(\left\|\bfY - \widetilde{\bfY} \right\|_\op + \left\|\mu_y(\bfX, \bfY) - \mu_y(\widetilde{\bfX}, \widetilde{\bfY}) \right\|_\op \\
        &\hspace{2.4cm} + \left\|\bfX - \widetilde{\bfX} \right\|_\Frob + \left\|\mu_x(\bfX, \bfY) - \mu_x(\widetilde{\bfX}, \widetilde{\bfY}) \right\|_\op\biggr) \\
        &+ \left\|\Omega(\bfX, \bfY) - \Omega(\widetilde{\bfX}, \widetilde{\bfY})\right\|_\op \\
        \leq& \dfrac{\const}{\sqrt{n}} \left\|(\bfX, \bfY) - (\widetilde{\bfX}, \widetilde{\bfY})\right\|_{\op}\eqsp,
    \end{align}
    thus,
    \begin{align} \label{eq:Lipschitz_constant_h}
        \Ltt_{\msk_n}^h \leq \dfrac{\const}{\sqrt{n}} \eqsp.
    \end{align}
    Combining \eqref{eq:Lipschitz_constant_r} and \eqref{eq:Lipschitz_constant_h}, as well as using the fact that for any $(\bfX,\bfY) \in \msk_n$,
    \begin{align}
        \|r(\bfX, \bfY, \zeta)\|_\op \leq \dfrac{1}{\lambda_n - |\zeta| \Ltt_\varphi^2} \eqsp, \quad \text{and} \quad \|h(\bfX, \bfY)\|_\op \leq \const \eqsp,
    \end{align}
    we obtain,
    \begin{align} \label{eq:Lipschitz_constant_q}
        \Ltt_{\msk_n}^q \leq \dfrac{\|\mathbf{a}_n\|_2\Ltt_{\msk_n}^h}{\lambda_n} + \const \|\mathbf{a}_n\|_2 \Ltt_{\msk_n}^r(0) \leq \dfrac{\const\|\mathbf{a}_n\|_2(1 + \lambda_n^{1/2})}{\lambda_n^{3/2} \sqrt{n}}  \eqsp.
    \end{align}
    and,
    \begin{align} \label{eq:Lipschitz_constant_s}
        \Ltt_{\msk_n}^s(\zeta) \leq \dfrac{\const(1 + (\lambda_n - r_n \Ltt_\varphi^2)^{1/2})}{(\lambda_n - r_n \Ltt_\varphi^2)^{3/2} \sqrt{n}} \eqsp.
    \end{align}
    We now check that $(\Xtt,\Ytt)$ validates the assumptions of \Cref{proposition:Stochdom_local_lipschitz_transforms} for the function $q$ and $s$. 
    First remark that \Cref{proposition:Stochdom_local_lipschitz_transforms_a} of \Cref{proposition:Stochdom_local_lipschitz_transforms} is guarenteed by \Cref{ass:porportionality}. 
    We thus focus on proving that  \Cref{proposition:Stochdom_local_lipschitz_transforms_b} of \Cref{proposition:Stochdom_local_lipschitz_transforms} holds.
    Note that $q$ satisfies  \Cref{proposition:Stochdom_local_lipschitz_transforms_b} trivially. To prove that $s$ satisfies \Cref{proposition:Stochdom_local_lipschitz_transforms_b}, we first note that for any $t \in [0,2\uppi]$,
    \begin{align}
        \partial_\zeta s(\Xtt,\Ytt,r_n \rme^{\rmi t}) = -h(\Xtt,\Ytt)^\top r(\Xtt,\Ytt,r_n \rme^{\rmi t}) \Sigma r(\Xtt,\Ytt,r_n \rme^{\rmi t}) h(\Xtt, \Ytt) \eqsp,
    \end{align}
    additionally we write,
    \begin{align}
        \left|\partial_\zeta s(\Xtt,\Ytt,r_n \rme^{\rmi t}) - \E\left[\partial_\zeta s(\Xtt,\Ytt,r_n \rme^{\rmi t})\right]\right| 
        &\leq
        \left|\partial_\zeta s(\Xtt,\Ytt,r_n \rme^{\rmi t})\right| + \E\left[\left|\partial_\zeta s(\Xtt,\Ytt,r_n \rme^{\rmi t})\right|\right] \\
        &\leq
        \dfrac{1}{\lambda_n - r_n \Ltt_\varphi^2}\left\|h(\Xtt,\Ytt)^\top\right\|_2^2 + \dfrac{1}{\lambda_n - r_n \Ltt_\varphi^2}\E\left[\left\|h(\Xtt,\Ytt)^\top\right\|_2^2\right] \\
        &\stochdom 
        \dfrac{1}{\lambda_n - r_n \Ltt_\varphi^2} + \dfrac{1}{\lambda_n - r_n \Ltt_\varphi^2}\E\left[\left\|h(\Xtt,\Ytt)^\top\right\|_2^2\right]
    \end{align}
    Furthermore, from the Lipschitz properties of $\varphi$, $\mu_x$, $\mu_y$ and $\Omega$, we have,
    \begin{align}
        \left\|h(\Xtt,\Ytt)^\top\right\|_2^2 \leq \dfrac{\const}{n} \E\left[\|(\Xtt,\Ytt)\|_\op^2\right] \leq \const' =: \mathrm{D}_n(r_n \rme^{\rmi t})\eqsp,
    \end{align}
    It results that the ratio $\mathrm{D}_n(r_n \rme^{\rmi t}) / \Ltt_{\msk_n}^q(r_n \rme^{\rmi t})$ is bounded by a constant, 
    independant of $t$ thus \Cref{proposition:Stochdom_local_lipschitz_transforms_b} of \Cref{proposition:Stochdom_local_lipschitz_transforms} holds.
    Having established the correctness of the assumptions of \Cref{proposition:Stochdom_local_lipschitz_transforms} for $q$ and $s$, 
    the statement follows by applying \Cref{proposition:Stochdom_local_lipschitz_transforms}.
\end{proof}

Combining previous results, we obtain the following statement, which garentees the concentration of $\mathcal{G}(\alpha, \lambda)$ around its mean:
\begin{theorem}
    \label{theorem:concentration_generalization_error}
    For all $n \geq 0$, let $\lambda_n > 0$ and $r_n \geq 0$, and assume \Cref{ass:porportionality}, \Cref{ass:data_distribution} and \Cref{ass:artificial_data} hold. 
    Then the following concentration bounds hold:
    \eqref{eq:generalization_error} satisfies,
    \begin{align}
        \left|\theta_*^\top \Sigma_* \hat{\theta}_{\alpha, \lambda_n} - \E\left[\theta_*^\top \Sigma_* \hat{\theta}_{\alpha, \lambda_n}\right]\right| \stochdom \dfrac{\|\beta_\star\|_2}{\lambda_n^{3/2} \sqrt{n}}\left(1 + \sqrt{\lambda_n}\right) \eqsp,
    \end{align}
    \begin{equation}
       \sup_{t \in [0,2\uppi]}\left|\xi_{\alpha, \lambda_n}(r_n \rme^{\rmi t}) - \E\left[\xi_{\alpha, \lambda_n}(r_n \rme^{\rmi t})\right]\right| \stochdom \dfrac{1 + r_n}{(\lambda_n - r_n \Ltt_\varphi^2)^{3/2}\sqrt{n}}\left(1 + \sqrt{\lambda_n - r_n \Ltt_\varphi^2}\right) \eqsp,
    \end{equation}
    and the generalization error satisfies,
    \begin{align}
        \left|\mathcal{G}(\alpha, \lambda_n) - \E[\mathcal{G}(\alpha, \lambda_n)]\right| \stochdom \dfrac{1 + \lambda_n^2}{\lambda_n ^{5/2}\sqrt{n}} + \dfrac{\|\beta_\star\|_2(1 + \sqrt{\lambda_n})}{\lambda_n^{3/2} \sqrt{n}}\eqsp.
    \end{align}
\end{theorem}

\begin{proof}
    \hypertarget{proof:theorem1}{}
    We leverage \eqref{eq:generalization_error_rewritten2}, as well as the previously derived concentration results for $\xi_{\alpha, \lambda_n}$ and $\theta_*^\top \Sigma_* \hat{\theta}_{\alpha, \lambda_n}$.
    Indeed we have, from \Cref{proposition:Concentration_resolvent_around_mean} that,
    \begin{equation} \label{eq:concentration_theta_Sigma_Qfun}
        \left|\theta_*^\top \Sigma_* \hat{\theta}_{\alpha, \lambda_n} - \E\left[\theta_*^\top \Sigma_* \hat{\theta}_{\alpha, \lambda_n}\right]\right| \stochdom \dfrac{\|\beta_\star\|_2}{\lambda_n^{3/2} \sqrt{n}}\left(1 + \sqrt{\lambda_n}\right) \eqsp,
    \end{equation}
    where the term $\|\Sigma_\star\|_\op \leq \Ltt_\varphi^2$ has been stashed inside the $\stochdom$ notation, by using \Cref{lemma:Covariance_maximum_eigenvalue}.
    Additionally, from \Cref{proposition:Concentration_resolvent_around_mean} applied with $0< r_n< \lambda / \Ltt_\varphi^2$, we get,
    \begin{equation}
        \sup_{t \in [0,2\uppi]}\left|\xi_{\alpha, \lambda_n}(r_n \rme^{\rmi t}) - \E\left[\xi_{\alpha, \lambda_n}(r_n \rme^{\rmi t})\right]\right| \stochdom \dfrac{1 + r_n}{(\lambda_n - r_n \Ltt_\varphi^2)^{3/2}\sqrt{n}}\left(1 + \sqrt{\lambda_n - r_n \Ltt_\varphi^2}\right) \eqsp.
    \end{equation}
    Combining the previous with \Cref{lemma:differentiation_lemma}, we have, for all $r < \lambda / \Ltt_\varphi^2$,
    \begin{align} \label{eq:concentration_derivative_Sfun}
        \left|\partial_\zeta \xi_{\alpha, \lambda_n}(0) - \E\left[\partial_\zeta \xi_{\alpha, \lambda_n}(0) \right]\right| \stochdom \dfrac{1 + r_n}{r_n(\lambda_n - r_n \Ltt_\varphi^2)^{3/2}\sqrt{n}}\left(1 + \sqrt{\lambda_n - r_n \Ltt_\varphi^2}\right) \eqsp.
    \end{align}
    which, for $r_n = \lambda_n / (2 \Ltt_\varphi^2)$ gives,
    \begin{align}
        \left|\partial_\zeta \xi_{\alpha, \lambda_n}(0)  - \E\left[\partial_\zeta \xi_{\alpha, \lambda_n}(0) \right]\right| \stochdom \dfrac{1 + \lambda_n^2}{\lambda_n ^{5/2}\sqrt{n}} \eqsp.
    \end{align}
    where we stashed the dependancy in $\Ltt_\varphi$ inside the $\stochdom$ notation.
    We conclude the proof by merging \eqref{eq:generalization_error_rewritten2}, \eqref{eq:concentration_theta_Sigma_Qfun} and \eqref{eq:concentration_derivative_Sfun}, and using the additivity of $\stochdom$.
\end{proof}

\section{New determinisitic equivalents under correlated data} \label{section:new_deterministic_equivalences}

In this part of the appendix, we state and prove new deterministic equivalents results that generalize the ones of \Cref{proposition:det_equivs}. 
In particular, \Cref{prop:Detequiv_R_twobytwo}, \Cref{prop:DetEquiv_Q_noaug}, and \Cref{prop:DetEquiv_S_noaug} generalize proposition A.4 of \cite{Schroeder2024} to the context of non-i.i.d. data. 
This results are the first of their kind, with the exception of \cite{MorissetHardyDurmus2025} which deals with a slightly different setting and only 
focus on the resolvent matrix. Their result is an hedge case of \Cref{prop:Detequiv_R_twobytwo}. 

First, we introduce two useful identities, namely the resolvent identity, which states for any two invertible matrices 
$\mathbf{A}, \mathbf{B} \in \R^{d\times d}$, that,
\begin{align} \label{eq:resolvent_identity}
    \mathbf{A}^{-1} - \mathbf{B}^{-1} = \mathbf{A}^{-1} \left\{\mathbf{B} - \mathbf{A}\right\} \mathbf{B}^{-1} \eqsp.
\end{align}
And the leave-two-out trick, which is abstracted as the bellow \Cref{lemma:Leave-One-Out-Lemma}, and will allow us to desantangle 
the contribution of a pair of sample from the rest of the data.
\begin{lemma}\label{lemma:Leave-One-Out-Lemma}
    Let $\mathbf{A} \in \R^{p \times p}$, $\mathbf{U} \in \R^{p \times k}$ and $\mathbf{V} \in \R^{k \times p}$, then,
    \begin{align}
        \left(\mathbf{A} + \mathbf{U}\mathbf{V}\right)^{-1}\mathbf{U} = \mathbf{A}^{-1} \mathbf{U} \left(\Id{k} + \mathbf{V}\mathbf{A}^{-1} \mathbf{U}\right)^{-1}          
    \end{align}
\end{lemma}
\begin{proof}
    We have thanks to the Woodburry matrix inversion lemma \cite{HendersonSearle1981_WoodburyFormula} that, 
    \begin{align}
        \left(\mathbf{A} + \mathbf{U}\mathbf{V}\right)^{-1}\mathbf{U} &= \left(\mathbf{A}^{-1} - \mathbf{A}^{-1}\mathbf{U} \left(\Id{k} + \mathbf{V}\mathbf{A}^{-1} \mathbf{U}\right)^{-1} \mathbf{V} \mathbf{A}^{-1} \right)  \mathbf{U} \\
        &= \mathbf{A}^{-1}\mathbf{U} \left\{\Id{k} -  \left(\Id{k} + \mathbf{V}\mathbf{A}^{-1} \mathbf{U}\right)^{-1} \mathbf{V} \mathbf{A}^{-1} \mathbf{U}\right\} \\
        &= \mathbf{A}^{-1} \mathbf{U} \left(\Id{k} + \mathbf{V}\mathbf{A}^{-1} \mathbf{U}\right)^{-1} \eqsp.
    \end{align}
\end{proof}

We consider a variant of the usual deterministic equivalent result of sample covariance resolvents.
More precisely, we will consider a sample covariance matrix build from non-i.i.d. samples, in the particular
setting where we have i.i.d. pairs $(\Phi_i^{(1)}, \Phi_i^{(2)})$, for $i \in \{1, \ldots, n\}$.
We define
\begin{align}
    \Phi^{(1)} = (\Phi_i^{(1)})_{i=1}^n \in \R^{p \times k} \eqsp, \quad  \text{ and } \quad \Phi^{(2)} = (\Phi_i^{(2)})_{i=1}^n \in \R^{p \times n} \eqsp,
\end{align}
and consider the sample covariance matrices $\SampleCov^{(1)}$ and $\SampleCov^{(2)}$ of $\Phi^{(1)}$ and $\Phi^{(2)}$ respectively defined as,
\begin{align}
    \SampleCov^{(1)} = \dfrac{1}{n} \Phi^{(1)} \Phi^{(1)\top} \in \R^{p \times p} \eqsp, \quad \text{ and } \quad \SampleCov^{(2)} = \dfrac{1}{n} \Phi^{(2)} \Phi^{(2)\top} \in \R^{p \times p} \eqsp.
\end{align}
as well as the resolvent matrix of $\SampleCov^{(1)} + \SampleCov^{(2)}$, for any complex matrix $\bfD \in \C^{p \times p}$ with $\uplambda_{\min}^\Re(\bfD) > 0$,
\begin{align}
    \Resolvent(\bfD) = \left(\SampleCov^{(1)} + \SampleCov^{(2)} + \bfD\right)^{-1} \eqsp.
\end{align}
Previous works such as \cite{chouard2022quantitative} do not apply in our setting, 
as $\Phi^{(1)}$ and $\Phi^{(2)}$ are not independent. 
Perhaps surprisingly, this lack of independence fundamentally alters the structure 
of the deterministic equivalent of $\Resolvent(\bfD)$, rendering the expressions 
derived in \cite{chouard2022quantitative} inapplicable. 
For the sake of brevity, and throughout this appendix, we focus solely on 
approximating expectations $\E[Q]$ for the quantities $Q$ under consideration, 
and we do not explicitly address their concentration around the mean. 
Nevertheless, under a joint concentration assumption on the pairs 
$(\Phi_i^{(1)}, \Phi_i^{(2)})$ for $i \in \{1, \ldots, n\}$—for instance, assuming 
joint Lipschitz concentration—these quantities naturally concentrate around 
their expectations. 
In that case, the proof would follow similar lines to 
\Cref{proposition:Concentration_resolvent_around_mean}. 
Finally, we stress that particular effort was devoted to deriving 
\emph{non-asymptotic bounds} that hold for a finite number of samples $n$ and a 
finite feature dimension $p$.

We first prove:
\begin{proposition}
    \label{prop:Detequiv_R_twobytwo}
    Let $\Psi \in \R^{2p \times n}$ be a random matrix with centered, i.i.d., and $\kappa_n$-Lipschitz concentrated columns.
    Let $\bfD \in \C^{p \times p}$ be a deterministic matrix with $\uplambda_{\min}^\Re(\bfD) > 0$.
    Denote $\Phi^{(1)} = \Psi_{1:p, :} \in \R^{p \times n}$ and $\Phi^{(2)} = \Psi_{p+1:2p, :}\in \R^{p \times n}$, as well as,
    \begin{equation} \label{def:resolvent_in_statement}
        \Resolvent(\bfD) = \left(n^{-1} \Phi^{(1) \top} \Phi^{(1)} + n^{-1} \Phi^{(2) \top} \Phi^{(2)} + \bfD\right)^{-1} \eqsp.
    \end{equation}
    Denote $\Phi_1 = (\Phi^{(1)}_1, \Phi^{(2)}_1) \in \R^{p \times 2}$ and write
    \begin{align}
        \detequiv(\bfD) = \left(\E\left[\Phi_1 \Brm(\bfD) \Phi_1^\top\right]+ \bfD \right)^{-1} \eqsp,
        \quad \text{ where } \quad
        \Brm(\bfD) = \left(\Id{2} + n^{-1}\E\left[\Phi_1^\top \E\left[\Resolvent(\bfD)\right]\Phi_1\right]\right)^{-1} \eqsp.
    \end{align}
    Then, it holds for a large enough universal constant $\const$,
    \begin{align}
        \left\| \E\left[\Resolvent(\bfD)\right] - \detequiv(\bfD) \right\|_\Frob
        \leq \eqsp \dfrac{\const \sqrt{p} (1+ \kappa_n^8)}{n \uplambda_{\min}^\Re(\bfD)^{7/2}} \left\{\sqrt{\uplambda_{\min}^\Re(\bfD)}\left(1 + \dfrac{\sqrt{p}}{n}\right) + \dfrac{\sqrt{p}}{n} + \dfrac{\sqrt{p}}{\sqrt{\uplambda_{\min}^\Re(\bfD)} n}\right\} \eqsp.
    \end{align}
\end{proposition}
\begin{proof}
    Introduce the following notations,
    \begin{align}
        \SampleCov^{(1)} = \dfrac{1}{n}\Phi^{(1)} \Phi^{(1)\top} \in \R^{p \times p} \eqsp, \quad \text{ and } \quad \SampleCov^{(2)} = \dfrac{1}{n}\Phi^{(2)} \Phi^{(2)\top} \in \R^{p \times p} \eqsp,
    \end{align}
    \begin{align}
        \SampleCov_-^{(1)} = \SampleCov^{(1)} - \dfrac{1}{n} \Phi_1^{(1)} \Phi_1^{(1)\top} \in \R^{p \times p} \eqsp, \quad \text{ and } \quad \SampleCov_-^{(2)} = \SampleCov^{(2)} - \dfrac{1}{n} \Phi_1^{(2)} \Phi_1^{(2)\top} \in \R^{p \times p} \eqsp,
    \end{align}
    \begin{align}
        \Sigma_1 = \E\left[\Phi_1^{(1)} \Phi_1^{(1)\top}\right] \eqsp, \quad \Sigma_2 = \E\left[\Phi_1^{(2)} \Phi_1^{(2)\top}\right] \eqsp, \quad \Sigma_{1,2} = \E\left[\Phi_1^{(1)} \Phi_1^{(2)\top}\right] \eqsp,
    \end{align}
    and,
    \begin{align}
        \Resolvent_-(\bfD) = (\SampleCov^{(1)}_- + \SampleCov^{(2)}_- + \bfD)^{-1} \eqsp.
    \end{align}
    We first notice that thanks to the above notations, we have,
    \begin{align}
        \Resolvent(\bfD) = \left(\SampleCov_-^{(1)} + \SampleCov_-^{(2)} + \dfrac{1}{n}\Phi_1 \Phi_1^\top + \bfD\right)^{-1} \eqsp,
    \end{align}
    thus, using the matrix inversion \Cref{lemma:Leave-One-Out-Lemma}, we get,
    \begin{align} \label{eq:Woodbury_Resolvent_identity}
        \Resolvent(\bfD) = \Resolvent_-(\bfD) - \dfrac{1}{n} \Resolvent_-(\bfD) \Phi_1 \left(\Id{2} + \dfrac{1}{n}\Phi_{1}^\top\Resolvent_-(\bfD) \Phi_1\right)^{-1} \Phi_1^\top \Resolvent_-(\bfD) \eqsp,
    \end{align}
    We will denote $\hat{\Brm}(\bfD) = \left(\Id{2} + n^{-1}\Phi_{1}^\top \Resolvent_-(\bfD) \Phi_1\right)^{-1}$, the dilation matrix
    that naturally appears in the above application of \Cref{lemma:Leave-One-Out-Lemma}.
    Multiplying by $\Phi_1$ in \eqref{eq:Woodbury_Resolvent_identity}, we get,
    \begin{align} \label{eq:LeaveTwoOutIdentity}
        \Resolvent(\bfD) \Phi_1 = \Resolvent_-(\bfD) \Phi_1 \hat{\Brm}(\bfD) \eqsp.
    \end{align}
    Leveraging \eqref{eq:LeaveTwoOutIdentity}, as well as the resolvent identity \eqref{eq:resolvent_identity}, it results that,
    \begin{align}
        \E\left[\Resolvent(\bfD)\right] - \detequiv(\bfD) =& \eqsp \E\left[\Resolvent(\bfD)\left\{\E\left[\Phi_1 \Brm(\bfD) \Phi_1^\top\right] - \SampleCov^{(1)} - \SampleCov^{(2)}\right\}\detequiv(\bfD) \right] \\
        =& \eqsp \E\left[\Resolvent(\bfD)\left\{\E\left[\Phi_1 \Brm(\bfD) \Phi_1^\top\right] - \Phi_1 \Phi_1^\top\right\}\detequiv(\bfD) \right] \\
        =& \eqsp \E\left[\Resolvent(\bfD) - \Resolvent_-(\bfD)\right]\E\left[\Phi_1 \Brm(\bfD) \Phi_1^\top\right] \detequiv(\bfD) \\
        &+ \E\left[\Resolvent_-(\bfD) \left\{\E\left[\Phi_1 \Brm(\bfD) \Phi_1^\top\right] - \Phi_1 \Brm(\bfD) \Phi_1^\top\right\}\right] \detequiv(\bfD) \\
        &+ \E\left[\Resolvent_-(\bfD) \Phi_1 \left\{ \Brm(\bfD)  -  \hat{\Brm}(\bfD) \right\}\Phi_1^\top\right] \detequiv(\bfD) \eqsp,
    \end{align}
    which can be further simplified using
    \begin{align}
        \E\left[\Resolvent_-(\bfD) \left\{\E\left[\Phi_1 \Brm(\bfD) \Phi_1^\top\right] - \Phi_1 \Brm(\bfD) \Phi_1^\top\right\}\right] = 0 \eqsp,
    \end{align}
    as well as using \eqref{eq:Woodbury_Resolvent_identity} again, this proves,
    \begin{align} \label{eq:Resolvent_decomposition_leavetwoout}
        \E\left[\Resolvent(\bfD)\right] - \detequiv(\bfD)
        &= \dfrac{1}{n} \E\left[\Resolvent_-(\bfD) \Phi_1 \hat{\Brm}(\bfD) \Phi_1^\top \Resolvent_-(\bfD)\right]\E\left[\Phi_1 \Brm(\bfD) \Phi_1^\top\right] \detequiv(\bfD) \\
        &\quad + \E\left[\Resolvent_-(\bfD) \Phi_1 \left\{ \Brm(\bfD)  -  \hat{\Brm}(\bfD) \right\}\Phi_1^\top\right] \detequiv(\bfD) \eqsp.
    \end{align}
    The remainging of the proof focuses on bounding the Frobenius norms of the two terms in the right-hand side of \eqref{eq:Resolvent_decomposition_leavetwoout}.
    To this end, we first remark that,
    \begin{align}
        \hat{\Brm}(\bfD) =\left(\Id{2} + n^{-1}\Phi_1^\top \Resolvent_-(\bfD)\Phi_1\right)^{-1} \preceq \Id{2} \eqsp, \quad \text{and} \quad  \Brm(\bfD) =\left(\Id{2} + n^{-1}\E\left[\Phi_1^\top \E\left[\Resolvent(\bfD)\right]\Phi_1\right]\right)^{-1} \preceq \Id{2}\eqsp.
    \end{align}
    This implies in particular that,
    \begin{align} \label{eq:B_matrix_bound_1}
        \E\left[\Phi_1 \Brm(\bfD) \Phi_1^\top\right] \preceq \E\left[\Phi_1 \Phi_1^\top\right] = \Sigma_1 + \Sigma_2 + \Sigma_{1,2} + \Sigma_{1,2}^\top \eqsp.
    \end{align}
    and,
    \begin{align}
        \E\left[\Resolvent_-(\bfD) \Phi_1 \hat{\Brm}(\bfD) \Phi_1^\top \Resolvent_-(\bfD)\right]
        \preceq \E\left[\Resolvent_-(\bfD) \Phi_1 \Phi_1^\top \Resolvent_-(\bfD)\right]
        = \E\left[\Resolvent_-(\bfD) \left\{ \Sigma_1 + \Sigma_2 + \Sigma_{1,2} + \Sigma_{1,2}^\top \right\} \Resolvent_-(\bfD)\right] \eqsp,
    \end{align}
    where we have used the tower property and $\preceq$ denotes the Lowner ordering on positive semi-definite matrices.
    Hence, leveraging \Cref{lemma:Covariance_maximum_eigenvalue}, as well as the non-decreasing property of $\|\cdot\|_\op$ for $\preceq$, it results that,
    \begin{align}
        \|\E\left[\Phi_1 \Brm(\bfD) \Phi_1^\top\right]\|_\op \leq \|\Sigma_1 + \Sigma_2 + \Sigma_{1,2} + \Sigma_{1,2}^\top\|_\op \leq \kappa_n^2 \eqsp,
    \end{align}
    and,
    \begin{align}
        \left\|\E\left[\Resolvent_-(\bfD) \Phi_1 \hat{\Brm}(\bfD) \Phi_1^\top \Resolvent_-(\bfD)\right]\right\|_\op \leq \dfrac{\kappa_n^2}{\uplambda_{\min}^\Re(\bfD)^2} \eqsp.
    \end{align}
    Using these bounds, we can control the norm of the first term in \eqref{eq:Resolvent_decomposition_leavetwoout}, indeed,
    \begin{align}
        \dfrac{1}{n} \left\|\E\left[\Resolvent_-(\bfD) \Phi_1 \hat{\Brm}(\bfD) \Phi_1^\top \Resolvent_-(\bfD)\right]\E\left[\Phi_1 \Brm(\bfD) \Phi_1^\top\right] \detequiv(\bfD)\right\|_\Frob
        &\leq \dfrac{\sqrt{p} \kappa_n^4}{n \uplambda_{\min}^\Re(\bfD)^3} \eqsp.
        \label{eq:Resolvent_decomposition_leavetwoout_term_1}
    \end{align}
    We now focus on the second term in \eqref{eq:Resolvent_decomposition_leavetwoout}, to this end, write,
    \begin{align}
        \left\|\E\left[\Resolvent_-(\bfD) \Phi_1 \left\{ \Brm(\bfD)  -  \hat{\Brm}(\bfD) \right\}\Phi_1^\top\right] \detequiv(\bfD)\right\|_\Frob
        &= \sup_{\|\mathbf{A}\|_\Frob = 1} \E\left[ \tr\left( \mathbf{A} \Resolvent_-(\bfD) \Phi_1 \left\{ \Brm(\bfD)  -  \hat{\Brm}(\bfD) \right\}\Phi_1^\top\detequiv(\bfD)\right)\right] \\
        &= \sup_{\|\mathbf{A}\|_\Frob = 1} \E\left[\tr\left(  \Phi_1^\top\detequiv(\bfD)\mathbf{A}\Resolvent_-(\bfD) \Phi_1 \left\{ \Brm(\bfD)  -  \hat{\Brm}(\bfD) \right\} \right) \right]
        \label{eq:Annoying_term_to_bound}
    \end{align}
    and introduce,
    \begin{align}
        \tilde{\Brm}(\bfD) = \left(\Id{2} + \dfrac{1}{n} \E\left[\Phi_1 \Resolvent_-(\bfD) \Phi_1^\top\right]\right)^{-1} \eqsp.
    \end{align}
    Then, for any $\mathbf{A}$ such that $\|\mathbf{A}\|_\Frob = 1$, standard calculations yield,
    \begin{align}
        \E\left[\tr\left(  \Phi_1^\top\detequiv(\bfD)\mathbf{A}\Resolvent_-(\bfD) \Phi_1\left\{ \Brm(\bfD)  -  \hat{\Brm}(\bfD) \right\} \right) \right]
        =& \eqsp \tr\left(\E\left[\Phi_1^\top \detequiv(\bfD)\mathbf{A}\Resolvent_-(\bfD) \Phi_1\right]\left\{ \Brm(\bfD)  -  \tilde{\Brm}(\bfD) \right\}\right) \\
        &+ \E\left[\tr\left(\Phi_1^\top\detequiv(\bfD)\mathbf{A}\Resolvent_-(\bfD) \Phi_1 \left\{\tilde{\Brm}(\bfD) - \hat{\Brm}(\bfD)\right\}\right)\right]
        \label{eq:Annoying_term_with_A_to_bound}
    \end{align}
    we control the two terms in \eqref{eq:Annoying_term_with_A_to_bound} , first using the Cauchy-Schwarz inequality,
    \begin{align} \label{eq:Annoying_term_with_A_to_bound_term_1}
        \tr\left(\E\left[\Phi_1^\top\detequiv(\bfD)\mathbf{A}\Resolvent_-(\bfD) \Phi_1\right]\left\{ \Brm(\bfD)  -  \tilde{\Brm}(\bfD) \right\}\right) \leq \|\E\left[\Phi_1^\top\detequiv(\bfD)\mathbf{A}\Resolvent_-(\bfD) \Phi_1\right]\|_\Frob \|\Brm(\bfD)  -  \tilde{\Brm}(\bfD)\|_\Frob
    \end{align}
    furthermore, we have,
    \begin{align}
        \|\E\left[\Phi_1^\top\detequiv(\bfD)\mathbf{A}\Resolvent_-(\bfD) \Phi_1\right] \|_\Frob
        &= \left\|\begin{pmatrix}
            \tr\left(\Sigma_1 \detequiv(\bfD) \mathbf{A} \E\left[\Resolvent_-(\bfD)\right]\right) & \tr\left(\Sigma_{1,2} \detequiv(\bfD) \mathbf{A} \E\left[\Resolvent_-(\bfD)\right]\right) \\
            \tr\left(\Sigma_{1,2} \detequiv(\bfD) \mathbf{A} \E\left[\Resolvent_-(\bfD)\right]\right) & \tr\left(\Sigma_{2} \detequiv(\bfD) \mathbf{A} \E\left[\Resolvent_-(\bfD)\right]\right)
        \end{pmatrix}\right\|_\Frob \\
        &\leq \dfrac{2\sqrt{p}\kappa_n^2}{\uplambda_{\min}^\Re(\bfD)^2} \eqsp.
        \label{eq:Expectation_of_quadform_bound}
    \end{align}
    and, using the $1$-Lipschitz property of $\mathbf{M} \mapsto (\Id{2} + \mathbf{M})^{-1}$ on the cone of positive semi-definite matrices,
    as well as \Cref{eq:Woodbury_Resolvent_identity}, we have that,
    \begin{align}
        \|\Brm(\bfD) - \tilde{\Brm}(\bfD)\|_\Frob
        &\leq \dfrac{1}{n} \left\|\E\left[\Phi_1^\top \E\left[\Resolvent(\bfD) - \Resolvent_-(\bfD)\right]\Phi_1\right]\right\|_\Frob
        \leq \dfrac{1}{n^2} \left\|\E\left[\Phi_1^\top \E\left[\Resolvent_-(\bfD) \Phi_1 \hat{\Brm}(\bfD) \Phi_1^\top \Resolvent_-(\bfD)\right] \Phi_1\right]\right\|_\Frob \\
        &\leq \dfrac{1}{n^2} \left\|\E\left[\Phi_1^\top \Resolvent_-(\bfD) \E\left[\Phi_1 \Phi_1^\top\right] \Resolvent_-(\bfD) \Phi_1\right]\right\|_\Frob
        \leq \dfrac{\kappa_n^4 \sqrt{p}}{\uplambda_{\min}^\Re(\bfD)^2 n^2} \eqsp.
        \label{eq:DilationMatrixBias_Bound}
    \end{align}
    Plugging \eqref{eq:Expectation_of_quadform_bound} and \eqref{eq:DilationMatrixBias_Bound} in \eqref{eq:Annoying_term_with_A_to_bound_term_1}, we get,
    \begin{align} \label{eq:First_term_final}
        \tr\left(\E\left[\Phi_1^\top\detequiv(\bfD)\mathbf{A}\Resolvent_-(\bfD) \Phi_1\right]\left\{ \Brm(\bfD)  -  \tilde{\Brm}(\bfD) \right\}\right)
        \leq \dfrac{2 \kappa_n^6 p}{\uplambda_{\min}^\Re(\bfD)^4 n^2} \eqsp.
    \end{align}
    Moving on to the second term in the right hand side of \eqref{eq:Annoying_term_with_A_to_bound}, we introduce
    \begin{align}
        \chi(\bfD) = \Phi_1^\top \detequiv(\bfD) \Abf \Resolvent_-(\bfD) \Phi_1 \eqsp,
    \end{align}
    so that, using \eqref{eq:resolvent_identity} yields
    \begin{align} \label{eq:Annoying_term_with_A_to_bound_term_first_decomposition}
        &\E\left[\tr\left(\Phi_1^\top\detequiv(\bfD)\mathbf{A}\Resolvent_-(\bfD) \Phi_1 \left\{\tilde{\Brm}(\bfD) - \hat{\Brm}(\bfD)\right\}\right)\right] \\
        =&\eqsp \dfrac{1}{n}\E\left[\tr\left(\hat{\Brm}(\bfD)\chi(\bfD) \tilde{\Brm}(\bfD) \left\{\Phi_1^\top \Resolvent_-(\bfD) \Phi_1  - \E\left[\Phi_1^\top \Resolvent_-(\bfD) \Phi_1\right]\right\}\right)\right] \\
        \leq& \dfrac{1}{n} \sqrt{\Var(\hat{\Brm}(\bfD)\chi(\bfD)) \Var(\Phi_1\Resolvent_-(\bfD) \Phi_1^\top)} \eqsp.
    \end{align}
    Note that the last line results from the Cauchy-Schwarz inequality, as well as $\tilde{\Brm}(\bfD)$ being a deterministic matrix with $\|\tilde{\Brm}(\bfD)\|_\op \leq 1$,
    the $\Var$ operator being defined for random matrices in \eqref{eq:matrix_variance_definition}.
    We bound the variances of each entry of $\Phi_1\detequiv(\bfD)\mathbf{A}\Resolvent_-(\bfD) \Phi_1^\top$ thanks to \Cref{corollary:randomHansonWright} which gives,
    \begin{align} \label{eq:Variance_of_quadform_bound}
        \Var(\Phi_1\Resolvent_-(\bfD) \Phi_1^\top)
        &\leq \const \left(\dfrac{\kappa_n^4 p}{\uplambda_{\min}^\Re(\bfD)^2} + \dfrac{p \kappa_n^6}{n \uplambda_{\min}^\Re(\bfD)^3} \right)
    \end{align}
    and, applying \Cref{lemma:matrix_variance_product_bound} to the variance of $\hat{\Brm}(\bfD)\chi(\bfD)$, we get,
    \begin{align} \label{eq:Variance_of_hatBTheta_bound_1}
        \Var(\hat{\Brm}(\bfD)\chi(\bfD)) \leq 2 \Var(\chi(\bfD)) + 2 \|\E\left[\chi(\bfD)\right]\|_\op^2 \Var(\hat{\Brm}(\bfD)) \eqsp.
    \end{align}
    additionally the $1$-Lipschitz property of $\mathbf{M} \mapsto (\Id{2} + \mathbf{M})^{-1}$ on the space of positive matrices, implies
    \begin{align} \label{eq:Variance_of_hatBTheta_bound_2}
        \Var(\hat{\Brm}(\bfD)) \leq \dfrac{1}{n^2} \Var(\Phi_1^\top \Resolvent_-(\bfD) \Phi_1) \leq \dfrac{\const}{n^2} \left(\dfrac{\kappa_n^4p}{\uplambda_{\min}^\Re(\bfD)^2} + \dfrac{p \kappa_n^6}{n \uplambda_{\min}^\Re(\bfD)^3} \right) \eqsp.
    \end{align}
    Finally, using \Cref{lemma:Covariance_maximum_eigenvalue} and \Cref{corollary:randomHansonWright}, we have that,
    \begin{align} \label{eq:Variance_of_hatBTheta_bound_3}
        \|\E\left[\chi(\bfD)\right]\|_\op \leq \dfrac{1}{\uplambda_{\min}^\Re(\bfD)^2} \left(\|\Sigma_1\|_\op + \|\Sigma_2\|_\op + \|\Sigma_{1,2}\|_\op + \|\Sigma_{1,2}^\top\|_\op \right) \leq \dfrac{\kappa_n^2}{\uplambda_{\min}^\Re(\bfD)^2} \eqsp,
    \end{align}
    and,
    \begin{align} \label{eq:Variance_of_hatBTheta_bound_4}
        \Var(\chi(\bfD)) \leq \const \left(\dfrac{\kappa_n^4}{\uplambda_{\min}^\Re(\bfD)^4} + \dfrac{p \kappa_n^6}{n \uplambda_{\min}^\Re(\bfD)^5}\right)\eqsp.
    \end{align}
    To conclude, we have by merging \eqref{eq:Variance_of_quadform_bound}, \eqref{eq:Variance_of_hatBTheta_bound_1}, \eqref{eq:Variance_of_hatBTheta_bound_2}, \eqref{eq:Variance_of_hatBTheta_bound_3} and \eqref{eq:Variance_of_hatBTheta_bound_4}, we get,
    \begin{align} \label{eq:Variance_of_hatBTheta_bound_final}
        \Var(\hat{\Brm}(\bfD)\chi(\bfD))
        \leq \const \left(\dfrac{\kappa_n^4}{\uplambda_{\min}^\Re(\bfD)^4} + \dfrac{\kappa_n^6 p}{\uplambda_{\min}^\Re(\bfD)^5 n} + \dfrac{\kappa_n^8 p}{\uplambda_{\min}^\Re(\bfD)^6 n^2} + \dfrac{\kappa_n^{10} p}{\uplambda_{\min}^\Re(\bfD)^7 n^3} \right) \eqsp.
    \end{align}
    And, plugging \eqref{eq:Variance_of_quadform_bound} and \eqref{eq:Variance_of_hatBTheta_bound_final} in \eqref{eq:Annoying_term_with_A_to_bound_term_first_decomposition}, we get for a large enough constant $\const$,
    \begin{align} \label{eq:Second_term_final}
        &\E\left[\tr\left(\Phi_1^\top\detequiv(\bfD)\mathbf{A}\Resolvent_-(\bfD) \Phi_1 \left\{\tilde{\Brm}(\bfD) - \hat{\Brm}(\bfD)\right\}\right)\right] \\
        \leq&\eqsp \dfrac{\const}{n} \left(\dfrac{\kappa_n^2 \sqrt{p}}{\uplambda_{\min}^\Re(\bfD)} + \dfrac{\sqrt{p} \kappa_n^3}{\sqrt{n} \uplambda_{\min}^\Re(\bfD)^{3/2}} \right)
        \left(\dfrac{\kappa_n^2}{\uplambda_{\min}^\Re(\bfD)^2} + \dfrac{\kappa_n^3 \sqrt{p}}{\uplambda_{\min}^\Re(\bfD)^{5/2} \sqrt{n}} + \dfrac{\kappa_n^4\sqrt{p}}{\uplambda_{\min}^\Re(\bfD)^2 n} + \dfrac{\kappa_n^5 \sqrt{p}}{\uplambda_{\min}^\Re(\bfD)^{5/2} n^{3/2}} \right) \\
        =&\eqsp \dfrac{\const\sqrt{p}}{n} \Bigg(
            \dfrac{\kappa_n^4 }{\uplambda_{\min}^\Re(\bfD)^3}
            + \dfrac{2\kappa_n^5 \sqrt{p}}{\uplambda_{\min}^\Re(\bfD)^{7/2} \sqrt{n}}
            + \dfrac{\kappa_n^6 \sqrt{p}}{\uplambda_{\min}^\Re(\bfD)^{3} n}
            + \dfrac{2\kappa_n^7 \sqrt{p}}{\uplambda_{\min}^\Re(\bfD)^{7/2} n^{3/2}}
            + \dfrac{\kappa_n^6 \sqrt{p}}{\uplambda_{\min}^\Re(\bfD)^{4} n}
            + \dfrac{\kappa_n^8 \sqrt{p}}{\uplambda_{\min}^\Re(\bfD)^{4} n^{2}}
        \Bigg)\eqsp.
    \end{align}
    By plugging \eqref{eq:First_term_final}, \eqref{eq:Second_term_final} and \eqref{eq:Annoying_term_with_A_to_bound} in \eqref{eq:Annoying_term_to_bound}, we thus get,
    \begin{align} \label{eq:Resolvent_decomposition_leavetwoout_term_2}
        &\left\|\E\left[\Resolvent_-(\bfD) \Phi_1 \left\{ \Brm(\bfD)  -  \hat{\Brm}(\bfD) \right\} \Phi_1^\top \right] \detequiv(\bfD)\right\|_\Frob \\
        &\leq \dfrac{\const\sqrt{p}}{n} \Bigg(
            \dfrac{\kappa_n^4 }{\uplambda_{\min}^\Re(\bfD)^3}
            + \dfrac{\kappa_n^5 \sqrt{p}}{\uplambda_{\min}^\Re(\bfD)^{7/2} \sqrt{n}}
            + \dfrac{\kappa_n^6 \sqrt{p}}{\uplambda_{\min}^\Re(\bfD)^{3} n}
            + \dfrac{\kappa_n^7 \sqrt{p}}{\uplambda_{\min}^\Re(\bfD)^{7/2} n^{3/2}} \\
           &\hspace{1.5cm} + \dfrac{\kappa_n^6 \sqrt{p}}{\uplambda_{\min}^\Re(\bfD)^{4} n}
            + \dfrac{\kappa_n^8 \sqrt{p}}{\uplambda_{\min}^\Re(\bfD)^{4} n^{2}}
        \Bigg) \eqsp.
    \end{align}
    The claim follows from \eqref{eq:Resolvent_decomposition_leavetwoout}, \eqref{eq:Resolvent_decomposition_leavetwoout_term_1} and\eqref{eq:Resolvent_decomposition_leavetwoout_term_2}, as well as merely reorganizing each term.
\end{proof}

We derive similar results for some variants of $\Resolvent(\bfD)$, which directly relate to \Cref{proposition:det_equivs}. 
Indeed, the following \Cref{prop:DetEquiv_Q_noaug}can be readily applied to approximate $\theta_\star^\top \Sigma_\star \E\left[\hat{\theta}_{\alpha, \lambda_n}\right]$, 
we show:
\begin{proposition} \label{prop:DetEquiv_Q_noaug}
    Let $\Psi \in \R^{(2p + 2) \times n}$ be a random matrix with centered, i.i.d., and $\kappa_n$-Lipschitz concentrated columns.
    Let $\bfD \in \C^{p \times p}$ be a deterministic matrix with $\uplambda_{\min}^\Re(\bfD) > 0$.
    Denote $\Phi^{(1)} = \Psi_{1:p, :} \in \R^{p \times n}$, $\Phi^{(2)} = \Psi_{p+1:2p, :}\in \R^{p \times n}$, and $\Phi^{(3)} = \Psi_{2p+1, :} \in \R^{1 \times n}$, and $\Phi^{(4)} = \Psi_{2p+2, :} \in \R^{1 \times n}$
    as well as for all $(i,j) \in \{(1,2), (3,4)\}$, $\Phi^{(i,j)} = (\Phi^{(i)}, \Phi^{(j)})$ and $\Phi_1^{(i,j)} = (\Phi_1^{(i)}, \Phi_1^{(j)})$, we consider
    \begin{align}
        \Resolvent(\bfD) = \left(n^{-1} \Phi^{(1,2)} \Phi^{(1,2)\top} + \bfD\right)^{-1}\eqsp, \quad \Qfun(\bfD) =  \Resolvent(\bfD) n^{-1}\Phi^{(1,2)} \Phi^{(3, 4)\top} \eqsp.
    \end{align}
    Write,
    \begin{align}
        \detequiv(\bfD) = \left(\E\left[\Phi_1^{(1,2)} \Brm(\bfD) \Phi_1^{(1,2)\top}\right]+ \bfD \right)^{-1} \eqsp,
        \quad \text{where} \quad
        \Brm(\bfD) = \left(\Id{2} + n^{-1}\E\left[\Phi_1^{(1,2)\top} \E\left[\Resolvent(\bfD)\right]\Phi_1^{(1,2)}\right]\right)^{-1} \eqsp,
    \end{align}
    define,
    \begin{align}
        \Qequiv(\bfD) = \detequiv(\bfD) \E\left[\Phi_1^{(1,2)} \Brm(\bfD) \Phi_1^{(3,4)\top}\right] \eqsp.
    \end{align}
    Then, it holds for a large enough universal constant $\const$,
    \begin{align}
        \|\E\left[\Qfun(\bfD)\right] - \Qequiv(\bfD)\|_2
        \leq& \eqsp \kappa_n^2\left\|\E\left[\Resolvent(\bfD)\right] - \detequiv(\bfD)\right\|_\Frob \\
        &+ \dfrac{\const \sqrt{p}(1 + \kappa_n^{13})}{\uplambda_{\min}^\Re(\bfD)^4 n} \biggl\{\uplambda_{\min}^\Re(\bfD) + \uplambda_{\min}^\Re(\bfD)^2 + \dfrac{p}{n} + \dfrac{p}{\uplambda_{\min}^\Re(\bfD) n^2} + \dfrac{p}{\uplambda_{\min}^\Re(\bfD)^2 n^3} + \dfrac{\uplambda_{\min}^\Re(\bfD)^{1/2}}{\sqrt{n}} \\
        &\hspace{2.6cm} + \dfrac{p}{\uplambda_{\min}^\Re(\bfD)^{1/2} n^{3/2}} + \dfrac{p}{\uplambda_{\min}^\Re(\bfD)^{3/2} n^{5/2}} + \dfrac{p}{\uplambda_{\min}^\Re(\bfD)^{5/2} n^{7/2}}\biggr\} \eqsp,
    \end{align}
\end{proposition}
\begin{proof}
    We first introduce the following matrices,
    \begin{align}
        \Resolvent_-(\bfD) = \left(n^{-1} \Phi^{(1,2)}\Phi^{(1,2)\top} - n^{-1} \Phi^{(1,2)}_1\Phi^{(1,2)\top}_1 + \bfD\right)^{-1}
    \end{align}
    for all $i \in \{1,2\}$,
    \begin{align}
        \SampleCov^{(i)} = \dfrac{1}{n}\Phi^{(i) \top} \Phi^{(i)}\eqsp,
        \quad
        \SampleCov_-^{(i)} = \SampleCov^{(i)} - \dfrac{1}{n} \Phi_1^{(i)} \Phi_1^{(i)\top} \eqsp,
        \quad \text{and} \quad
        \Sigma_i = \E\left[\Phi_1^{(i)}\Phi_1^{(i)\top}\right]
    \end{align}
    and for all $(i,j) \in \{(1,2),(3,4)\}$,
    \begin{align}
        \SampleCov^{(i,j)} = \dfrac{1}{n}\Phi^{(i) \top} \Phi^{(j)} \eqsp,
        \quad
        \SampleCov_-^{(i,j)} = \SampleCov^{(i,j)} - \dfrac{1}{n} \Phi_1^{(i)} \Phi_1^{(j)\top} \eqsp,
        \quad \text{ and } \quad
        \Sigma_{i,j} = \E\left[\Phi_1^{(i)}\Phi_1^{(j)\top}\right]
    \end{align}
    as well as the dilation matrix,
    \begin{align}
        \hat{\Brm}(\bfD) = \left(\Id{2} + n^{-1}\Phi_1^{(1,2)\top} \Resolvent_-(\bfD)\Phi_1^{(1,2)}\right)^{-1}\eqsp.
    \end{align}
    Thanks to calculations similar to the ones presented in the proof of \Cref{prop:Detequiv_R_twobytwo}, we write,
    \begin{align}
        \E\left[\Qfun(\bfD)\right] &= \E\left[\Resolvent(\bfD) \Phi_1^{(1, 2)} \Phi_1^{(3,4)\top}\right]
        = \E\left[\Resolvent_-(\bfD) \Phi_1^{(1, 2)} \hat{\Brm}(\bfD) \Phi_1^{(3,4)\top}\right] \\
        &= \E\left[\Resolvent_-(\bfD) \Phi_1^{(1, 2)} \left\{\hat{\Brm}(\bfD) - \Brm(\bfD)\right\} \Phi_1^{(3,4)\top}\right] + \E\left[\Resolvent_-(\bfD)\right] \E\left[\Phi_1^{(1, 2)} \Brm(\bfD)\Phi_1^{(3,4)\top}\right]\eqsp,
    \end{align}
    hence,
    \begin{align} \label{eq:Q_decomposition}
        \E\left[\Qfun(\bfD)\right] - \Qequiv(\bfD) =&\eqsp \E\left[\Resolvent_-(\bfD) \Phi_1^{(1, 2)} \left\{\hat{\Brm}(\bfD) - \Brm(\bfD)\right\} \Phi_1^{(3,4)\top}\right] \\
        &+  \left\{\E\left[\Resolvent_-(\bfD)\right] - \E\left[\Resolvent(\bfD)\right]\right\} \E\left[\Phi_1^{(1, 2)} \Brm(\bfD)\Phi_1^{(3,4)\top}\right] \\
        &+ \left\{\E\left[\Resolvent(\bfD)\right] - \detequiv(\bfD)\right\} \E\left[\Phi_1^{(1, 2)} \Brm(\bfD)\Phi_1^{(3,4)\top}\right]\eqsp.
    \end{align}
    Recalling that $\Brm(\bfD) \preceq \Id{2}$, and using \Cref{lemma:Covariance_maximum_eigenvalue},we have,
    \begin{align}
        \left\|\E\left[\Phi_1^{(1,2)} \Brm(\bfD) \Phi_1^{(3,4)\top}\right]\right\|_\op \leq \left\|\E\left[\Phi_1^{(1,2)} \Phi_1^{(3,4)\top}\right]\right\|_\op \leq \kappa_n^2\eqsp,
    \end{align}
    Additionally, we have from the resolvent identity, and \Cref{lemma:Covariance_maximum_eigenvalue} again,
    \begin{align}
        \left\|\E\left[\Resolvent_-(\bfD)\right] - \E\left[\Resolvent(\bfD)\right]\right\|_\Frob
        &= \dfrac{1}{n} \left\|\E\left[\Resolvent_-(\bfD) \Phi_1^{(1,2)} \Phi_1^{(1,2)\top} \Resolvent_-(\bfD)\right]\right\|_\Frob \\
        &\leq \dfrac{\kappa_n^2 \sqrt{p}}{\uplambda_{\min}^\Re(\bfD)^2 n} \eqsp,
    \end{align}
    thus,
    \begin{align} \label{eq:BoundTerm2_decompositionQ}
        \|\left\{\E\left[\Resolvent_-(\bfD)\right] - \E\left[\Resolvent(\bfD)\right]\right\} \E\left[\Phi_1 \Brm(\bfD)\Phi_1^{(3)}\right]\|_\Frob \leq \dfrac{\kappa_n^4 \sqrt{p}}{\uplambda_{\min}^\Re(\bfD)^2 n} \eqsp,
    \end{align}
    and,
    \begin{align} \label{eq:Q_norm_decomposition}
        \left\|\E\left[\Qfun(\bfD)\right] - \Qequiv(\bfD)\right\|_2 \leq& \eqsp \left\|\E\left[\Resolvent_-(\bfD) \Phi_1^{(1, 2)} \left\{\hat{\Brm}(\bfD) - \Brm(\bfD)\right\} \Phi_1^{(3,4)\top}\right]\right\|_2 \\
        & + \dfrac{\kappa_n^4 \sqrt{p}}{\uplambda_{\min}^\Re(\bfD)^2 n} + \kappa_n^2\left\|\E\left[\Resolvent(\bfD)\right] - \detequiv(\bfD)\right\|_\Frob
    \end{align}
    It remains to control the first term in the right-hand side of \eqref{eq:Q_norm_decomposition}, to this end, we introduce the notation $\widetilde{\Brm}(\bfD)$ for,
    \begin{align}
        \widetilde{\Brm}(\bfD) = \left(\Id{2} + n^{-1}\E\left[\Phi_1^{(1,2)\top} \Resolvent_-(\bfD) \Phi_1^{(1,2)} \right]\right)^{-1}\eqsp,
    \end{align}
    then we have,
    \begin{align}
        \left\|\E\left[\Resolvent_-(\bfD) \Phi_1^{(1,2)} \left\{\hat{\Brm}(\bfD) - \Brm(\bfD)\right\} \Phi_1^{(3,4)\top}\right]\right\|_2
        \leq&\eqsp \left\|\E\left[\Resolvent_-(\bfD) \Phi_1^{(1,2)} \left\{\hat{\Brm}(\bfD) - \widetilde{\Brm}(\bfD)\right\} \Phi_1^{(3,4)\top}\right]\right\|_2 \\
        &+ \dfrac{\kappa_n^2}{\uplambda_{\min}^{\Re}(\bfD)^2} \|\Brm(\bfD) - \widetilde{\Brm}(\bfD)\|_\op \eqsp.
    \end{align}
    Proceeding similarly to \eqref{eq:DilationMatrixBias_Bound}, we have,
    \begin{align}\label{eq:DilationMatrixBias_Bound_2}
        \|\Brm(\bfD) - \widetilde{\Brm}(\bfD)\|_\op \leq \dfrac{\kappa_n^4 \sqrt{p}}{\uplambda_{\min}^\Re(\bfD)^2 n^2} \eqsp,
    \end{align}
    and from the resolvent identity,
    \begin{align}
        &\left\|\E\left[\Resolvent_-(\bfD) \Phi_1^{(1,2)} \left\{\hat{\Brm}(\bfD) - \widetilde{\Brm}(\bfD)\right\} \Phi_1^{(3,4)\top}\right]\right\|_2 \\
        =&\eqsp \dfrac{1}{n} \left\|\E\left[\Resolvent_-(\bfD) \Phi_1^{(1,2)} \hat{\Brm}(\bfD)\left\{\Phi_1^{(1,2)^\top} \Resolvent_-(\bfD) \Phi_1^{(1,2)}  - \E\left[\Phi_1^{(1,2)^\top} \Resolvent_-(\bfD) \Phi_1^{(1,2)}\right]\right\} \widetilde{\Brm}(\bfD) \Phi_1^{(3,4)\top}\right]\right\|_2 \eqsp.
    \end{align}
    Thus, using $\|\mathbf{x}\|_2 = \sup_{\|\mathbf{a}\|_2 = 1} \mathbf{a}^\top \mathbf{x}$ and the Cauchy-Schwarz inequality, we have,
    \begin{align}
        &\left\|\E\left[\Resolvent_-(\bfD) \Phi_1^{(1,2)} \left\{\hat{\Brm}(\bfD) - \widetilde{\Brm}(\bfD)\right\} \Phi_1^{(3,4)\top}\right]\right\|_2 \\
        \leq& \eqsp \sup_{\|\mathbf{a}\|_2 = 1}\dfrac{1}{n} \sqrt{\Var\left(\widetilde{\Brm}(\bfD) \Phi_1^{(3,4)\top} \mathbf{a}^\top \Resolvent_-(\bfD) \Phi_1^{(1,2)} \hat{\Brm}(\bfD)\right) \Var\left(\Phi_1^{(1,2)^\top} \Resolvent_-(\bfD) \Phi_1^{(1,2)}\right)} \\
        \leq& \dfrac{\const \sqrt{p}}{n}\left(\dfrac{\kappa_n^2 }{\uplambda_{\min}^\Re(\bfD)} + \dfrac{\kappa_n^3}{\uplambda_{\min}^\Re(\bfD)^{3/2} \sqrt{n}}\right) \sup_{\|\mathbf{a}\|_2 = 1} \sqrt{\Var\left(\Phi_1^{(3,4)\top} \mathbf{a}^\top \Resolvent_-(\bfD) \Phi_1^{(1,2)} \hat{\Brm}(\bfD)\right)} \eqsp.
        \label{eq:BoundTerm1_decompositionQ}
    \end{align}
    The last line follows from the same lines as \eqref{eq:Variance_of_quadform_bound},
    and the fact that $\widetilde{\Brm}(\bfD)$ is deterministic and $\widetilde{\Brm}(\bfD) \preceq \Id{2}$. Now, fix $\mathbf{a}$ such that $\|\mathbf{a}\|_2 = 1$, we write for simplicity,
    \begin{align}
        \chi(\bfD) = \Phi_1^{(3,4)\top} \mathbf{a}^\top \Resolvent_-(\bfD) \Phi_1^{(1,2)} \eqsp,
    \end{align}
    then we have from \Cref{lemma:matrix_variance_product_bound},
    \begin{align} \label{eq:VarThetaHatBrm_intermediate}
        \Var\left(\chi(\bfD) \hat{\Brm}(\bfD)\right) \leq \Var\left(\chi(\bfD)\right) + \dfrac{\kappa_n^4}{\uplambda_{\min}^\Re(\bfD)^2} \Var\left(\hat{\Brm}(\bfD)\right)\eqsp,
    \end{align}
    then, similarly to \eqref{eq:Variance_of_hatBTheta_bound_2}, we have,
    \begin{align}
        \Var\left(\hat{\Brm}(\bfD)\right) \leq \dfrac{\const}{n^2} \left(\dfrac{\kappa_n^4p}{\uplambda_{\min}^\Re(\bfD)^2} + \dfrac{p \kappa_n^6}{n \uplambda_{\min}^\Re(\bfD)^3} \right) \eqsp,
    \end{align}
    and using \Cref{corollary:randomHansonWright} on all the entries of $\chi(\bfD)$, then over the entries, it results,
    \begin{align}
        \Var\left(\chi(\bfD)\right) \leq \const \left(\dfrac{\kappa_n^4}{\uplambda_{\min}^\Re(\bfD)^2} + \dfrac{p \kappa_n^6}{n \uplambda_{\min}^\Re(\bfD)^3}\right) \eqsp,
    \end{align}
    and finally,
    \begin{align} \label{eq:VarThetaHat_final}
        \Var\left(\chi(\bfD) \hat{\Brm}(\bfD)\right)
        \leq \const \left(\dfrac{\kappa_n^4}{\uplambda_{\min}^\Re(\bfD)^2} + \dfrac{p \kappa_n^6}{n \uplambda_{\min}^\Re(\bfD)^3}\right) + \dfrac{\const p}{n^2} \left(\dfrac{\kappa_n^8}{\uplambda_{\min}^\Re(\bfD)^4} + \dfrac{ \kappa_n^{10}}{n \uplambda_{\min}^\Re(\bfD)^5} \right) \eqsp.
    \end{align}
    Plugging \eqref{eq:VarThetaHat_final} in \eqref{eq:BoundTerm1_decompositionQ}, we get,
    \begin{align} \label{eq:BoundTerm1_decompositionQ_final}
        &\left\|\E\left[\Resolvent_-(\bfD) \Phi_1^{(1,2)} \left\{\hat{\Brm}(\bfD) - \widetilde{\Brm}(\bfD)\right\} \Phi_1^{(3,4)\top}\right]\right\|_\Frob \\
        &\leq \dfrac{\const \sqrt{p}}{n} \biggl( \dfrac{\kappa_n^6}{\uplambda_{\min}^{\Re}(\bfD)^3} + \dfrac{\kappa_n^8 p}{\uplambda_{\min}^{\Re}(\bfD)^4 n} + \dfrac{\kappa_n^{10} p}{\uplambda_{\min}^{\Re}(\bfD)^5 n^2} + \dfrac{\kappa_n^{12} p}{\uplambda_{\min}^{\Re}(\bfD)^6 n^3} + \dfrac{\kappa_n^7}{\uplambda_{\min}^{\Re}(\bfD)^{7/2} \sqrt{n}} \\
        &\hspace{1.4cm} + \dfrac{\kappa_n^9 p}{\uplambda_{\min}^{\Re}(\bfD)^{9/2} n^{3/2}} + \dfrac{\kappa_n^{11} p}{\uplambda_{\min}^{\Re}(\bfD)^{11/2} n^{5/2}} + \dfrac{\kappa_n^{13} p}{\uplambda_{\min}^{\Re}(\bfD)^{13/2} n^{7/2}}\biggr) \eqsp.
    \end{align}
    The result follows from plugging \eqref{eq:BoundTerm1_decompositionQ_final} in \eqref{eq:Q_norm_decomposition}, as well as merely reorganizing each term.
\end{proof}

Using \Cref{prop:DetEquiv_Q_noaug} with $\Psi = (\Psi_i)_{i=1}^n$, s.t. for all $i \leq n$, $\Psi_i = (\sqrt{1-\alpha}\varphi(\Xtt_i), \sqrt{\alpha}\mu_x(\Ztt_i), \sqrt{1-\alpha}\Ytt_i,  \sqrt{\alpha}\mu_y(\Ztt_i))$ and $\bfD = \alpha \E\left[\Lambda(\Ztt)\right] + \lambda_n \Id{p}$,
one can approximate $\E\left[\hat{\theta}_{\alpha,\lambda}\right]$. This will be formalized later in \Cref{sec:proofs_main_results}. In the meantime, we prove a comparable result that allows to approximate $\E\left[\xi_{\alpha, \lambda}(\zeta)\right]$ defined in \eqref{eq:def_xi_alpha_lambda_zeta},
for any $\zeta \in \msu_\lambda$: 

\begin{proposition} \label{prop:DetEquiv_S_noaug}
    Let $\Psi \in \R^{(2p + 2)\times n}$ be a random matrix with centered i.i.d., and $\kappa_n$ Lipschitz concentrated columns.
    Let $\bfD \in \C^{p\times p}$ be a deterministic matrix with $\uplambda_{\min}^\Re(\bfD) > 0$.
    Denote $\Phi^{(1)} = \Phi_{1:p, :} \in \R^{p \times n}$, $\Phi^{(2)} = \Phi_{p+1:2p, :} \in \R^{p\times n}$, $\Phi^{(3)} = \Psi_{2p+1, :} \in \R^{1 \times n}$ and $\Phi^{(4)} = \Psi_{2p+2, :} \in \R^{1 \times n}$,
    as well as for all $(i,j) \in \{(1,2), (3,4)\}$, $\Phi^{(i,j)} = (\Phi^{(i)}, \Phi^{(j)})$ and for $\Phi_1^{(i,j)} = (\Phi_1^{(i)}, \Phi_1^{(j)})$, we consider
    \begin{align}
        \Resolvent(\bfD) = \left(n^{-1} \Phi^{(1,2)}\Phi^{(1,2)\top} + \bfD\right)^{-1}\eqsp, \quad \Sfun(\bfD) = n^{-1} \Phi^{(3,4)} \Phi^{(1,2)\top} \Resolvent(\bfD) n^{-1} \Phi^{(1,2)} \Phi^{(3,4)\top} \eqsp .
    \end{align}
    Define,
    \begin{align}
        \detequiv(\bfD) = \left(\E\left[\Phi_1^{(1,2)} \Brm(\bfD) \Phi_1^{(1,2) \top}\right] + \bfD\right)^{-1} \eqsp, \quad \text{ where} \quad \Brm(\bfD) = \left(\Id{2} + n^{-1} \E\left[\Phi_1^{(1,2)^\top }\E\left[\Resolvent(\bfD)\right] \Phi_1^{(1,2)}\right]\right)^{-1} \eqsp,
    \end{align}
    and,
    \begin{align}
        \Sequiv(\bfD) = \E\left[\Phi_1^{(3,4)} \{\Id{2} - \Brm(\bfD)\}\Phi_1^{(3,4)\top}\right] + \E\left[\Phi_1^{(3,4)} \Brm(\bfD) \Phi_1^{(1,2)\top}\right] \detequiv(\bfD) \E\left[\Phi_1^{(1,2)}\Brm(\bfD) \Phi_1^{(1,2)}\right] \eqsp.
    \end{align}
    Then, it holds for a large enough constant $\const$,
    \begin{align}
        \left|\E\left[\Sfun(\bfD)\right] - \Sequiv(\bfD)\right| 
        \leq&\eqsp 
        \const \kappa_n^2 \left\|\E\left[\Resolvent(\bfD)\right] - \detequiv(\bfD)\right\|_\Frob 
        + \dfrac{\const (1+ \kappa_n^7)\sqrt{p}}{\uplambda_{\min}^{\Re}(\bfD)^{5/2}n} \left(\uplambda_{\min}^{\Re}(\bfD)^{1/2} + \uplambda_{\min}^{\Re}(\bfD) + \dfrac{\sqrt{p}}{ \sqrt{n}}\right) \\
        &+ \dfrac{\const \kappa_n^8  p^{3/2}}{\uplambda_{\min}^\Re(\bfD)^{3}n^2}
        + \dfrac{\const (1+\kappa_n^8)\sqrt{p}}{\uplambda_{\min}^\Re(\bfD)^3 n^{3/2}} \left(\uplambda_{\min}^\Re(\bfD)^{1/2} + \dfrac{\sqrt{p}}{ \sqrt{n}}\right) 
        + \dfrac{\const\kappa_n^4}{\uplambda_{\min}^\Re(\bfD) n} \\
        &+ \dfrac{\const \kappa_n^6 p}{\uplambda_{\min}^\Re(\bfD)^2 n^2}
        + \dfrac{\const (1+\kappa_n^9) p}{\uplambda_{\min}^\Re(\bfD)^{7/2} n^{5/2}}\left(1 + \uplambda_{\min}^\Re(\bfD)^{1/2} + \uplambda_{\min}^\Re(\bfD)^{3/2}\right) \\
        &+ \dfrac{\const (1+\kappa_n^{10}) p}{\uplambda_{\min}^\Re(\bfD)^4 n^3} \left(1 + \uplambda_{\min}^\Re(\bfD)\right)
        + \dfrac{\const (1+\kappa_n^9)\sqrt{p}}{\uplambda_{\min}^\Re(\bfD)^{7/2} n^{5/2}}\left(1 + \uplambda_{\min}^\Re(\bfD)^{1/2} +  \right)     
    \end{align}
\end{proposition}

\begin{proof}
    We first introduce the following matrices,
    \begin{align}
        \Resolvent_{-}(\bfD) = \left(n^{-1} \Phi^{(1,2)}\Phi^{(1,2)\top} - n^{-1} \Phi_1^{(1,2)}\Phi_1^{(1,2)\top} + \bfD\right)^{-1} \eqsp,
    \end{align}
    \begin{align}
        \Resolvent_{--}(\bfD) = \left(n^{-1} \Phi^{(1,2)}\Phi^{(1,2)\top} - n^{-1} \Phi_1^{(1,2)}\Phi_1^{(1,2)\top} - n^{-1} \Phi_2^{(1,2)}\Phi_2^{(1,2)\top} + \bfD\right)^{-1} \eqsp,
    \end{align}
    where for all  $(i,j) \in \{1,2,3,4\}^2$, $\Phi_2^{(i,j)} = (\Phi_2^{(i)}, \Phi_2^{(j)})$ and
    \begin{align}
        \Sigma_i = \E\left[\Phi_1^{(i)}\Phi_1^{(i)\top}\right] \eqsp, \quad \Sigma_{i,j} = \E\left[\Phi_1^{(i)}\Phi_1^{(j)\top}\right]\eqsp,
    \end{align}
    as well as the matrix empirical dilation factors,
    \begin{align}
        \hat{\Brm}(\bfD) = \left(\Id{2} + n^{-1} \Phi_1^{(1,2) \top } \Resolvent_-(\bfD) \Phi_1^{(1,2)}\right)^{-1} \eqsp.
    \end{align}
    Using the fact that the rows of $\Psi$ have identical distribution, we start by writting,
    \begin{align}
        \E\left[\Sfun(\bfD)\right]
        =
        \underbrace{\dfrac{1}{n} \E\left[\Phi_1^{(3,4)} \Phi_1^{(1,2)\top} \Resolvent(\bfD) \Phi_1^{(1,2)} \Phi_1^{(3,4)\top}\right]}_{s_1}
        +
        \underbrace{\dfrac{n(n-1)}{n^2} \E\left[\Phi_1^{(3,4)} \Phi_1^{(1,2)\top} \Resolvent(\bfD) \Phi_2^{(1,2)} \Phi_2^{(3,4)\top}\right]}_{s_2}\eqsp.
    \end{align}
    In the following, we deal with the terms $s_1$ and $s_2$ separately.
    First, we focus on $s_1$ and write, thanks to \eqref{eq:LeaveTwoOutIdentity},
    \begin{align}
        s_1 &= \dfrac{1}{n}\E\left[\Phi_1^{(3,4)} \Phi_1^{(1,2)\top} \Resolvent(\bfD) \Phi_1^{(1,2)} \Phi_1^{(3,4)\top}\right] \\
        &= \E\left[\Phi_1^{(3,4)} \dfrac{1}{n}\Phi_1^{(1,2)\top} \Resolvent_-(\bfD) \Phi_1^{(1,2)} \hat{\Brm}(\bfD)\Phi_1^{(3,4)\top}\right] \\
        &= \E\left[\Phi_1^{(3,4)} \left\{\Id{2} - \hat{\Brm}(\bfD)\right\}\Phi_1^{(3,4)\top}\right] \\
        &= \underbrace{\E\left[\Phi_1^{(3,4)} \left\{\Id{2} - \Brm(\bfD)\right\}\Phi_1^{(3,4)\top}\right]}_{\overline{s}_1} + \E\left[\Phi_1^{(3,4)} \left\{\Brm(\bfD) - \hat{\Brm}(\bfD)\right\}\Phi_1^{(3,4)\top}\right] \eqsp.
        \label{eq:s1_decomposition}
    \end{align}
    We introduce for notational convenience,
    \begin{align}
        \widetilde{\Brm}(\bfD) = \left(\Id{2} + n^{-1} \E\left[\Phi_1^{(1,2)\top}\Resolvent_-(\bfD) \Phi_1^{(1,2)}\right]\right)^{-1} \eqsp.
    \end{align}
    The previous calculation in \eqref{eq:s1_decomposition} combined with the triangle inequality yields,
    \begin{align}
        \left|s_1 - \overline{s}_1\right|
        \leq& \eqsp
        \left|\E\left[\Phi_1^{(3,4)} \left\{\widetilde{\Brm}(\bfD) - \Brm(\bfD)\right\}\Phi_1^{(3,4)\top}\right]\right|
        +
        \left|\E\left[\Phi_1^{(3,4)} \left\{\widetilde{\Brm}(\bfD) - \hat{\Brm}(\bfD)\right\}\Phi_1^{(3,4)\top}\right]\right| \\
        \leq& \eqsp
        \dfrac{1}{n^2}\left|\E\left[\Phi_1^{(3,4)} \widetilde{\Brm}(\bfD)\E\left[\Phi_1^{(1,2)\top}\E\left[\Resolvent_-(\bfD)\Phi_1^{(1,2)}\hat{\Brm}(\bfD)\Phi_1^{(1,2)\top}\Resolvent_-(\bfD)\right]\Phi_1^{(1,2)}\right]\Brm(\bfD) \Phi_1^{(3,4)\top}\right]\right| \\
        &+ \dfrac{1}{n}\left|\E\left[\Phi_1^{(3,4)} \widetilde{\Brm}(\bfD)\left\{ \Phi_1^{(1,2)\top} \Resolvent_-(\bfD) \Phi_1^{(1,2)} - \E\left[\Phi_1^{(1,2)\top}\Resolvent_-(\bfD) \Phi_1^{(1,2)}\right]\right\}\hat{\Brm}(\bfD) \Phi_1^{(3,4)\top}\right]\right|
    \end{align}
    where the last inequality follows from \eqref{eq:resolvent_identity} and \eqref{eq:LeaveTwoOutIdentity}.
    Further rewritting the upper bound, and using the Cauchy-Schwarz inequality, we get,
    \begin{align}
        \left|s_1 - \overline{s}_1\right|
        \leq& \eqsp
        \dfrac{1}{n^2}\left|\tr\left(\Brm(\bfD) \E\left[\Phi_1^{(3,4)\top}\Phi_1^{(3,4)}\right] \widetilde{\Brm}(\bfD)\E\left[\Phi_1^{(1,2)\top}\E\left[\Resolvent_-(\bfD)\Phi_1^{(1,2)}\hat{\Brm}(\bfD)\Phi_1^{(1,2)\top}\Resolvent_-(\bfD)\right]\Phi_1^{(1,2)}\right]\right)\right| \\
        &+ \dfrac{1}{n}\left|\tr\left(\E\left[\hat{\Brm}(\bfD) \Phi_1^{(3,4)\top}\Phi_1^{(3,4)} \widetilde{\Brm}(\bfD)\left\{\Phi_1^{(1,2)\top} \Resolvent_-(\bfD) \Phi_1^{(1,2)} - \E\left[\Phi_1^{(1,2)\top}\Resolvent_-(\bfD) \Phi_1^{(1,2)}\right]\right\}\right]\right)\right| \\
        \label{eq:s1_term1_bound_1_1}
        \leq& \eqsp \dfrac{2\kappa_n^2}{n^2}  \left\|\E\left[\Phi_1^{(1,2)\top}\E\left[\Resolvent_-(\bfD)\Phi_1^{(1,2)}\hat{\Brm}(\bfD)\Phi_1^{(1,2)\top}\Resolvent_-(\bfD)\right]\Phi_1^{(1,2)}\right]\right\|_\Frob \\
        &+ \dfrac{1}{n} \sqrt{\Var\left(\hat{\Brm}(\bfD) \Phi_1^{(3,4)\top}\Phi_1^{(3,4)} \widetilde{\Brm}(\bfD)\right) \Var\left(\Phi_1^{(1,2)\top} \Resolvent_-(\bfD) \Phi_1^{(1,2)}\right)} \eqsp,
    \end{align}
    where we have used \Cref{lemma:Covariance_maximum_eigenvalue}, as well as $\widetilde{\Brm}(\bfD) \preceq \Id{2}$ and $\Brm(\bfD) \preceq \Id{2}$.
    Additionaly,
    \begin{align}
        \left\|\E\left[\Phi_1^{(1,2)\top}\E\left[\Resolvent_-(\bfD)\Phi_1^{(1,2)}\hat{\Brm}(\bfD)\Phi_1^{(1,2)\top}\Resolvent_-(\bfD)\right]\Phi_1^{(1,2)}\right]\right\|_\Frob
        \leq \dfrac{4\sqrt{p}\kappa_n^4}{\uplambda_{\min}^\Re(\bfD)^2} \eqsp,
    \end{align}
    and, $\widetilde{\Brm}(\bfD)$ being deterministic with $\widetilde{\Brm}(\bfD) \preceq \Id{2}$, we have,
    \begin{align}
        \Var\left(\hat{\Brm}(\bfD) \Phi_1^{(3,4)\top}\Phi_1^{(3,4)} \widetilde{\Brm}(\bfD)\right) \leq \Var\left(\hat{\Brm}(\bfD) \Phi_1^{(3,4)\top}\Phi_1^{(3,4)}\right) \eqsp.
    \end{align}
    Thus, \eqref{eq:s1_term1_bound_1_1} implies,
    \begin{align} \label{eq:s1_term1_bound_1}
        \left|s_1 - \overline{s}_1\right|
        \leq
        \dfrac{8\sqrt{p} \kappa_n^6}{\uplambda_{\min}^\Re(\bfD)^2n^2} + \dfrac{1}{n} \sqrt{\Var\left(\hat{\Brm}(\bfD) \Phi_1^{(3,4)\top}\Phi_1^{(3,4)} \right) \Var\left(\Phi_1^{(1,2)\top} \Resolvent_-(\bfD) \Phi_1^{(1,2)}\right)} \eqsp.
    \end{align}
    By similar lines as in \eqref{eq:Variance_of_quadform_bound} we get,
    \begin{align} \label{eq:s1_term1_bound_var1}
        \Var\left(\Phi_1^{(1,2)\top} \Resolvent_-(\bfD) \Phi_1^{(1,2)}\right)
        \leq
        \const \left(\dfrac{\kappa_n^4 p}{\uplambda_{\min}^{\Re}(\bfD)^2} + \dfrac{\kappa_n^6 p}{\uplambda_{\min}^{\Re}(\bfD)^3 n}\right) \eqsp,
    \end{align}
    finally, using \Cref{lemma:matrix_variance_product_bound} and \Cref{corollary:randomHansonWright}, we get,
    \begin{align} \label{eq:s1_term1_bound_var2}
        \Var\left(\hat{\Brm}(\bfD) \Phi_1^{(3,4)\top}\Phi_1^{(3,4)} \right) &\leq \Var\left(\Phi_1^{(3,4)\top} \Phi_1^{(3,4)}\right) + \kappa_n^4 \Var\left(\hat{\Brm}(\bfD)\right) \\
        &\leq \Var\left(\Phi_1^{(3,4)\top} \Phi_1^{(3,4)}\right) + \dfrac{\kappa_n^4 }{n^2}\Var\left(\Phi_1^{(1,2)\top} \Resolvent_-(\bfD) \Phi_1^{(1,2)}\right) \\
        &\leq 2\kappa_n^4 +  \dfrac{\const}{n^2} \left(\dfrac{\kappa_n^8 p}{\uplambda_{\min}^{\Re}(\bfD)^2} + \dfrac{\kappa_n^{10} p}{\uplambda_{\min}^{\Re}(\bfD)^3 n}\right)  \eqsp.
    \end{align}
    Plugging \eqref{eq:s1_term1_bound_var1} and \eqref{eq:s1_term1_bound_var2} into \eqref{eq:s1_term1_bound_1}, we get for a large enough constant $\const$,
    \begin{align}
        \left|s_1 - \overline{s}_1\right|
        \leq \eqsp
        &\dfrac{\const \kappa_n^4 \sqrt{p}}{\uplambda_{\min}^\Re(\bfD)^2 n} + \dfrac{\const \kappa_n^5 \sqrt{p}}{\uplambda_{\min}^\Re(\bfD)^{3/2} n^{3/2}}
        + \dfrac{\const p}{n^2} \left(\dfrac{\kappa_n^6}{\uplambda_{\min}^\Re(\bfD)^2} + \dfrac{\kappa_n^7}{\uplambda_{\min}^\Re(\bfD)^{5/2} \sqrt{n}}\right)
        + \dfrac{\const \kappa_n^8 p}{\uplambda_{\min}^\Re(\bfD)^{3} n^3} \eqsp. \label{eq:s1_equiv_final_bound}
    \end{align}
    Moving on to the second term $s_2$, we denote by $\mathcal{F}_- = \sigma(\Psi_i \,: \, i \in \{2,\ldots,n\})$ and $\mathcal{F}_{--} = \sigma(\Psi_i \,: \, i \in \{3,\ldots,n\})$ 
    the $\sigma$-algebra generated by $\Psi$ but its first two columns, as well as the following dilation matrices where two samples were removed,
    \begin{align}
        \hat{\Brm}_{-i}(\bfD) = \left(\Id{2} + n^{-1} \Phi_i^{(1,2)\top}\Resolvent_{--}(\bfD) \Phi_i^{(1,2)}\right)^{-1} \eqsp, \quad \widetilde{\Brm}_{-i}(\bfD) = \left(\Id{2} + n^{-1} \E\left[\Phi_i^{(1,2)\top}\Resolvent_{--}(\bfD) \Phi_i^{(1,2)}\right]\right)^{-1}
    \end{align}
    We then have from \Cref{lemma:Leave-One-Out-Lemma} and the tower property, 
    \begin{align}
        s_2 =&\eqsp \dfrac{n-1}{n} \E\left[\Phi_1^{(3,4)} \Phi_1^{(1,2)\top} \Resolvent(\bfD) \Phi_2^{(1,2)} \Phi_2^{(3,4)\top}\right] \\
        =&\eqsp \dfrac{n-1}{n} \E\left[\Phi_1^{(3,4)} \hat{\Brm}(\bfD) \Phi_1^{(1,2)\top} \Resolvent_{--}(\bfD) \Phi_2^{(1,2)} \hat{\Brm}_{-2}(\bfD) \Phi_2^{(3,4)\top}\right] \\
        =&\eqsp   \tilde{s}_2^{1}
        +  \dfrac{n-1}{n} \E\left[\Phi_1^{(3,4)} \left\{\hat{\Brm}(\bfD) - \hat{\Brm}_{-1}(\bfD)\right\} \Phi_1^{(1,2)\top} \Resolvent_{--}(\bfD) \Phi_2^{(1,2)} \hat{\Brm}_{-2}(\bfD) \Phi_2^{(3,4)\top}\right]\eqsp,
    \end{align}
where we write for  ease of notation
    \begin{align}
        \tilde{s}_2^{1} = \dfrac{n-1}{n} \E\left[\E\left[\Phi_1^{(3,4)} \hat{\Brm}_{-1}(\bfD) \Phi_1^{(1,2)\top} \eqsp \middle| \eqsp \mathcal{F}_{--}\right] \Resolvent_{--}(\bfD) \E\left[\Phi_2^{(1,2)} \hat{\Brm}_{-2}(\bfD) \Phi_2^{(3,4)\top} \eqsp \middle| \eqsp \mathcal{F}_{--}\right]\right] \eqsp.
    \end{align}
Then,  using \eqref{eq:resolvent_identity}, we have,
    \begin{align}
        &\left|s_2 - \tilde{s}_2^{1}\right| \\
        \leq&\eqsp \dfrac{1}{n}\left|\E\left[\Phi_1^{(3,4)} \hat{\Brm}(\bfD)\Phi_1^{(1,2)\top}\left\{\Resolvent_{--}(\bfD) - \Resolvent_{-}(\bfD)\right\}\Phi_1^{(1,2)}\hat{\Brm}_{-1}(\bfD) \Phi_1^{(1,2)\top} \Resolvent_{--}(\bfD) \Phi_2^{(1,2)} \hat{\Brm}_{-2}(\bfD) \Phi_2^{(3,4)\top}\right]\right| \\
        \leq&\eqsp \dfrac{1}{n^2}\left|\E\left[\Phi_1^{(3,4)} \hat{\Brm}(\bfD)\Phi_1^{(1,2)\top}\Resolvent_{--}(\bfD) \Phi_2^{(1,2)}\Phi_2^{(1,2)\top} \Resolvent_{-}(\bfD)\Phi_1^{(1,2)}\hat{\Brm}_{-1}(\bfD) \Phi_1^{(1,2)\top} \Resolvent_{--}(\bfD) \Phi_2^{(1,2)} \hat{\Brm}_{-2}(\bfD) \Phi_2^{(3,4)\top}\right]\right| \\
        \leq&\eqsp \dfrac{1}{n^2}\left|\E\left[\Phi_1^{(3,4)} \hat{\Brm}(\bfD)\Phi_1^{(1,2)\top}\Resolvent_{--}(\bfD) \Phi_2^{(1,2)}\hat{\Brm}_{-2}(\bfD)\Phi_2^{(1,2)\top}\Resolvent_{--}(\bfD)\Phi_1^{(1,2)}\hat{\Brm}_{-1}(\bfD) \Phi_1^{(1,2)\top} \Resolvent_{--}(\bfD) \Phi_2^{(1,2)} \hat{\Brm}_{-2}(\bfD) \Phi_2^{(3,4)\top}\right]\right| \eqsp,
    \end{align}
    where we used \Cref{lemma:Leave-One-Out-Lemma} in the last inequality.
Further writting
    \begin{align}
        &\chi_1(\bfD) = \Phi_2^{(1,2)} \hat{\Brm}_{-2}(\bfD) \Phi_2^{(1,2)\top} \Resolvent_{--}(\bfD) \Phi_1^{(1,2)} \hat{\Brm}(\bfD) \Phi_1^{(3,4)\top} \eqsp, \\
        &\chi_2(\bfD) = \Phi_1^{(1,2)}\hat{\Brm}_{-1}(\bfD) \Phi_1^{(1,2)\top} \Resolvent_{--}(\bfD) \Phi_2^{(1,2)} \hat{\Brm}_{-2}(\bfD) \Phi_2^{(3,4)\top} \eqsp,
    \end{align}
    we have from the previous as well as the Cauchy-Schwarz inequality,
    \begin{align}
        \left|s_2 - \tilde{s}_2^{1}\right|
        \leq&\eqsp \dfrac{1}{n^2}\left|\E\left[\chi_1(\bfD)^{\top} \Resolvent_{--}(\bfD) \chi_2(\bfD)\right]\right| \\
        \leq&\eqsp \dfrac{1}{n^2}\left\|\E\left[\chi_1(\bfD)\right]\right\|_2 \left\|\E\left[\Resolvent_{--}(\bfD)\right]\right\|_\op \left\|\E\left[\chi_2(\bfD)\right]\right\|_2 \\
        &+ \dfrac{1}{n^2} \biggl(\sqrt{\Var\left(\chi_1(\bfD)\right)\Var\left(\Resolvent_{--}(\bfD)\chi_2(\bfD)\right)} \\
        &+ \left\|\E\left[\chi_1(\bfD)\right]\right\|_2 \sqrt{\E\left[\|\Resolvent_{--}(\bfD) - \E\left[\Resolvent_{--}(\bfD)\right]\|_\op^2\right] \Var\left(\chi_2(\bfD)\right)} \biggr)\eqsp.
    \end{align}
    In order to bound this error, we first remark that,
    \begin{align}
        \left\|\Resolvent_{--}(\bfD)\right\|_\op \leq 1/ {\uplambda_{\min}^\Re(\bfD)} \eqsp,
\qquad         \E\left[\|\Resolvent_{--}(\bfD) - \E\left[\Resolvent_{--}(\bfD)\right]\|_\op^2\right]
\leq \const/ {\uplambda_{\min}^\Re(\bfD)}^2 \eqsp.
    \end{align}
    Furthermore, we have from lines simialar as \Cref{lemma:matrix_variance_product_bound},
    \begin{align}
        \Var\left(\Resolvent_{--}(\bfD) \chi_2(\bfD)\right)
        \leq&\eqsp
        \dfrac{2}{\uplambda_{\min}^\Re(\bfD)^2}\Var\left(\chi_2(\bfD)\right) +2 \left\|\E\left[\chi_2(\bfD)\right]\right\|_2 \E\left[\|\Resolvent(\bfD) - \E\left[\Resolvent(\bfD)\right]\|_\op^2\right] \\
        \leq&\eqsp
        \dfrac{2}{\uplambda_{\min}^\Re(\bfD)^2}\left(\Var\left(\chi_2(\bfD)\right) + \left\|\E\left[\chi_2(\bfD)\right]\right\|_2^2\right)\eqsp,
    \end{align}
    where we have used
    \begin{align}
        \E\left[\|\Resolvent(\bfD) - \E\left[\Resolvent(\bfD)\right]\|_\op^2\right] \leq \dfrac{1}{\uplambda_{\min}^\Re(\bfD)^2} \eqsp.
    \end{align}
    We have shown,
    \begin{align}
        \left|s_2 - \tilde{s}_2^{1}\right|
        \leq&\eqsp
        \dfrac{1}{\uplambda_{\min}^{\Re}(\bfD)n^2}\left\|\E\left[\chi_1(\bfD)\right]\right\|_2 \left\|\E\left[\chi_2(\bfD)\right]\right\|_2 + \dfrac{1}{\uplambda_{\min}^{\Re}(\bfD)n^2}\sqrt{\Var(\chi_1(\bfD))\Var(\chi_2(\bfD))} \\
        &+ \dfrac{\left\|\E\left[\chi_2(\bfD)\right]\right\|_2}{\uplambda_{\min}^\Re(\bfD)n^{2}} \sqrt{\Var\left(\chi_1(\bfD)\right)}
        + \dfrac{\const \left\|\E\left[\chi_1(\bfD)\right]\right\|_2}{\uplambda_{\min}^\Re(\bfD) n^2 } \sqrt{\Var\left(\chi_2(\bfD)\right)} \eqsp.
        \label{eq:Bound_s2_minus_s2tilde1_almostthere}
    \end{align}
    In addition, we bound the first two moments of $\chi_1(\bfD)$ and $\chi_2(\bfD)$, first by conditionning on $\mathcal{F}_{--}$ and using the tower property, \Cref{lemma:Covariance_maximum_eigenvalue}, we write
    \begin{align}
        \left\|\E\left[\chi_2(\bfD)\right]\right\|_2
        &\leq \E\left[\left\|\E\left[\Phi_1^{(1,2)}\hat{\Brm}_{-1}(\bfD) \Phi_1^{(1,2)\top} \eqsp \middle| \eqsp \mathcal{F}_{--}\right] \Resolvent_{--}(\bfD) \E\left[\Phi_2^{(1,2)} \hat{\Brm}_{-2}(\bfD) \Phi_2^{(3,4)\top} \eqsp \middle|\eqsp \mathcal{F}_{--}\right]\right\|_2\right]  \\
        &\leq \dfrac{\kappa_n^4}{\uplambda_{\min}^{\Re}(\bfD)}\eqsp,
        \label{eq:E_chi2_bound_final}
    \end{align}
    similarly for the expectation of $\chi_1(\bfD)$, write,
    \begin{align}
        \left\|\E\left[\chi_1(\bfD)\right]\right\|_2
        \leq&\eqsp
        \left\| \E\left[ \Phi_2^{(1,2)} \hat{\Brm}_{-2}(\bfD) \Phi_2^{(1,2)\top} \Resolvent_{--}(\bfD)\right] \E\left[\Phi_1^{(1,2)} \hat{\Brm}(\bfD) \Phi_1^{(3,4)\top}\right]\right\|_2 \\
        &+ \sup_{\|\mathbf{a}\|_2 = 1} \sqrt{\Var\left( \mathbf{a}^\top \Phi_2^{(1,2)} \hat{\Brm}_{-2}(\bfD) \Phi_2^{(1,2)\top} \Resolvent_{--}(\bfD) \right)\Var\left(\Phi_1^{(1,2)} \hat{\Brm}(\bfD) \Phi_1^{(3,4)\top} \right)} \\
        \leq&\eqsp
        \dfrac{\kappa_n^4}{\uplambda_{\min}^\Re(\bfD)} + \sup_{\|\mathbf{a}\|_2 = 1}\sqrt{\Var\left( \mathbf{a}^\top \Phi_2^{(1,2)} \hat{\Brm}_{-2}(\bfD) \Phi_2^{(1,2)\top} \Resolvent_{--}(\bfD) \right)\Var\left(\Phi_1^{(3,4)} \hat{\Brm}(\bfD) \Phi_1^{(1,2)\top}\right)} \eqsp.
    \end{align}
    Thanks to \Cref{lemma:matrix_variance_product_bound}, and \Cref{corollary:randomHansonWright} we have for any $\mathbf{a} \in \sphere^p$,
    \begin{align}
        \Var\left( \mathbf{a}^\top \Phi_2^{(1,2)} \hat{\Brm}_{-2}(\bfD) \Phi_2^{(1,2)\top} \Resolvent_{--}(\bfD) \right)
        &\leq \dfrac{2}{\uplambda_{\min}^\Re(\bfD)^2}\Var\left( \mathbf{a}^\top \Phi_2^{(1,2)} \hat{\Brm}_{-2}(\bfD) \Phi_2^{(1,2)\top} \right)
            + 2\kappa_n^4 \Var\left(\Resolvent_{--}(\bfD)\right) \\
        &\leq \dfrac{2}{\uplambda_{\min}^\Re(\bfD)^2}\Var\left( \mathbf{a}^\top \Phi_2^{(1,2)} \hat{\Brm}_{-2}(\bfD) \Phi_2^{(1,2)\top} \right)
        + \dfrac{2\kappa_n^6 p}{\uplambda_{\min}^\Re(\bfD)^3 n} \eqsp.
    \end{align}
    Furthermore, we have,
    \begin{align}
        \Var\left( \mathbf{a}^\top \Phi_2^{(1,2)} \hat{\Brm}_{-2}(\bfD) \Phi_2^{(1,2)\top} \right) &= \Var\left(\sum_{i,j=1}^2 \left(\hat{\Brm}_{-2}(\bfD)\right)_{i,j} \mathbf{a}^\top \Phi_2^{(i)}\Phi_2^{(j)\top}\right) \\
        &\leq 4 \sum_{i,j=1}^2 \Var\left( \left(\hat{\Brm}_{-2}(\bfD)\right)_{i,j} \mathbf{a}^\top \Phi_2^{(i)}\Phi_2^{(j)\top}\right) \\
        &\leq  \const \kappa_n^4 + \const   \sum_{i,j=1}^2 \Var\left( \mathbf{a}^\top \Phi_2^{(i)}\Phi_2^{(j)\top}\right)\leq  \const \kappa_n^4   \eqsp,
    \end{align}
    similarly,
    \begin{align} \label{eq:ThisRandomQuadraticFormBoundIWantToCite}
        \Var\left(\Phi_1^{(3,4)} \hat{\Brm}(\bfD) \Phi_1^{(1,2)\top}\right)
        &\leq \const \kappa_n^4 \eqsp.
    \end{align}
    Putting this all together, we have,
    \begin{align} \label{eq:E_chi1_bound_final}
        \left\|\E\left[\chi_1(\bfD)\right]\right\|_2  \leq \dfrac{\const \kappa_n^4}{\uplambda_{\min}^\Re(\bfD)} + \dfrac{\const\kappa_n^5 \sqrt{p}}{\uplambda_{\min}^\Re(\bfD)^{3/2} \sqrt{n}} \eqsp.
    \end{align}
    Respectively for the variances of $\chi_1(\bfD)$ and $\chi_2(\bfD)$, we apply the law of total variance,
    to this end denote $M^{(i,j)}(\bfD) = \E[\Phi_1^{(1,2)\top}\hat{\Brm}_{-1}(\bfD) \Phi_1^{(i,j)} \mid \mathcal{F}_{--}]$,
    first dealing with $\chi_2(\bfD)$, we have
    \begin{align}
        \Var\left(\chi_2(\bfD)\right)
        =&\eqsp
        \E\left[\Var\left(\chi_2(\bfD) \mid \mathcal{F}_{--}\right)\right] + \Var\left(M^{(1,2)}(\bfD)\Resolvent_{--}(\bfD) M^{(3,4)}(\bfD) \right) \\
        \leq&\eqsp
        \E\left[\Var\left(\chi_2(\bfD) \mid \mathcal{F}_{--}\right)\right] + \E\left[\left\|M^{(1,2)}(\bfD)\Resolvent_{--}(\bfD) M^{(3,4)}(\bfD) \right\|_2^2 \right]\\
        \leq&\eqsp
        \E\left[\Var\left(\chi_2(\bfD) \mid \mathcal{F}_{--}\right)\right] + \dfrac{\kappa_n^8}{\uplambda_{\min}^\Re(\bfD)^2}
        \eqsp,
    \end{align}
    and, from the independance of $(\Phi_1^{(1,2)}, \Phi_1^{(3,4)})$ with $(\Phi_2^{(1,2)}, \Phi_2^{(3,4)})$
    \begin{align}
        &\Var\left(\chi_2(\bfD) \mid \mathcal{F}_{--}\right) \\
        \leq&\eqsp
        \dfrac{1}{\uplambda_{\min}^\Re(\bfD)^2} \biggr\{\Var\left(\Phi_1^{(1,2)}\hat{\Brm}_{-1}(\bfD) \Phi_1^{(1,2)\top} \mid \mathcal{F}_{--}\right)\Var\left( \Phi_2^{(1,2)}\hat{\Brm}_{-2}(\bfD) \Phi_2^{(3,4)\top} \mid \mathcal{F}_{--}\right) \\
        & \hspace{1.7cm} + \left\|\E\left[ \Phi_1^{(1,2)}\hat{\Brm}_{-1}(\bfD) \Phi_1^{(1,2)\top} \mid \mathcal{F}_{--}\right]\right\|_\op^2 \Var\left( \Phi_2^{(1,2)}\hat{\Brm}_{-2}(\bfD) \Phi_2^{(3,4)} \mid \mathcal{F}_{--}\right) \\
        & \hspace{1.7cm} + \Var\left(\Phi_1^{(1,2)}\hat{\Brm}_{-1}(\bfD) \Phi_1^{(1,2)\top} \mid \mathcal{F}_{--}\right) \left\|\E\left[\Phi_2^{(1,2)}\hat{\Brm}_{-1}(\bfD) \Phi_2^{(3,4)\top} \mid \mathcal{F}_{--}\right]\right\|_2^2 \biggl\} \\
        \leq&\eqsp
        \dfrac{1}{\uplambda_{\min}^\Re(\bfD)^2} \biggr\{\Var\left(\Phi_1^{(1,2)}\hat{\Brm}_{-1}(\bfD) \Phi_1^{(1,2)\top} \mid \mathcal{F}_{--}\right)\Var\left( \Phi_2^{(1,2)}\hat{\Brm}_{-2}(\bfD) \Phi_2^{(3,4)\top} \mid \mathcal{F}_{--}\right) \\
        & \hspace{1.7cm} + \kappa_n^4 \Var\left( \Phi_2^{(1,2)}\hat{\Brm}_{-2}(\bfD) \Phi_2^{(3,4)\top} \mid \mathcal{F}_{--}\right) \\
        & \hspace{1.7cm} + \kappa_n^4 \Var\left(\Phi_1^{(1,2)}\hat{\Brm}_{-1}(\bfD) \Phi_1^{(1,2)\top} \mid \mathcal{F}_{--}\right) \biggl\}
        \eqsp,
    \end{align}
    similarly as for \eqref{eq:ThisRandomQuadraticFormBoundIWantToCite}, it holds,
    \begin{align}
      \label{eq:hanson_to_do}
        \Var\left(\Phi_1^{(1,2)}\hat{\Brm}_{-1}(\bfD) \Phi_1^{(1,2)\top} \mid \mathcal{F}_{--}\right) \leq \const p \kappa_n^4
        \eqsp, \quad
        \Var\left(\Phi_2^{(1,2)}\hat{\Brm}_{-2}(\bfD) \Phi_2^{(3,4)\top} \mid \mathcal{F}_{--}\right) \leq \const \kappa_n^4 \eqsp,
    \end{align}
    thus,
    \begin{align} \label{eq:Var_chi2_bound_final}
        \Var\left(\chi_2(\bfD)\right) \leq \dfrac{\const \kappa_n^8 p}{\uplambda_{\min}^\Re(\bfD)^2} \eqsp.
    \end{align}
    Finally, to control the variance of $\chi_1(\bfD)$, first use the law of total variance to write,
    \begin{align}
        \Var\left(\chi_1(\bfD)\right)
        =&\eqsp \dfrac{1}{\uplambda_{\min}^\Re(\bfD)^2}\E\left[\left\|\Phi_2^{(1,2)}\hat{\Brm}_{-2}(\bfD) \Phi_2^{(1,2)\top}\right\|_\op^2 \Var\left(\Phi_1^{(1,2)} \hat{\Brm}(\bfD) \Phi_1^{(3,4)\top} \mid \mathcal{F}_{-}\right)\right] \\
        &+ \Var\left(\Phi_2^{(1,2)}\hat{\Brm}_{-2}(\bfD) \Phi_2^{(1,2)\top} \Resolvent_{--}(\bfD)\E[\Phi_1^{(1,2)} \hat{\Brm}(\bfD)\Phi_1^{(3,4)\top} \mid \mathcal{F}_{-}]\right) \eqsp,
    \end{align}
proceeding similarly to \eqref{eq:ThisRandomQuadraticFormBoundIWantToCite} and using $\Phi_2^{(1,2)}\hat{\Brm}_{-2}(\bfD) \Phi_2^{(1,2)\top} \preceq \Phi_2^{(1,2)} \Phi_2^{(1,2)\top}$ on the first term,
    as well as \Cref{lemma:matrix_variance_product_bound} on the second yields,
    \begin{align}
        \Var\left(\chi_1(\bfD)\right)
        \leq&\eqsp \dfrac{\const \kappa_n^4}{\uplambda_{\min}^\Re(\bfD)^2} \E\left[\|\Phi_2^{(1,2)} \Phi_2^{(1,2)\top}\|_\op^2\right] + \dfrac{\kappa_n^4}{\uplambda_{\min}^\Re(\bfD)^2} \Var\left(\Phi_2^{(1,2)}\hat{\Brm}_{-2}(\bfD) \Phi_2^{(1,2)\top} \right)\\
        & + \kappa_n^4 \Var\left(\Resolvent_{--}(\bfD)\E[\Phi_1^{(3,4)} \hat{\Brm}(\bfD)\Phi_1^{(1,2)\top} \mid \mathcal{F}_{-}]\right) \eqsp,
    \end{align}
    we conclude using that $\|x x^{\top}\|^2_{\op} = \|x\|_2^4$,
    \begin{align} \label{eq:Var_chi1_bound_final}
        \Var\left(\chi_1(\bfD)\right) \leq \const\dfrac{\kappa_n^8 p^2}{\uplambda_{\min}^\Re(\bfD)^2}  \eqsp.
    \end{align}
    Plugging the previous upper bounds \eqref{eq:E_chi2_bound_final}, \eqref{eq:E_chi1_bound_final}, \eqref{eq:Var_chi2_bound_final} and \eqref{eq:Var_chi1_bound_final}
    inside of  \eqref{eq:Bound_s2_minus_s2tilde1_almostthere}, we get,
    \begin{align}\label{eq:s2_equiv_s2tilde1_final}
        \left|s_2 - \tilde{s}_2^{1}\right| \leq\const \biggl(&\eqsp \dfrac{\sqrt{p}}{n^{5/2}}\left(\dfrac{\kappa_n^8}{\uplambda_{\min}^\Re(\bfD)^3} +  \dfrac{\kappa_n^9}{\uplambda_{\min}^\Re(\bfD)^{7/2}}\right) 
        + \dfrac{\kappa_n^8 p^{3/2}}{\uplambda_{\min}^\Re(\bfD)^{3} n^2}  \\ 
        &+ \dfrac{p}{n^{5/2}} \left(\dfrac{\kappa_n^8}{\uplambda_{\min}^\Re(\bfD)^2} +  \dfrac{\kappa_n^9}{\uplambda_{\min}^\Re(\bfD)^{5/2}} + \dfrac{\kappa_n^9}{\uplambda_{\min}^\Re(\bfD)^{7/2}}\right)\biggr) \eqsp.
    \end{align}
    In order to further simplify the expression of the equivalent $\tilde{s}_2^{(1)}$, we define,
    \begin{align}
        \tilde{s}_2^{2} = \dfrac{n-1}{n} \E\left[\Phi_1^{(3,4)} \widetilde{\Brm}_{-1}(\bfD) \Phi_1^{(1,2)\top}\right] \E\left[\Resolvent_{--}(\bfD)\right] \E\left[\Phi_2^{(1,2)} \widetilde{\Brm}_{-2}(\bfD) \Phi_2^{(3,4)\top}\right] \eqsp,
    \end{align}
    and we have,
    \begin{align}
        &\dfrac{n}{n-1} \left(\tilde{s}_2^2 - \tilde{s}_2^1\right) \\
        =&\eqsp \E\left[\E\left[\Phi_1^{(3,4)} \left\{\widetilde{\Brm}_{-1}(\bfD) - \hat{\Brm}_{-1}(\bfD) \right\} \Phi_1^{(1,2)\top} \mid \mathcal{F}_{--}\right]\Resolvent_{--}(\bfD)\E\left[\Phi_2^{(1,2)} \hat{\Brm}_{-2}(\bfD) \Phi_2^{(3,4)\top} \mid \mathcal{F}_{--}\right]\right] \\
        &+  \E\left[\E\left[\Phi_1^{(3,4)} \widetilde{\Brm}_{-1}(\bfD) \Phi_1^{(1,2)\top} \right]\Resolvent_{--}(\bfD)\E\left[\Phi_2^{(1,2)} \left\{\widetilde{\Brm}_{-2}(\bfD) - \hat{\Brm}_{-2}(\bfD)\right\} \Phi_2^{(3,4)\top} \mid \mathcal{F}_{--}\right]\right] \eqsp.
    \end{align}
    It results, from \Cref{lemma:Covariance_maximum_eigenvalue}
    \begin{align}
        \left|\tilde{s}_2^{1} - \tilde{s}_2^{2}\right|
        \leq&\eqsp
        \left|\E\left[\E\left[\Phi_1^{(3,4)}\left\{\widetilde{\Brm}_{-1}(\bfD) - \hat{\Brm}_{-1}(\bfD) \right\} \Phi_1^{(1,2)\top} \mid \mathcal{F}_{--}\right]\Resolvent_{--}(\bfD)\E\left[\Phi_2^{(1,2)} \hat{\Brm}_{-2}(\bfD) \Phi_2^{(3,4)\top} \mid \mathcal{F}_{--}\right]\right]\right| \\
        &+ \kappa_n^2 \left\|\E\left[\Resolvent_{--}(\bfD)\E\left[\Phi_1^{(1,2)} \left\{\widetilde{\Brm}_{-1}(\bfD) - \hat{\Brm}_{-1}(\bfD) \right\} \Phi_1^{(3,4)\top} \eqsp \middle| \eqsp \mathcal{F}_{--}\right]\right]\right\|_2\eqsp.
    \end{align}
    Recall that for any vector $\mathbf{v} \in \R^p$, it holds,
    $\|\mathbf{v}\|_2 = \sup_{\|\mathbf{a}\|_2 = 1} \mathbf{a}^\top \mathbf{v}$, thus using the Cauchy-Schwarz inequality and  \eqref{eq:resolvent_identity} for the difference $\widetilde{\Brm}_{-1}(\bfD) - \hat{\Brm}_{-1}(\bfD)$,
    \begin{align}
        \left|\tilde{s}_2^{1} - \tilde{s}_2^{2}\right|
        \leq&\eqsp
        \left|\E\left[\tr\left(\left\{\widetilde{\Brm}_{-1}(\bfD) - \hat{\Brm}_{-1}(\bfD) \right\} \Phi_1^{(1,2)\top} \Resolvent_{--}(\bfD)\E\left[\Phi_2^{(1,2)} \hat{\Brm}_{-2}(\bfD) \Phi_2^{(3,4)\top} \mid \mathcal{F}_{--}\right] \Phi_1^{(3,4)}\right)\right] \right| \\
        &+ \kappa_n^2\sup_{\|\mathbf{a}\|_2=1} \E\left[\tr\left(\left\{\widetilde{\Brm}_{-1}(\bfD) - \hat{\Brm}_{-1}(\bfD)\right\}\Phi_1^{(3,4)\top} \mathbf{a}^\top \Resolvent_{--}(\bfD) \Phi_1^{(1,2)} \right)\right] \\
        \leq&\eqsp \label{eq:Bound_1_tildeS_2_1_minus_tildeS_2_2}
        \dfrac{1}{n}\sqrt{\Var\left(\widetilde{\Brm}_{-1}(\bfD) \Phi_1^{(1,2)\top} \Resolvent_{--}(\bfD) \Phi_1^{(1,2)}\right) \Var\left(\hat{\Brm}_{-1}(\bfD) \Phi_1^{(1,2)\top} \Resolvent_{--}(\bfD)\E\left[\Phi_2^{(1,2)} \hat{\Brm}_{-2}(\bfD) \Phi_2^{(3,4)\top} \mid \mathcal{F}_{--}\right] \Phi_1^{(3,4)}\right)} \\
        &+ \sup_{\|\mathbf{a}\|_2=1} \dfrac{\kappa_n^2}{n}\sqrt{\Var\left( \widetilde{\Brm}_{-1}(\bfD)\Phi_1^{(1,2)\top} \Resolvent_{--}(\bfD) \Phi_1^{(1,2)}\right) \Var\left(\hat{\Brm}_{-1}(\bfD) \Phi_1^{(3,4)\top} \mathbf{a}^\top \Resolvent_{--}(\bfD) \Phi_1^{(1,2)} \right)} \eqsp.
    \end{align} 
    Towards bounding the previous error term, we write thanks to \Cref{corollary:randomHansonWright} and \Cref{lemma:matrix_variance_product_bound} that,
    \begin{align}\label{eq:DamnThisBoundIsEveryWhere}
        \Var\left(\widetilde{\Brm}_{-1}(\bfD) \Phi_1^{(1,2)\top} \Resolvent_{--}(\bfD) \Phi_1^{(1,2)}\right) \leq     \Var\left(\Phi_1^{(1,2)\top} \Resolvent_{--}(\bfD) \Phi_1^{(1,2)}\right) \leq \const \left(\dfrac{\kappa_n^4 p}{\uplambda_{\min}^\Re(\bfD)^2} + \dfrac{\kappa_n^6 p}{\uplambda_{\min}^\Re(\bfD)^3 n}\right) \eqsp,
    \end{align}

    \begin{align}
        \Var\left(\hat{\Brm}_{-1}(\bfD)\right) \leq \dfrac{1}{n^2} \Var\left(\Phi_1^{(1,2)\top} \Resolvent_{--}(\bfD) \Phi_1^{(1,2)}\right) \leq \dfrac{\const}{n^2} \left(\dfrac{\kappa_n^4 p}{\uplambda_{\min}^\Re(\bfD)^2} + \dfrac{\kappa_n^6 p}{\uplambda_{\min}^\Re(\bfD)^3 n}\right) \eqsp,
    \end{align}
    \begin{align}
        \Var\left(\hat{\Brm}_{-1}(\bfD) \Phi_1^{(3,4)\top} \mathbf{a}^\top \Resolvent_{--}(\bfD) \Phi_1^{(1,2)} \right) \leq&\eqsp \Var\left(\Phi_1^{(3,4)\top} \mathbf{a}^\top \Resolvent_{--}(\bfD) \Phi_1^{(1,2)}\right) + \dfrac{\kappa_n^4}{\uplambda_{\min}^\Re(\bfD)^2} \Var\left(\hat{\Brm}_{-1}(\bfD)\right) \\
        \leq&\eqsp \const \left(\dfrac{\kappa_n^4}{\uplambda_{\min}^\Re(\bfD)^2} + \dfrac{\kappa_n^6 p}{\uplambda_{\min}^\Re(\bfD)^3 n}\right) + \dfrac{\const p}{n^2} \left(\dfrac{\kappa_n^8}{\uplambda_{\min}^\Re(\bfD)^4} + \dfrac{\kappa_n^{10}}{\uplambda_{\min}^\Re(\bfD)^5 n}\right) \eqsp.
    \end{align}
    and,
    \begin{align}
        &\Var\left(\hat{\Brm}_{-1}(\bfD) \Phi_1^{(1,2)\top} \Resolvent_{--}(\bfD)\E\left[\Phi_2^{(1,2)} \hat{\Brm}_{-2}(\bfD) \Phi_2^{(3,4)\top} \mid \mathcal{F}_{--}\right] \Phi_1^{(3,4)} \right)\\
        \leq&\eqsp \Var\left(\Phi_1^{(1,2)\top} \Resolvent_{--}(\bfD)\E\left[\Phi_2^{(1,2)} \hat{\Brm}_{-2}(\bfD) \Phi_2^{(3,4)\top} \mid \mathcal{F}_{--}\right] \Phi_1^{(3,4)}\right) + \dfrac{\kappa_n^8}{\uplambda_{\min}^\Re(\bfD)^2}\Var\left(\hat{\Brm}_{-1}(\bfD)\right) \\
        \leq&  \const \left(\dfrac{\kappa_n^8}{\uplambda_{\min}^\Re(\bfD)^2} + \dfrac{\kappa_n^{10} p}{\uplambda_{\min}^\Re(\bfD)^3 n}\right) + \dfrac{\const p}{n^2} \left(\dfrac{\kappa_n^{12}}{\uplambda_{\min}^\Re(\bfD)^4} + \dfrac{\kappa_n^{14}}{\uplambda_{\min}^\Re(\bfD)^5 n}\right) \eqsp.
    \end{align}
    Plugging the previous fours variance bounds in \eqref{eq:Bound_1_tildeS_2_1_minus_tildeS_2_2}, we conclude
    \begin{align}
        \left|\tilde{s}_2^{1} - \tilde{s}_2^{2}\right| \leq \dfrac{\const \left(1 + \kappa_n^4\right)\sqrt{p}}{n} &\left(\dfrac{\kappa_n^2}{\uplambda_{\min}^\Re(\bfD)} + \dfrac{\kappa_n^3}{\uplambda_{\min}^\Re(\bfD)^{3/2} \sqrt{n}}\right) \\
        &\times \biggl(\dfrac{\kappa_n^2 }{\uplambda_{\min}^\Re(\bfD)} + \dfrac{\kappa_n^3 \sqrt{p}}{\uplambda_{\min}^\Re(\bfD)^{3/2} \sqrt{n}} + \dfrac{\kappa_n^4 \sqrt{p}}{\uplambda_{\min}^\Re(\bfD)^2 n} +  \dfrac{\kappa_n^5 \sqrt{p}}{\uplambda_{\min}^\Re(\bfD)^{5/2} n^{3/2}}\biggr) \eqsp.
        \label{eq:s2tilde1_equiv_final}
    \end{align}
    Finally, we show that $\tilde{s}_2^2$ is close to $\overline{s}_2$ defined as
    \begin{align}
        \overline{s}_2 = \E\left[\Phi_1^{(3,4)} \Brm(\bfD) \Phi_1^{(1,2)\top}\right] \detequiv(\bfD) \E\left[\Phi_1^{(1,2)}\Brm(\bfD) \Phi_1^{(1,2)}\right]\eqsp,
    \end{align}
    recall
    \begin{align}
        \tilde{s}_2^{2} = \dfrac{n-1}{n} \E\left[\Phi_1^{(3,4)} \widetilde{\Brm}_{-1}(\bfD) \Phi_1^{(1,2)\top}\right] \E\left[\Resolvent_{--}(\bfD)\right] \E\left[\Phi_2^{(1,2)} \widetilde{\Brm}_{-2}(\bfD) \Phi_2^{(3,4)\top}\right] \eqsp,
    \end{align}
    then it holds
    \begin{align}
        \left|\widetilde{s}_2^2 - \overline{s}_2\right|
        \leq&\eqsp
        \dfrac{\kappa_n^4}{\uplambda_{\min}^\Re(\bfD) n}
        + \dfrac{2\kappa_n^4}{\uplambda_{\min}^\Re(\bfD)} \left\|\Brm(\bfD) - \widetilde{\Brm}_{-1}(\bfD)\right\|_\op
        + \kappa_n^4 \left\|\E\left[\Resolvent_{--}(\bfD)\right] - \detequiv(\bfD)\right\|_\Frob \eqsp,
    \end{align}
    where we have used the i.i.d porperty of $(\Phi_1^{(1,2)}, \Phi_1^{(3,4)})$ and $(\Phi_2^{(1,2)}, \Phi_2^{(3,4)})$, we conclude on this error term by using \eqref{eq:resolvent_identity} which gives,
    \begin{align}
        \left\|\Brm(\bfD) - \widetilde{\Brm}_{-1}(\bfD)\right\|_\op 
        =&\eqsp
        \dfrac{1}{n}\left\|\Brm(\bfD)\E\left[\Phi_1^{(1,2)\top}\E\left[\Resolvent(\bfD) - \Resolvent_{--}(\bfD)\right]\Phi_1^{(1,2)}\right]\widetilde{\Brm}_{-1}(\bfD) \right\|_\op\\
        \leq&\eqsp
        \dfrac{\kappa_n^2}{n} \left\|\E\left[\Resolvent(\bfD) - \Resolvent_{--}(\bfD)\right]\right\|_\op 
        \leq
        \dfrac{2\kappa_n^2}{\uplambda_{\min}^\Re(\bfD) n}
        \eqsp.
    \end{align}
    It results,
    \begin{align}
        \label{eq:s2tilde2_equiv_final}
        \left|\widetilde{s}_2^2 - \overline{s}_2\right| 
        \leq&\eqsp 
        \dfrac{\kappa_n^4}{\uplambda_{\min}^\Re(\bfD) n} + \dfrac{4\kappa_n^6 }{\uplambda_{\min}^\Re(\bfD)^2 n} + \kappa_n^2 \left\|\E\left[\Resolvent_{--}(\bfD)\right] - \detequiv(\bfD)\right\|_\Frob \\
        \leq&\eqsp \const \left(\dfrac{\kappa_n^4}{\uplambda_{\min}^\Re(\bfD) n} + \dfrac{\kappa_n^6 }{\uplambda_{\min}^\Re(\bfD)^2 n} + \kappa_n^2 \left\|\E\left[\Resolvent(\bfD)\right] - \detequiv(\bfD)\right\|_\Frob \right)\eqsp.
    \end{align}
    Putting this all together, we use the triangle inequality to get,
    \begin{align}
        \left|s_2 - \overline{s}_2\right| 
        \leq&\eqsp
        \const \kappa_n^4 \left\|\E\left[\Resolvent(\bfD)\right] - \detequiv(\bfD)\right\|_\Frob 
        + \dfrac{\const (1 +\kappa_n^{10}) \sqrt{p}}{\uplambda_{\min}^\Re(\bfD)^2 n} \left(1 + \dfrac{\sqrt{p}}{\uplambda_{\min}^\Re(\bfD)^{1/2} \sqrt{n}}\right) 
        + \dfrac{\const p^{3/2}}{n^2} \dfrac{\kappa_n^8 }{\uplambda_{\min}^\Re(\bfD)^{3}}\\
        \label{eq:s2_equiv_final_bound}
        &+ \dfrac{\const (1 +\kappa_n^{10})\sqrt{p}}{\uplambda_{\min}^\Re(\bfD)^{5/2} n^{3/2}} \left(1 + \dfrac{\sqrt{p}}{\uplambda_{\min}^\Re(\bfD)^{1/2} \sqrt{n}}\right) 
        + \dfrac{\const}{n}\dfrac{\kappa_n^4}{\uplambda_{\min}^\Re(\bfD)} \\
        &+ \dfrac{\const (1+\kappa_n^{10})p}{\uplambda_{\min}^\Re(\bfD)^2 n^{5/2}}\left(1 +  \dfrac{1}{\uplambda_{\min}^\Re(\bfD)^{1/2}} + \dfrac{1}{\uplambda_{\min}^\Re(\bfD)^{3/2}} \right) \\
        &+ \dfrac{\const p}{n^3} \dfrac{1 + \kappa_n^{10}}{\uplambda_{\min}^\Re(\bfD)^4}
        + \dfrac{\const (1+\kappa_n^{10})\sqrt{p}}{\uplambda_{\min}^\Re(\bfD)^3n^{5/2}}\left(1 +  \dfrac{1}{\uplambda_{\min}^\Re(\bfD)^{1/2}}\right) \eqsp.
    \end{align}
    Where the last inequality follows from \eqref{eq:s2_equiv_s2tilde1_final}, \eqref{eq:s2tilde1_equiv_final} and \eqref{eq:s2tilde2_equiv_final}.
    The statement follows from \eqref{eq:s1_equiv_final_bound} and \eqref{eq:s2_equiv_final_bound} and the triangle inequality.
\end{proof}

Similarly as for \Cref{prop:DetEquiv_Q_noaug}, \Cref{prop:DetEquiv_S_noaug} may be directly applied to approximate $\E\left[\xi_{\alpha, \lambda}(\zeta)\right]$ defined in \eqref{eq:def_xi_alpha_lambda_zeta}, for all $\zeta \in \msu_\lambda$. 
Indeed, setting $\bfD = \alpha \E\left[\Lambda(\Ztt)\right] + \lambda_n \Id{p} + \zeta \Sigma$ and $\Psi = (\Psi_i)_{i=1}^n$ s.t. for all $i \leq n$, $\Psi_i = (\sqrt{1-\alpha}\varphi(\Xtt_i), \sqrt{\alpha}\mu_x(\Ztt_i), \sqrt{1-\alpha}\Ytt_i,  \sqrt{\alpha}\mu_y(\Ztt_i))$,
yields $\Xi(\bfD) = \xi_{\alpha, \lambda}(\zeta)$. Thanks to our assumptions on $\Ztt$ and $\mathcal{T}$, in \Cref{ass:data_distribution} and \Cref{ass:artificial_data}, it is clear that such $\Psi$ and $\bfD$ satisfy the assumptions of \Cref{prop:DetEquiv_S_noaug}, 
which ensure $\E\left[\xi_{\alpha, \lambda}(\zeta)\right] = \overline{\Xi}(\bfD)$. 

In addition to \Cref{lemma:differentiation_lemma}, this results allows to approxmiate $\E\left[\hat{\theta}_{\alpha, \lambda}^\top \Sigma \hat{\theta}_{\alpha, \lambda}\right]$, which will be detailed in \Cref{sec:proofs_main_results}.

\section{Proofs of \Cref{proposition:det_equivs}, and \Cref{thm:main_result}} \label{sec:proofs_main_results}

This final section is dedicated to the proofs of the main results of the paper, namely \Cref{proposition:det_equivs}, and \Cref{thm:main_result}.

We begin by writting the proof of \Cref{proposition:det_equivs} in full detail, as it will be used to prove \Cref{thm:main_result}.
To this end, we leverage \Cref{theorem:concentration_generalization_error} as well as \Cref{prop:DetEquiv_Q_noaug}, \Cref{prop:DetEquiv_S_noaug} and \Cref{lemma:differentiation_lemma}: 
\begin{proof}[Proof of \Cref{proposition:det_equivs}]
    We will use the following notation, for any $\zeta_n \in \msu_{\lambda_n}$
    \begin{align}
        \bfD_{\alpha, \lambda_n}(\zeta_n) &= \alpha \E\left[\Lambda(\Ztt)\right] + \lambda_n \Id{p} + \zeta_n \Sigma \eqsp, \\
        \widetilde{\Resolvent}_{\alpha, \lambda_n}(\zeta_n) &= ((1-\alpha) \SampleCov + \alpha \SampleCov' + \bfD_{\alpha, \lambda_n}(\zeta_n))^{-1} \eqsp, \\
        \widetilde{\mathrm{H}}_{\alpha} &= (1 - \alpha) \SampleCrossCov + \alpha \SampleCrossCov' + \alpha \E\left[\Omega(\Ztt)\right]\eqsp, \\
        \widetilde{\theta}_{\alpha, \lambda_n}&= \widetilde{\Resolvent}_{\alpha, \lambda_n}(0) \widetilde{\mathrm{H}}_{\alpha} \eqsp, \\
        \widetilde{\xi}_{\alpha, \lambda_n}(\zeta_n) &= \widetilde{\mathrm{H}}_{\alpha} \widetilde{\Resolvent}_{\alpha, \lambda_n}(\zeta_n) \widetilde{\mathrm{H}}_{\alpha} \eqsp.
    \end{align}
    As well as,
    \begin{equation}
        \begin{aligned}
        \makebox[\linewidth][c]{$
        \forall i \leq n\eqsp, \quad
        \Phi_i= (\varphi(\Xtt_i), \mu_x(\Ztt_i)) \eqsp, 
        \qquad
        \Psi_i= (\Ytt_i^\top, \mu_y(\Ztt_i)^\top)^\top \eqsp,
        $} \\
        \makebox[\linewidth][c]{$
        \mathrm{W}_\alpha = \Diag((\sqrt{\alpha}, \sqrt{1-\alpha})) \eqsp, 
        \qquad
        \Brm_{\alpha, \lambda_n}(\zeta_n) = \mathrm{W}_\alpha\Big(\Id{2}+ n^{-1}\mathrm{W}_\alpha\E\left[\Phi_1^\top\E\left[\Resolvent_{\alpha,\lambda_n}(\zeta_n)\right]\Phi_1\right]\mathrm{W}_\alpha^\top\Big)^{-1}\mathrm{W}_\alpha \eqsp, 
        $} \\
        \makebox[\linewidth][c]{$
        \overline{\Sigma}(\zeta_n) = \E\!\left[\Phi_1 \Brm_{\alpha, \lambda_n}(\zeta_n) \Phi_1^\top\right] \eqsp,
        \qquad
        \overline{\Gamma}(\zeta_n)
        = \E\!\left[\Phi_1 \Brm_{\alpha, \lambda_n}(\zeta_n) \Psi_1^\top\right] \eqsp,
        $} \\
        \makebox[\linewidth][c]{$
        \overline{\Resolvent}_{\alpha, \lambda_n}(\zeta_n) = \left(\overline{\Sigma}(\zeta_n) + \bfD_{\alpha, \lambda_n}(\zeta_n)\right)^{-1} \eqsp,
        \qquad
        \overline{\mathrm{H}}_\alpha(\zeta_n) = \overline{\Gamma}(\zeta_n) + \E\left[\Omega(\Ztt)\right] \eqsp,
        $} \\
        \makebox[\linewidth][c]{$
        \overline{\theta}_{\alpha, \lambda_n} = \overline{\Resolvent}_{\alpha, \lambda_n}(0) \overline{\mathrm{H}}_\alpha(0) \eqsp,
        \qquad
        \overline{\xi}_{\alpha, \lambda_n}(\zeta_n) = \E\left[\Psi_1 \left\{\Id{2} - \Brm_{\alpha, \lambda_n}(\zeta_n)\right\} \Psi_1^\top\right] + \overline{\mathrm{H}}_\alpha(\zeta_n)^\top \overline{\Resolvent}_{\alpha, \lambda_n}(\zeta_n) \overline{\mathrm{H}}_\alpha(\zeta_n) \eqsp.
        $}
        \end{aligned}
    \end{equation}
    Note that the previous definition either match the ones of \Cref{proposition:det_equivs}, or generalize them to the case of a complex-value regularization $\zeta_n \Sigma$ in the resolvent.   
    We have thanks to the variance bounds on $\Lambda$ and $\Omega$ in \Cref{ass:artificial_data}, we have,
    \begin{align}
        \E\left[\|\Lambda(\Ztt) - \E\left[\Lambda(\Ztt)\right]\|_\Frob^2\right] \leq \dfrac{\const}{n} \eqsp, \quad \E\left[\|\Omega(\Ztt) - \E\left[\Omega(\Ztt)\right]\|_2^2\right] \leq \dfrac{\const}{n} \eqsp,
    \end{align}
    hence for any $\zeta_n \in \msu_{\lambda_n}$, we have,
    \begin{align}
        \left\|\E\left[\Resolvent_{\alpha, \lambda_n}\right] - \E\left[\widetilde{\Resolvent}_{\alpha, \lambda_n}(0)\right]\right\|_\Frob \leq \dfrac{\const}{\lambda_n^2 \sqrt{n}} \eqsp, \\
        \label{eq:Approximation_Resolvent_variants}
        \left\|\E\left[\hat{\theta}_{\alpha, \lambda_n}\right] - \E\left[\widetilde{\theta}_{\alpha, \lambda_n}\right]\right\|_2 \leq \dfrac{\const}{\lambda_n^2 \sqrt{n}} \eqsp, \\
        \left|\E\left[\xi_{\alpha, \lambda_n}(\zeta_n)\right] - \E\left[\widetilde{\xi}_{\alpha, \lambda_n}(\zeta_n)\right]\right| \leq \dfrac{\const}{\lambda_n^2 \sqrt{n}} \eqsp.
    \end{align} 
    
    We prove each point of \Cref{proposition:det_equivs} separately, to this end, we leverage each bound provided in \eqref{eq:Approximation_Resolvent_variants}.

    \begin{enumerate}
        \item \textit{Deterministic equivalent for $\Resolvent_{\alpha,\lambda_n}$.} 
        
             First recall from the proof of \Cref{proposition:Concentration_resolvent_around_mean}, that $\Resolvent_{\alpha,\lambda_n}$ writes as,
             \begin{align}
                \Resolvent_{\alpha, \lambda_n} = r(\Ztt, 0) \eqsp,
             \end{align}
             where $r$ was defined in \eqref{eq:def_r_and_h}. Furthermore, we have shown in \eqref{eq:Lipschitz_constant_r} that $r(\cdot, 0)$ is $\Ltt^r$-Lipschitz continuous, with
             \begin{align}
                \Ltt^r_{\msk_n} \leq \dfrac{\const(1 + \sqrt{\lambda_n}) }{\lambda_n^{3/2} \sqrt{n}} \leq \dfrac{\const}{\lambda_n^{3/2} \sqrt{n}} \eqsp,
             \end{align}
             where we have used the hypothesis $\lambda_n \geq 1/n$ to simplify the bound on the Lipschitz constant.
             It results that the function $\bfZ \mapsto \tr\left(\mathbf{A} r(\bfZ, 0)\right)$ is $\const \lambda_n^{-3/2} n^{-1/2}$-Lipschitz continuous, hence the Lipschitz concentration hypothesis on $\Ztt$ \cref{ass:data_distribution}
             yields,
             \begin{align} \label{eq:Concentration_R_final}
                \left|\tr\left(\mathbf{A} r(\Ztt, 0)\right) - \E\left[\tr\left(\mathbf{A} r(\Ztt, 0)\right)\right]\right| \stochdom \dfrac{\const}{\lambda_n^{3/2} \sqrt{n}} \eqsp.
             \end{align}
             Furthermore, using \Cref{prop:Detequiv_R_twobytwo} with $\Psi_i \leftarrow (\sqrt{1-\alpha}\varphi(\Ztt_i), \sqrt{\alpha}\mu_x(\Ztt_i))$, and $\bfD \leftarrow \bfD_{\alpha, \lambda_n}(0)$, we have,
             \begin{align} \label{eq:Bias_detequivR_final}
                \left\|\E\left[\widetilde{\Resolvent}_{\alpha, \lambda_n}\right] - \overline{\Resolvent}_{\alpha, \lambda_n}\right\|_\Frob \leq \dfrac{\const}{\lambda_n^{7/2} \sqrt{n}} \eqsp,
             \end{align}
             where we have used $\kappa < + \infty$ and $\limsup p / n < + \infty$ granted by \Cref{ass:porportionality} to simplify the bound.
             Combining \eqref{eq:Approximation_Resolvent_variants}, \eqref{eq:Concentration_R_final} and \eqref{eq:Bias_detequivR_final} concludes the proof of the first point.

        \item \textit{Deterministic equivalent for $\hat{\theta}_{\alpha, \lambda_n}$.}
        
             Recall from \Cref{proposition:Concentration_resolvent_around_mean} that, for any sequence $(\mathbf{a}_n)_{n \in \N}$ of vectors in $\R^p$,
             \begin{align} \label{eq:Concentration_theta_final}
                \left|\mathbf{a}_n^\top \hat{\theta}_{\alpha, \lambda_n} - \E\left[\mathbf{a}_n^\top \hat{\theta}_{\alpha, \lambda_n}\right]\right| \stochdom \dfrac{\|\mathbf{a}_n\|_2}{\lambda_n^{3/2} \sqrt{n}}\left(1 + \sqrt{\lambda_n}\right) \leq \dfrac{\const \|\mathbf{a}_n\|_2}{\lambda_n^{3/2} \sqrt{n}} \eqsp,
             \end{align}
             which we have simplified thanks to $\lambda_n \geq 1/n$. Furthermore, calling \Cref{prop:DetEquiv_Q_noaug} with $\Psi_i \leftarrow (\sqrt{1-\alpha}\varphi(\Ztt_i)^\top, \sqrt{\alpha}\mu_x(\Ztt_i)^\top, \sqrt{1-\alpha}\Ytt_i^\top, \sqrt{\alpha}\mu_y(\Ztt_i)^\top)^\top$, 
             and $\bfD \leftarrow \bfD_{\alpha, \lambda_n}(0)$, directly yields,
             \begin{align} \label{eq:Bias_detequivTheta_final}
                \left\|\E\left[\widetilde{\theta}_{\alpha, \lambda_n}\right] - \overline{\theta}_{\alpha, \lambda_n}\right\|_2 \leq \dfrac{\const}{\lambda_n^{4} \sqrt{n}} \eqsp,
             \end{align}
             where the bound of \Cref{prop:DetEquiv_Q_noaug} has been simplified thanks to $\lambda_n \geq 1/n$, $\kappa < + \infty$ and $\limsup p / n < + \infty$ granted by \Cref{ass:porportionality}.
             Combining \eqref{eq:Approximation_Resolvent_variants}, \eqref{eq:Concentration_theta_final} and \eqref{eq:Bias_detequivTheta_final} concludes the proof of the second point.

        \item \textit{Deterministic equivalent for $\xi_{\alpha, \lambda_n}(r_n \rme^{\rmi t})$.} 
        
             Recall from \Cref{proposition:Concentration_resolvent_around_mean} that, for any sequence $(r_n)_{n \in \N}$ of real numbers,
             \begin{align} \label{eq:intermediate_concentration_xi}
                \left|\xi_{\alpha, \lambda_n}(r_n \rme^{\rmi t}) - \E\left[\xi_{\alpha, \lambda_n}(r_n \rme^{\rmi t})\right]\right| \stochdom \dfrac{1 + r_n}{(\lambda_n - r_n \Ltt_\varphi^2)^{3/2}\sqrt{n}}\left(1 + \sqrt{\lambda_n - r_n \Ltt_\varphi^2}\right) \eqsp,
             \end{align}
             in particular, for $r_n = 0$, and thanks to $\lambda_n \geq 1/n$, we have,
             \begin{align} \label{eq:Concentration_xi_final}
                \left|\xi_{\alpha, \lambda_n}(0) - \E\left[\xi_{\alpha, \lambda_n}(0)\right]\right| \stochdom \dfrac{\const}{\lambda_n^{3/2} \sqrt{n}} \eqsp.
             \end{align}
             Furthermore, calling \Cref{prop:DetEquiv_S_noaug} with $\Psi_i \leftarrow (\sqrt{1-\alpha}\varphi(\Ztt)^\top, \sqrt{\alpha}\mu_x(\Ztt)^\top, \sqrt{1-\alpha}\Ytt^\top, \sqrt{\alpha}\mu_y(\Ztt)^\top)^\top$, and $\bfD \leftarrow \bfD_{\alpha, \lambda_n}(0)$, we have,
             \begin{align} \label{eq:Bias_detequivXi_final}
                \left|\E\left[\widetilde{\xi}_{\alpha, \lambda_n}(0)\right] - \overline{\xi}_{\alpha, \lambda_n}(0)\right| \leq \dfrac{\const}{\lambda_n^{7/2} \sqrt{n}} \eqsp,
             \end{align}
             where the bound of \Cref{prop:DetEquiv_S_noaug} has been simplified thanks to $\lambda_n \geq 1/n$, $\kappa < + \infty$ and $\limsup p / n < + \infty$ granted by \Cref{ass:porportionality}.
             Combining \eqref{eq:Approximation_Resolvent_variants}, \eqref{eq:Concentration_xi_final} and \eqref{eq:Bias_detequivXi_final} concludes the proof of the third point.

        \item \textit{Deterministic equivalent for $\hat{\theta}_{\alpha, \lambda_n}^\top \Sigma \hat{\theta}_{\alpha, \lambda_n}$.} 
             
             First leverage \eqref{eq:intermediate_concentration_xi} as well as \Cref{lemma:differentiation_lemma} to show that for any sequence $(r_n)_{n \in \N}$ of positive numbers,
             \begin{align}
                \left|\partial_\zeta \xi_{\alpha, \lambda_n}(0) - \E\left[\partial_\zeta \xi_{\alpha, \lambda_n}(0)\right]\right| \stochdom \dfrac{r_n^{-1} + 1}{(\lambda_n - r_n \Ltt_\varphi^2)^{3/2}\sqrt{n}}\left(1 + \sqrt{\lambda_n - r_n \Ltt_\varphi^2}\right)  \eqsp,
             \end{align}
             taking $r_n = \lambda_n / (2 \Ltt_\varphi^2)$, and recalling that $\partial_\zeta \xi_{\alpha, \lambda_n}(0) = -\hat{\theta}_{\alpha, \lambda_n}^\top \Sigma \hat{\theta}_{\alpha, \lambda_n}$, we have,
             \begin{align} \label{eq:intermediate_concentration_last_point_1}
                \left|\hat{\theta}_{\alpha, \lambda_n}^\top \Sigma \hat{\theta}_{\alpha, \lambda_n} + \E\left[\partial_\zeta \xi_{\alpha, \lambda_n}(0)\right]\right| \stochdom \dfrac{1}{\lambda_n^{5/2}\sqrt{n}} \eqsp.
             \end{align}
             In addition, applying \Cref{lemma:differentiation_lemma} to the third equation in \eqref{eq:Approximation_Resolvent_variants}, we have,
             \begin{align}
                \sup_{t \in  [0, 2 \uppi]} \left|\E\left[\xi_{\alpha, \lambda_n}(r_n \rme^{\rmi t})\right] - \E\left[\widetilde{\xi}_{\alpha, \lambda_n}(r_n \rme^{\rmi t})\right]\right| \leq \dfrac{\const}{\lambda_n^2 \sqrt{n}} \eqsp,
             \end{align}
             thus,
             \begin{align} \label{eq:intermediate_concentration_last_point_2}
                \left|\E\left[\partial_\zeta \xi_{\alpha, \lambda_n}(0)\right] - \E\left[\partial_\zeta \widetilde{\xi}_{\alpha, \lambda_n}(0)\right]\right| \leq \dfrac{\const}{\lambda_n^3 \sqrt{n}} \eqsp.
             \end{align}
             Finally, from \Cref{prop:DetEquiv_S_noaug} applied with $\Psi_i \leftarrow (\sqrt{1-\alpha}\varphi(\Ztt_i)^\top, \sqrt{\alpha}\mu_x(\Ztt_i)^\top, \sqrt{1-\alpha}\Ytt_i^\top, \sqrt{\alpha}\mu_y(\Ztt_i)^\top)^\top$, 
             and $\bfD \leftarrow \bfD_{\alpha, \lambda_n}(r_n \rme^{\rmi t})$, we have,
             \begin{align}
             \sup_{t \in  [0, 2 \uppi]} \left|\E\left[\widetilde{\xi}_{\alpha, \lambda_n}(r_n \rme^{\rmi t})\right] - \overline{\xi}_{\alpha, \lambda_n}(r_n \rme^{\rmi t})\right| \leq \dfrac{\const}{\lambda_n^{7/2} \sqrt{n}} \eqsp,
             \end{align}
             thus \Cref{lemma:differentiation_lemma} yields,
             \begin{align} \label{eq:intermediate_concentration_last_point_3}
                \left|\E\left[\partial_\zeta \widetilde{\xi}_{\alpha, \lambda_n}(0)\right] - \partial_\zeta \overline{\xi}_{\alpha, \lambda_n}(0)\right| \leq \dfrac{\const}{\lambda_n^{9/2} \sqrt{n}} \eqsp.
             \end{align}
             Combining \eqref{eq:intermediate_concentration_last_point_1}, \eqref{eq:intermediate_concentration_last_point_2} and \eqref{eq:intermediate_concentration_last_point_3} we conclude:
             \begin{align}
                \left|\hat{\theta}_{\alpha, \lambda_n}^\top \Sigma \hat{\theta}_{\alpha, \lambda_n} + \partial_\zeta \overline{\xi}_{\alpha, \lambda_n}(0)\right| \stochdom \dfrac{1}{\lambda_n^{9/2}\sqrt{n}} \eqsp.
             \end{align}
             Finally, we explicitely compute the diferential $\partial_\zeta \overline{\xi}_{\alpha, \lambda_n}(0)$ as follows:
             \begin{align}
                \partial_\zeta \overline{\xi}_{\alpha, \lambda_n}(0) = -\E\left[\Psi_1 \partial_\zeta \Brm_{\alpha, \lambda_n}(0) \Psi_1^\top\right] - \overline{\mathrm{H}}_\alpha(0)^\top \Resolvent_{\alpha, \lambda_n}(0) \partial_\zeta \overline{\Sigma}(0) \Resolvent_{\alpha, \lambda_n}(0) \overline{\mathrm{H}}_\alpha(0) + 2 (\partial_\zeta \overline{\mathrm{H}}_\alpha(0))^\top \Resolvent_{\alpha, \lambda_n}(0) \overline{\mathrm{H}}_\alpha(0) \eqsp,
             \end{align}
             and verify that,
             \begin{align}
                \partial_\zeta \Brm_{\alpha, \lambda_n}(0) = \mathrm{D}_{\alpha, \lambda_n} \eqsp, \qquad \partial_\zeta \overline{\Sigma}(0) = \E\left[\Phi_1 \mathrm{D}_{\alpha, \lambda} \Phi_1\right] + \Sigma \eqsp, \qquad \partial_\zeta \overline{\mathrm{H}}_\alpha(0) = \E\left[\Phi_1 \mathrm{D}_{\alpha, \lambda} \Psi_1\right] \eqsp,
             \end{align}
             where $\mathrm{D}_{\alpha, \lambda_n}$ is the matrix defined in \eqref{eq:definition_dilation_matrices_D}. This concludes the proof of \Cref{proposition:det_equivs}.
    \end{enumerate}
\end{proof}

Lastly, we prove \Cref{thm:main_result}, 
\begin{proof}[Proof of \Cref{thm:main_result}]
    We begin by recalling the definition of $\overline{\mathcal{G}}(\alpha, \lambda_n)$, we have,
    \begin{align}
        \overline{\mathcal{G}}(\alpha, \lambda_n) = \theta_\star^\top \Sigma_{\star\star} \theta_\star + \sigma^2 - 2 \theta_\star^\top \Sigma_\star \overline{\theta}_{\alpha, \lambda_n} + \overline{\theta}_{\alpha, \lambda_n}^\top \{\overline{\Sigma}' + \Sigma\} \overline{\theta}_{\alpha, \lambda_n} - 2 \overline{\theta}_{\alpha, \lambda_n}^\top\overline{\Gamma}'  + \overline{\gamma} \eqsp,
    \end{align}
    where one might identify in the previous term $-\partial_\zeta \overline{\xi}_{\alpha, \lambda_n}(0) = \overline{\theta}_{\alpha, \lambda_n}^\top \{\overline{\Sigma}' + \Sigma\} \overline{\theta}_{\alpha, \lambda_n} - 2 \overline{\theta}_{\alpha, \lambda_n}^\top\overline{\Gamma}'  + \overline{\gamma}$.
    In addition, we have from \eqref{eq:generalization_error_rewritten2} that,
    \begin{align}
      \mathcal{G}(\alpha, \lambda_n) = \theta_\star^\top \Sigma_{\star\star} \theta_\star + \sigma^2 - 2 \theta_\star^\top \Sigma_\star \hat{\theta}_{\alpha, \lambda_n} - \partial_\zeta \xi_{\alpha, \lambda_n}(0) \eqsp,
    \end{align}
    thus, 
    \begin{align}
    \left|\mathcal{G}(\alpha, \lambda_n) - \overline{\mathcal{G}}(\alpha, \lambda_n)\right| \leq 2\left|\theta_\star^\top \Sigma_{\star} \left\{\hat{\theta}_{\alpha,\lambda_n} - \overline{\theta}_{\alpha, \lambda_n}\right\}\right| + \left|\partial_\zeta \xi_{\alpha, \lambda_n}(0) - \partial_\zeta \overline{\xi}_{\alpha, \lambda_n}(0)\right| \eqsp.
    \end{align}
    The result follows directly from \Cref{proposition:det_equivs}.
\end{proof}

\end{document}